\documentclass{article}

 \usepackage[main, final]{neurips_2026}

\usepackage[utf8]{inputenc} 
\usepackage[T1]{fontenc}    
\usepackage{hyperref} 
\usepackage{url}            
\usepackage{booktabs}       
\usepackage{nicefrac}       
\usepackage{microtype}      
\usepackage{xcolor}   

\usepackage{graphicx} 
\usepackage{natbib}
\usepackage{subcaption}

\usepackage{algorithm}
\usepackage{algpseudocode}
\usepackage{thmtools}
\usepackage{thm-restate}

\usepackage{url}
\usepackage{paralist}
\usepackage{amsthm}
\usepackage{amsmath,amsfonts,bm,amssymb}
\usepackage{pifont}

\usepackage{cleveref}
\usepackage{comment}
\usepackage{todonotes}
\usepackage{placeins}
\usepackage{booktabs}
\usepackage{multirow}

\usepackage{tcolorbox}

\usepackage{mwe} 
\usepackage{wrapfig}    

\usepackage{enumitem}

\def\eqref#1{equation~\ref{#1}}

\def\1{\bm{1}}

\def\ra{{\textnormal{a}}}

\def\rg{{\textnormal{g}}}

\def\rs{{\textnormal{s}}}

\def\ru{{\textnormal{u}}}
\def\rv{{\textnormal{v}}}

\def\rx{{\textnormal{x}}}
\def\ry{{\textnormal{y}}}

\def\rmA{{\mathbf{A}}}

\def\rmE{{\mathbf{E}}}

\def\rmG{{\mathbf{G}}}

\def\rmI{{\mathbf{I}}}

\def\rmK{{\mathbf{K}}}

\def\rmM{{\mathbf{M}}}

\def\rmO{{\mathbf{O}}}
\def\rmP{{\mathbf{P}}}
\def\rmQ{{\mathbf{Q}}}

\def\rmS{{\mathbf{S}}}

\def\rmU{{\mathbf{U}}}
\def\rmV{{\mathbf{V}}}
\def\rmW{{\mathbf{W}}}
\def\rmX{{\mathbf{X}}}
\def\rmY{{\mathbf{Y}}}

\DeclareMathAlphabet{\mathsfit}{\encodingdefault}{\sfdefault}{m}{sl}
\SetMathAlphabet{\mathsfit}{bold}{\encodingdefault}{\sfdefault}{bx}{n}

\newcommand{\softmax}{\mathrm{softmax}}

\newcommand{\KL}{D_{\mathrm{KL}}}
\newcommand{\Var}{\mathrm{Var}}

\DeclareMathOperator{\sign}{sign}

\newtheorem{theorem}{Theorem}
\newtheorem{lemma}{Lemma}
\newtheorem{corollary}{Corollary}
\newtheorem{definition}{Definition}
\newtheorem{proposition}{Proposition}

\newtheorem{remark}{Remark}

\newcommand{\relu}{\mathrm{ReLU}}
\newcommand{\gelu}{\mathrm{GELU}}

\newcommand{\expectation}[2]{\underset{#1}{\mathbb{E}}\left[#2\right]}

\usepackage{xargs}                      

\usepackage{mathrsfs}

\newcommand{\SA}{\text{SA}}
\newcommand{\MLP}{\text{MLP}}
\newcommand{\LN}{\text{norm}}

\newcommand{\lp}{\left(}
\newcommand{\rp}{\right)}

\newcommand{\trace}{\mathrm{tr}}

\newcommand{\cmark}{\ding{51}}%
\newcommand{\xmark}{\ding{55}}%
\newcommand{\cmarkcolor}{\textcolor{applegreen}{\ding{51}}}
\newcommand{\xmarkcolor}{\textcolor{brightmaroon}{\ding{55}}}
\newcommand{\eg}{\emph{e.g.},}
\newcommand{\ie}{\emph{i.e.},}

\Crefname{equation}{Eq.}{Eqs.}
\Crefname{figure}{Figure}{Figures}
\Crefname{tabular}{Table}{Tables}
\Crefname{appendix}{Appendix}{Appendices}
\Crefname{section}{Section}{Sections}
\Crefname{assumption}{Assumption}{Assumptions}

\definecolor{applegreen}{rgb}{0.01, 0.75, 0.24}
\definecolor{brightmaroon}{rgb}{0.66, 0.13, 0.24}
\definecolor{warningyellow}{rgb}{0.86, 0.58, 0.17}

\title{The Quantization Benefits of Residual-Free Transformers}

\author{%
  Yiping Ji$^{1,2}$\thanks{Equal contribution. Emails: yiping.ji@adelaide.edu.au, mahalakshmi.sabanayagam@adelaide.edu.au} \\
  \And
  Mahalakshmi Sabanayagam$^{1*}$ \\
  \And
  Peyman Moghadam$^{2,3}$ \\
  \And
  Hemanth Saratchandran$^1$ \\
  \And
  Simon Lucey$^1$ \\
  \AND
  \normalfont $^1$ Australian Institute for Machine Learning, Adelaide University 
  $^2$ CSIRO \\
  $^3$ Commonwealth Bank of Australia 
}

\begin{document}

\maketitle

\begin{abstract}
Large-scale transformer training and deployment are increasingly constrained by the transfer of activations, gradients, and optimizer states across accelerators. 
Low-bit quantization offers a natural remedy, but transformer activations are often heavy-tailed and outlier-dominated, making simple quantization highly lossy. We show that this difficulty is not only a property of the quantizer, but also of the architecture. Specifically, residual connections can drive transformer activations away from Gaussianity during training. Using controlled comparisons between residual and residual-free transformers, we demonstrate that this effect leads to substantially higher quantization error and accuracy degradation at low precision in residual models. We explain the phenomenon through an excess kurtosis analysis, showing that residual mixing can amplify non-Gaussianity, whereas dense mixing in residual-free models contracts non-Gaussianity. We then show that residual-free transformers can be made trainable using orthogonal initialization, spectral or second-order optimization, and depth-aware scaling of attention temperature. In language tasks, while there is a small drop in full precision performance, these models retain near-Gaussian activations and exhibit significantly improved robustness to low-bit quantization. Our results identify an \emph{accuracy--compressibility trade-off} in transformer design and motivate architecture-level approaches to quantization-friendly foundation models.
\end{abstract}


\section{Introduction}
\label{sec:intro}
The recent success of large foundation models such as GPT-4, PaLM, and LLaMA has been enabled by aggressive scaling of model size, data, and compute \citep{hoffmann2022training}. 
However, this scaling trend has made training increasingly constrained by systems efficiency rather than model design alone. 
State-of-the-art models are now trained across thousands of accelerators using data, tensor, pipeline, and optimizer-state parallelism \citep{shoeybi2019megatron,huang2019gpipe,narayanan2019pipedream,rajbhandari2020zero,xu2021gspmd}. 
While these approaches enable training extremely large models, they also introduce substantial communication overhead. 
As models are distributed across more devices, more training time and energy are spent transferring activations, gradients, and optimizer states between devices instead of doing local computation.
Consequently, the efficiency of large-scale training is increasingly governed by how effectively these tensors can be communicated, and in particular by how well they can be compressed without losing performance. 
This gives rise to a central challenge in distributed training where \emph{communication overhead can dominate computation.}

A major response to tackle this problem has been low-precision training and quantization. FP8 training and low-bit quantization methods have substantially reduced the memory and communication cost of large models \citep{micikevicius2022fp8,zandieh2025turboquant}.  
However, the success of these methods often depends on increasingly specialized machinery beyond simple low-precision, such as mixed-precision outlier channels, per-channel or per-token scaling, clipping calibration, equivalent transformations, and layerwise reconstruction. A key reason is that transformer activations are often non-Gaussian, heavy-tailed, and dominated by a small number of outliers, making naive low-bit quantization highly lossy \citep{li2020additive}. Prior work has linked these structured activation outliers and large activation ranges to residual connections, attention, normalization, and training-dynamics effects \citep{bondarenko2021understanding,bondarenko2023quantizable,dettmers2022gpt3,xiao2023smoothquant,wei2023outlier,he2024understanding}.
This suggests that quantization difficulty is partly architectural, determined by whether the tensors being communicated are inherently poorly suited for compression.

Motivated by this architectural view of quantization difficulty, we take a different perspective on post-training quantization. 
Rather than asking how to compensate for pathological activations after they appear, we ask \textbf{which architectural mechanisms give rise to them, and whether they can be mitigated at the source.} 
Complementing prior work on activation outliers, 
we isolate the role of residual connections through controlled residual versus residual-free comparisons. 
We show that residual connections, while central to optimization in transformers, can drive activations away from Gaussianity during training, making them increasingly difficult to quantize. Whereas, \emph{residual-free transformers with orthogonal initialization \citep{saxe2014exact} and second-order optimization, such as 
Muon \citep{jordan2024muon}, SOAP \citep{vyassoap}, and Shampoo \citep{gupta2018shampoo,lin2025understanding}, maintain nearly Gaussian-like activations across layers and training time}. This statistical difference has direct consequences for low precision. As shown in \Cref{fig:rate_distortion}, under simple uniform quantization of activations and weights, residual models incur substantially higher accuracy degradation and information loss than their residual-free counterparts. The excess kurtosis measurements confirm that this gap is associated with heavier-tailed, non-Gaussian activations.
These results reveal an accuracy--compressibility tradeoff in residual models as well as motivate residual-free transformer 
as a potential architecture-level route to quantization-friendly foundation models.

\begin{figure}[t!]
    \centering
    \includegraphics[width=0.32\linewidth]{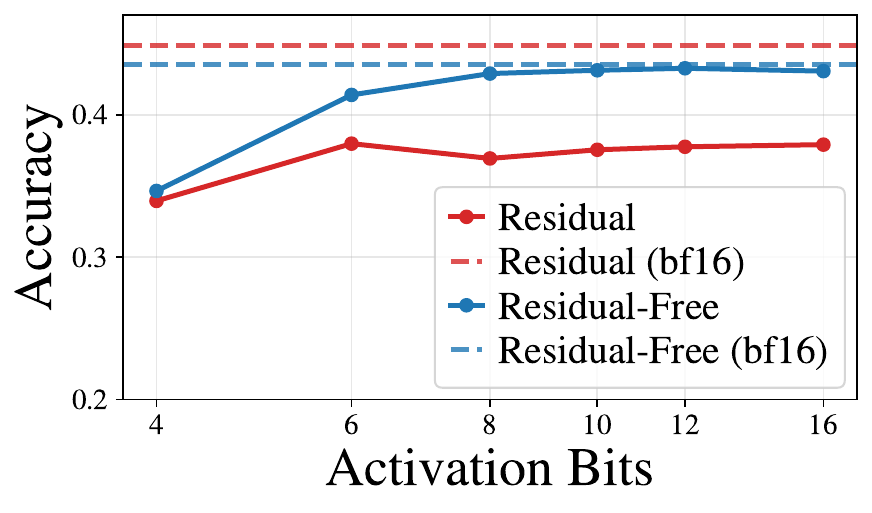}
    \includegraphics[width=0.32\linewidth]{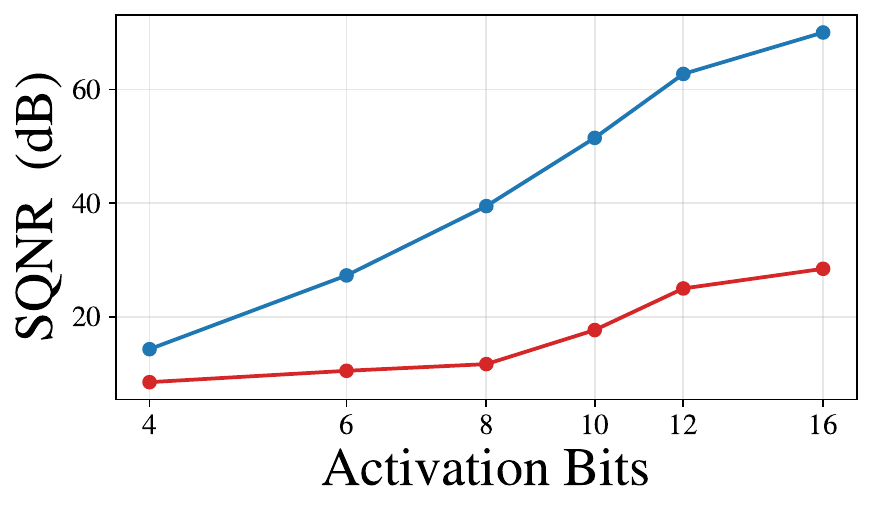}
    \includegraphics[width=0.32\linewidth]{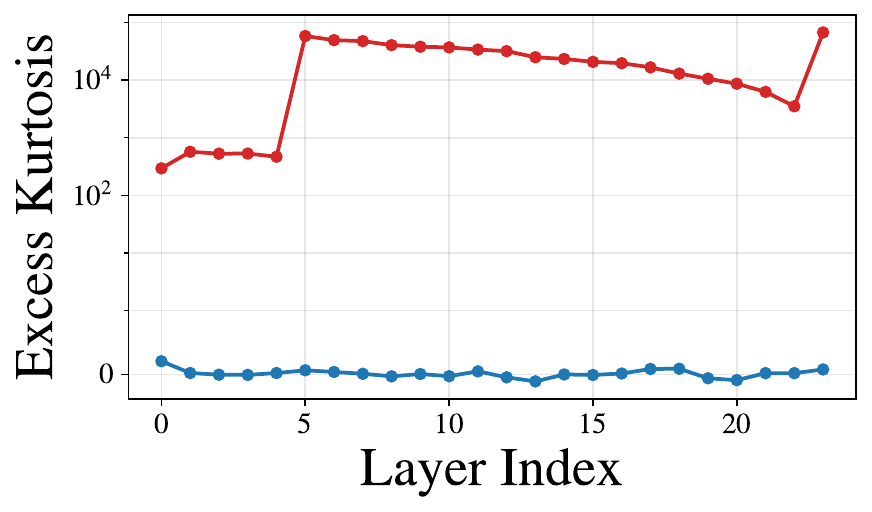}
\caption{\textbf{Residual-free transformers remain near-Gaussian and quantization-robust.} 
On $24$-layer models trained 
in bf16 precision, the residual-free model with \emph{orthogonal initialization and KL Shampoo} achieves performance comparable to the residual one on $8$ downstream tasks and retains substantially high accuracy under aggressive quantization (int) of activations and 8-bit weights, while the residual model degrades sharply (\textbf{left}). 
This gap is explained by a higher signal-to-quantization-noise ratio (SQNR) (\textbf{middle}) and near-Gaussian activations measured by excess kurtosis (near-zero indicates near-Gaussian in the fourth-moment sense) 
across layers (\textbf{right}), whereas residual models develop heavy-tailed activations.
}
\vspace{-0.2cm}
    \label{fig:rate_distortion}
\end{figure}

While the preservation of near-Gaussian activations in residual-free networks is promising for quantization, it remains a practical challenge to train these networks compared to their residual counterparts, as skip connections were introduced precisely to stabilize optimization in deep networks. A growing body of work has studied how to recover the optimization benefits of residual connections without explicit skip connections through different parameterizations and initializations \citep{hardt2017identity,xiao2018dynamical,he2023deep}. 
Recent work on residual-free transformers further shows that suitable initialization can substantially close the optimization and performance gaps to residual architectures \citep{ji2026cutting}. 
Building on this line of work, we further explore the trainability aspects of residual-free transformers and find that orthogonal initialization, spectral or second-order optimization, and appropriate depth-aware scalings of attention temperature can effectively train residual-free transformers. 
Crucially, these models retain near-Gaussian activation statistics throughout training, leading to substantially better low-bit quantization than residual transformers.
Our \textbf{contributions} are summarized as follows:  

\textbf{1. Residual connections are a controllable source of quantization-unfriendly activation statistics.}  We identify residual connections as a key architectural source of activation non-Gaussianity in transformers. Across controlled residual versus residual-free comparisons, we show that residual models develop heavy-tailed activations during training, whereas residual-free models can preserve near-Gaussian activations when paired with orthogonal initialization and second-order optimization.

\textbf{2. Excess kurtosis analysis explains why residual mixing amplifies non-Gaussianity.} We provide an analysis under simplified settings showing that residual mixing preserves and can amplify excess kurtosis, whereas dense mixing contracts it. Our analysis shows that preserving Gaussianity requires both appropriate initialization and optimization, and cannot be achieved by either alone (\Cref{sec:kurtosis_analysis}). 

\textbf{3. A practical residual-free training recipe yields quantization-friendly transformers.} We develop a practical recipe for training residual-free transformers in \Cref{sec:method} combining orthogonal initialization, second-order optimization, and depth-aware scaling of attention temperature.
In \Cref{sec:exp}, we show that while there is a small drop in full-precision performance, these models exhibit substantially higher robustness to simple uniform quantization in language tasks.
These results reveal an \emph{accuracy--compressibility trade-off} between residual and residual-free transformers.

\section{Related Work}
\label{sec:related_work}
We focus on works most directly related to residual-free transformers, activation
outliers, and quantization. A broader discussion appears in \Cref{app:related_work}.

\textbf{Residual-free transformers.}
\citep{he2023deep} study training transformers without residual connections or normalization and show that they suffer from rank collapse, where the kernel matrix converges to rank one with depth. They modify self-attention to preserve well-behaved kernels at initialization, but report substantially more training steps (roughly 5 times) to match residual baselines.
Further, \citet{ji2026cutting} show that residual-free transformers can train as efficiently as residual models using a principled initialization of the self-attention block, without architectural changes and while remaining compatible with FlashAttention \citep{dao2022flashattention,dao2023flashattention2}. However, their method still struggles to scale to deeper networks, for example, beyond 12 layers.

\textbf{Quantization and activation outliers in transformers.}
Many prior works have identified outliers in transformer activations and studied their effect on quantization. \citep{bondarenko2021understanding} identify that transformer activations exhibit structured outliers that are difficult to represent with low-bit fixed-point formats, while \citep{bondarenko2023quantizable} links these outliers to attention heads. \citet{he2024understanding} and \citet{nrusimha2024mitigating} analyze how architectural and optimization choices influence the emergence of outlier features during training. To address these issues, low-precision FP8 training and low-bit quantization are introduced. 
For example,  LLM.int8() \citep{dettmers2022gpt3} uses mixed-precision decomposition to isolate outlier channels, SmoothQuant \citep{xiao2023smoothquant} shifts activation outlier difficulty into weights through equivalent transformations, and methods such as QLoRA \citep{dettmers2023qlora} and TurboQuant \citep{zandieh2025turboquant} adopt specialized scaling, clipping, or reconstruction techniques. Our work takes a complementary perspective. Instead of compensating for outliers after they arise, we study whether architecture and optimization can produce activation statistics that are easier to quantize without any outlier-handling methods in a controlled setting.

\section{Analysis of Activations of Residual and Residual-Free Transformers}
\label{sec:kurtosis_analysis}

\subsection{Setup and Preliminaries}
\label{sec:setup_prelims}
We compare residual and residual-free transformers through the evolution of their layerwise activation statistics. We give only the notation needed for the main argument here, and defer a complete architectural specification to \Cref{app:prelims}.

\textbf{Notation.} We write vectors and matrices in boldface. For $\rv\in\mathbb{R}^d$, $v_i$ denotes its $i$-th coordinate. For $\rmM\in\mathbb{R}^{m\times n}$, $M_{ij}$ denotes its $(i,j)$-th entry. Unless otherwise stated, $\|\cdot\|$ denotes the Euclidean norm for vectors and the spectral norm for matrices, and $\rmI_d$ denotes the $d$-dimensional identity matrix. For a function $f$, we use $\nabla f$ to denote its gradient.

\textbf{Residual and residual-free blocks in transformer.}
Let $\rmX_\ell\in\mathbb{R}^{T\times d}$ denote the token representation at layer $\ell$, where $T$ is the sequence length and $d$ is the model dimension. Each transformer block consists of a self-attention map $\text{SA}(\cdot)$ and a position-wise feedforward map $\MLP(\cdot)$, both applied to normalized inputs through $\LN(\cdot)$.
For a residual transformer,
\begin{align}
    \rmY_\ell = \rmX_{\ell-1} + \SA\!\left(\LN(\rmX_{\ell-1})\right), \qquad 
    \rmX_\ell  = \rmY_\ell + \MLP\!\left(\LN(\rmY_\ell)\right). \label{eq:residual}
\end{align}
The residual-free counterpart removes the additive skip paths,
\begin{align}
    \rmY_\ell = \SA\!\left(\LN(\rmX_{\ell-1})\right), \qquad 
    \rmX_\ell
    = \MLP\!\left(\LN(\rmY_\ell)\right). \label{eq:residual_free}
\end{align}
This distinction, the presence or absence of the additive identity path, is the architectural feature we isolate in the analysis.
Here $\LN(\cdot)$ denotes either Layer Normalization \cite{ba2016layer} or RMS Normalization \citep{zhang2019root}, and serves to control the scale of the input to each sublayer. The self-attention operation $\SA(\cdot)$ for any input $\rmX$ is defined as
\begin{equation}
    \text{SA}(\mathbf{X}) = \mathbf{A}\mathbf{V}\mathbf{W}^{\text{O}}, \nonumber
\end{equation}
where $\mathbf{Q} = \mathbf{X}\mathbf{W}^{\text{Q}}, \mathbf{K} = \mathbf{X}\mathbf{W}^{\text{K}}, \mathbf{V} = \mathbf{X}\mathbf{W}^{\text{V}},$
and the attention matrix is $\mathbf{A} = \softmax\big(\frac{1}{\sqrt{d}}\mathbf{Q}\mathbf{K}^\top\big)$. 
The parameter matrices 
$\mathbf{W}^{\text{Q}}, \mathbf{W}^{\text{K}}, \mathbf{W}^{\text{V}}, \mathbf{W}^{\text{O}} \in \mathbb{R}^{d \times d}$ 
are learnable. In practice, $\SA(\cdot)$ is multi-head attention with $h$ heads and has the form
$$ \SA(\rmX) =
\mathrm{Concat}\bigl(\rmA_1\rmV_1,\ldots,\rmA_h\rmV_h\bigr)\rmW^{O},
$$
where $\rmA_i$ are attention matrices with projection matrices $\mathbf{W}^{\text{Q}}_i, \ \mathbf{W}^{\text{K}}_i, \ \mathbf{W}^{\text{V}}_i 
    \in \mathbb{R}^{d \times d_h},  d_h = \tfrac{d}{h}$ and $\rmW^O$ is the output projection.
The MLP block has the form
$$ \MLP(\rmX) = \phi(\rmX\rmW^U)\rmW^D,$$
where $\phi$ is a pointwise nonlinearity such as $\gelu$, and $\mathbf{W}^{\text{U}} \in \mathbb{R}^{d \times d_f}$ and 
$\mathbf{W}^{\text{D}} \in \mathbb{R}^{d_f \times d}$ are learnable parameters.
Thus, at the level of a single coordinate or token, each sublayer repeatedly applies a nonlinear transformation and then mixes coordinates through matrices such as $\rmW^O$ and $\rmW^D$.
Our analysis focuses on how this mixing changes higher-order activation statistics, and how the additive residual path changes that evolution.

\textbf{Initialization.}  
We analyze standard variance-preserving initializations, such as Xavier (Glorot) \citep{glorot2010understanding} or Kaiming \citep{he2015delving}, and orthogonal initialization \citep{saxe2014exact}. 
The key property for our analysis is that orthogonal mixing preserves norms and avoids the row-norm and singular-value fluctuations present in dense Gaussian initialization.

\textbf{Optimization models.}
To reason about training dynamics of loss $\mathcal{L}$ while keeping the analysis tractable, we use
two stylized update rules. First, \emph{sign gradient descent} (GD) \citep{bernstein2018signsgd},
\begin{align}
    \rmW_{t+1}
    =
    \rmW_t-\eta\,\sign(\nabla \mathcal{L}(\rmW_t)),
    \label{eq:sign_GD}
\end{align}
where $\eta > 0$ is the learning rate and the sign function is applied elementwise. It serves as a simplified model of coordinate-wise adaptive first-order methods such as Adam \citep{kingma2015adam} and AdamW \citep{loshchilov2018decoupled}. Second, \emph{spectral gradient descent} updates a matrix in the polar direction of its gradient. If $\nabla \mathcal{L}(\rmW_t)=\rmU_t\rmS_t\rmV_t^\top$, then
\begin{align}
    \rmW_{t+1}
    =
    \rmW_t-\eta\,\rmU_t\rmV_t^\top .
    \label{eq:spectral_GD}
\end{align}
This update is a tractable proxy for matrix-preconditioned or spectral
second-order methods such as Shampoo \citep{gupta2018shampoo}, Muon \citep{jordan2024muon}, SOAP \citep{vyassoap}, and KL Shampoo \citep{lin2025understanding}.
We emphasize that these update rules are not exact models of AdamW, Muon, SOAP, or KL Shampoo. They isolate the geometric distinction relevant to our analysis, in which coordinatewise updates can introduce anisotropy, whereas matrix-normalized or preconditioned updates can better preserve near-isometry induced by orthogonal initialization.

\subsection{Theoretical Analysis of Layerwise Activation Distributions}
\label{sec:theory}

We first establish that normalization is necessary to control second moments in both residual and residual-free transformers. We then analyze the higher-order statistics, thereby the distribution of activations, using excess kurtosis, which characterizes the tailedness of a distribution.
The theoretical results below are intentionally idealized. They are meant to isolate the fourth-moment effect of dense mixing and residual addition, not to fully model all dependencies present in trained transformers.

\textbf{Variance dynamics without normalization.}
We consider the layerwise recursions defined in \Cref{sec:setup_prelims} and analyze the behavior of the variance across depth.

\begin{lemma}[Normalization controls variance collapse and growth]
\label{lem:variance}
Consider the residual and residual-free update in \Cref{eq:residual,eq:residual_free} without normalization, that is, $\LN(\rmX) = \rmX$.
Assume coordinates of $\rmX_{\ell}$ are centered with variance $q_\ell$, and weight matrices are variance-preserving.
Then,
\begin{itemize}[leftmargin=*]
    \item \textbf{Residual-free case:}
    The variance of $\rmX_{\ell+1}$, $q_{\ell+1} = c q_\ell$ with $c < 1$ for common nonlinearities that contracts the input. Hence, the variance of deep layers, $q_L \to 0$ exponentially with depth.
    \item \textbf{Residual case:}
    The variance of $\rmX_{\ell+1}$, $q_{\ell+1} = (1+c) q_\ell$ with $c>0$ under weak correlation, so $q_L$ grows exponentially with depth.
\end{itemize}
\begin{tcolorbox}[colback=yellow!5!white,grow to left by=0\linewidth, width=1\linewidth, boxsep=0.5mm, arc=3mm, left=4mm, right=4mm, top=2mm, bottom=2mm]
Normalization $\LN(\cdot)$ prevents both variance collapse and growth by 
rescaling the input to each sublayer so that the variance remains $\mathcal O(1)$ across layers.
\end{tcolorbox}
\end{lemma}

Residual-free networks suffer variance collapse, while residual networks exhibit variance explosion in the absence of normalization, even when weights are variance preserving (proof in \Cref{app:variance_layer}). 
While this lemma is not used as a quantitative prediction for trained normalized transformers, it motivates the use of $\LN(\cdot)$ to maintain the variance and to condition the kurtosis analysis on normalized inputs and focus on higher-order statistics. All proofs are provided in \Cref{app:kurtosis}.

\begin{figure}
    \centering
    \includegraphics[width=\linewidth]{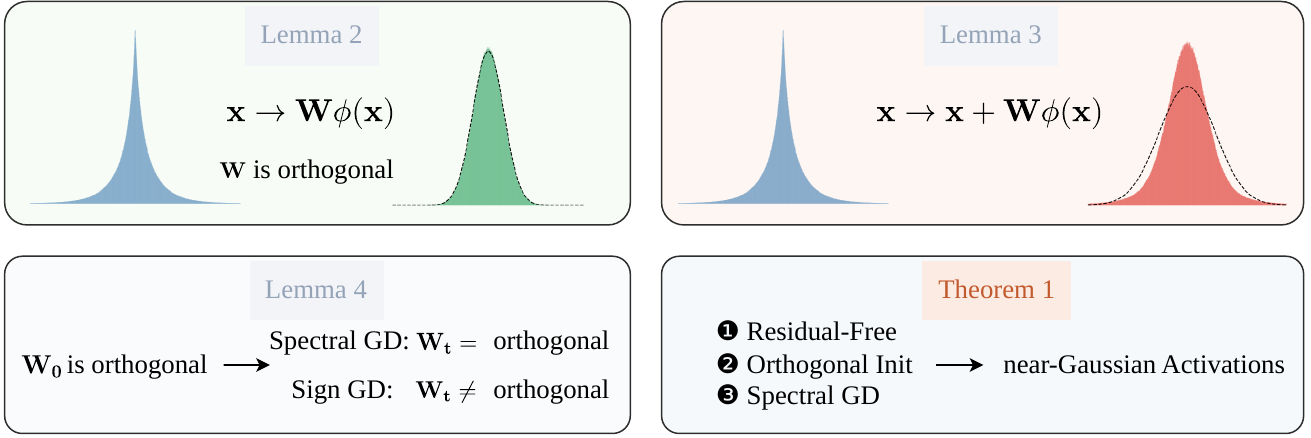}
    \caption{\textbf{Mechanism of Gaussianization vs accumulation in transformers.} Residual-free dense mixing contracts excess kurtosis $\gamma$ and drives activations toward Gaussianity (\Cref{lem:ex_kurt}), while residual addition preserves non-Gaussian components (\Cref{lem:skip_kurt}). When weight matrices remain near-isometric during training, for example, under orthogonal initialization and spectral updates (\Cref{lem:gram_drift}), residual-free models preserve near-Gaussian activations, whereas residual ones tend to be heavy-tailed (\Cref{thm:kurtosis_main}). 
    This difference explains the robustness gap in low-bit quantization.
}
    \label{fig:theory}
\end{figure}

\textbf{Excess kurtosis as a measure of non-Gaussianity.}
We use excess kurtosis as a scalar proxy for deviation from Gaussianity. 
While excess kurtosis does not uniquely characterize
a distribution, it provides a simple and tractable measure of tail
heaviness and deviation from Gaussian fourth moments. We therefore use it together
with negentropy and histograms in the experiments.
\Cref{fig:theory} outlines our analysis and summarizes the two competing mechanisms that govern the activation statistics.
Dense mixing without residual connections acts as a \emph{Gaussianizing} operator, contracting excess kurtosis. 
In contrast, residual addition introduces an \emph{accumulating} path that preserves and can amplify non-Gaussian components. The extent to which these effects manifest in deep networks further depends on whether the weight matrices remain close to isometric during training.

\begin{definition}[Excess kurtosis]
Let $v$ be a random variable with mean $\mu$ and variance $\sigma^2$. Its excess kurtosis is
$\gamma_v =
\frac{\mathbb{E}[(v-\mu)^4]}{\bigl(\mathbb{E}[(v-\mu)^2]\bigr)^2} - 3.
$
\label{def:ex_kurt}
\end{definition}

For standardized variables, $\gamma_v = \mathbb{E}[v^4] - 3$. Gaussian distributions satisfy $\gamma_v = 0$. Thus, values of $\gamma_v$ close to $0$ indicate that $v$ is close to Gaussian in the fourth-moment sense.
We now characterize how excess kurtosis evolves across layers, first for a single layer with residual-free connections. 

\begin{lemma}[Dense orthogonal mixing contracts excess kurtosis]
    Let $y_j = \sum_{i=1}^d W_{ij} \phi(x_i)$ where $\rx \in \mathbb{R}^d$ has independent coordinates and each random variable $x_i$ is centered and has isotropic variance similar to the output from $\LN(\cdot)$, and $\rmW \in \mathbb{R}^{d \times d}$. Assume further that each $x_i$ has the same excess kurtosis $\gamma_{x}$ and $\rmW$ is dense orthogonal. Then excess kurtosis of $y_j$, $\gamma_y$ is
    $$\gamma_y = \mathcal{O}\lp \frac{\gamma_x}{d} \rp.$$
    \label{lem:ex_kurt}
\end{lemma}
First, the independence assumptions should be interpreted as mean-field approximations to the post-normalization statistics since normalizations introduce dependencies across coordinates \citep{xiong2020layer}.
\Cref{lem:ex_kurt} is a fourth-moment version of a central-limit effect. When no coordinate dominates the mixing row, and cross-coordinate dependence is weak, the fourth cumulants are averaged across many coordinates, yielding a $1/d$ contraction. 
For $\relu$/$\gelu$, the coordinatewise excess kurtosis is \(O(1)\) under Gaussian inputs, and for $\softmax$, this remains true only in a non-saturated regime, which is maintained by orthogonal $\rmQ$ and $\rmK$, and temperature scaling (proof in \Cref{app:skipless}).
While Gaussian weights can also produce a local shrinkage effect, orthogonal initialization provides stronger geometric control over the mixing process, which is crucial for maintaining this behavior across layers as shown in \Cref{app:init_analysis}.

\begin{lemma}[Residual addition prevents $1/d$ excess kurtosis contraction]
\label{lem:skip_kurt}
Let the residual update be $y_j = x_j + \sum_{i=1}^d \phi(x_i) W_{ij}$, where $\rx \in \mathbb{R}^d$ has independent coordinates, and each random variable $x_i$ is centered and has isotropic variance similar to the output from $\LN(\cdot)$, and $\rmW \in \mathbb{R}^{d \times d}$ has i.i.d.\ Gaussian entries $\mathcal{N}(0,\frac{1}{d})$ independent of $\rx$.
Then the excess kurtosis of $y_j$ is $$\gamma_{y} = \mathcal{O}\lp \gamma_x + \frac{1}{d} \rp.$$
\end{lemma}

The key difference from \Cref{lem:ex_kurt} is the presence of the identity path. The residual update directly preserves the original coordinate $x_j$, and does not provide the same $1/d$ averaging
effect that drives Gaussianization. As a result, excess kurtosis does not contract across layers and instead remains on the order of the input kurtosis.  While we state the result for Gaussian $\rmW$ for analytical convenience, similar behavior holds for more general weight matrices (proof in \Cref{app:skip}).

\Cref{lem:ex_kurt,lem:skip_kurt} describe how kurtosis propagates under a specific weight matrix. The contraction in the residual-free case relies on maintaining effective dense mixing during training. Therefore, it is critical to analyze whether the weight matrices maintain the characteristics during training. 

\begin{lemma}[Spectral GD preserves near-isometry, whereas sign GD destroys isometry faster]
\label{lem:gram_drift}
Let $\rmW_{t-1}$ satisfy near-isometry and $\eta$ be the learning rate. Then the spectral update in \Cref{eq:spectral_GD} follows
\begin{align}
\|\rmW_{t}\rmW_{t}^\top - \rmI\|_2 \le \mathcal{O}(\eta),
\end{align}
and the sign update in \Cref{eq:sign_GD} follows
\begin{align}
\|\rmW_{t}\rmW_{t}^\top - \rmI\|_2 \le \mathcal{O}(\eta \sqrt{d} + \eta^2 d).
\end{align}
\end{lemma}

This stylized comparison suggests why matrix-normalized or second-order optimizers may better preserve the dense near-isometric geometry needed for kurtosis contraction. In contrast, sign-like updates suggest why coordinatewise updates such as AdamW introduce anisotropy much more rapidly, degrading the mixing structure and weakening Gaussianization even in residual-free structure. Refer to \Cref{app:gd_analysis} for proofs. Putting it together, we obtain \Cref{thm:kurtosis_main}.



\begin{theorem}[Near-Gaussian activations]
\label{thm:kurtosis_main}
Under the independence, dense-mixing, and near-isometry assumptions above, the combination of \textbf{residual-free + orthogonal initialization + spectral-type optimizer} admits near-Gaussian activations layerwise. 
The qualitative regimes predicted by the analysis are summarized below, where \cmarkcolor \, indicates near-Gaussian activations and \xmarkcolor \, indicates non-Gaussian activations. 

\centering
\begin{tabular}{llcc}
\toprule
\textbf{Architecture} & \textbf{Optimizer} & \textbf{Gaussian init.} & \textbf{Orthogonal init.} \\
\midrule
Residual       & Sign GD      & \xmarkcolor & \xmarkcolor \\
Residual       & Spectral GD  & \xmarkcolor & \xmarkcolor \\
Residual-free  & Sign GD      & \xmarkcolor & \xmarkcolor \\
Residual-free  & Spectral GD  & \xmarkcolor & \cmarkcolor \\
\bottomrule
\end{tabular}

\end{theorem}

Our analysis identifies a regime in which residual-free transformers can maintain near-Gaussian activations throughout training. This regime requires \emph{both} dense orthogonal mixing at initialization and optimization dynamics that approximately preserve this geometry, such as spectral or second-order updates. Either component alone is insufficient. In contrast, residual architectures generally lack the same Gaussianizing mechanism, since the skip path preserves higher-order moments. 
This architectural difference helps explain why the two model families behave
differently under simple uniform low-bit quantization.

\begin{tcolorbox}[colback=yellow!5!white,grow to left by=0\linewidth, width=1\linewidth, boxsep=0.5mm, arc=3mm, left=4mm, right=4mm, top=2mm, bottom=2mm]
\emph{Near-Gaussian activation is a three-part recipe: Residual-free, orthogonal initialization, and second-order optimization. 
}
\end{tcolorbox}
\section{Practical Recipe to Train Deep Residual-Free Transformers}
\label{sec:method}

While residual-free architectures offer favorable activation statistics for compression, they are typically substantially harder to optimize than their residual counterparts. Residual connections were introduced precisely to stabilize training in deep networks \citep{he2016deep,srivastava2015training}, and removing these connections can lead to poor signal propagation, slower convergence, and degraded final performance. 
A long line of work has therefore studied how to train deep networks without explicit residual connections and makes it clear that residual-free training is possible only when signal propagation is carefully controlled through appropriate initialization \citep{saxe2014exact,xiao2018dynamical} and parameterization \citep{hardt2017identity,zagoruyko2017diracnets,he2015delving}.  While \citep{ji2026cutting} achieved competitive performance for residual-free transformers with careful initialization of attention weights that maintain a well-conditioned Jacobian, it still remains a challenge to train very deep residual-free transformers.
These works show that the trainability gap can be narrowed, but also that residual-free training requires careful control of the network geometry. 

In this section, we develop a simple and effective recipe for training residual-free transformers, building on the insights we derived through our analysis in \Cref{sec:theory}. 
The key challenge is to preserve the Gaussianizing effect of dense mixing while maintaining stable optimization in the absence of residual connections. We achieve this through three components: $(i)$ orthogonal initialization to preserve isotropic mixing, $(ii)$ spectral or second-order optimization to maintain this geometry during training, and $(iii)$ depth-aware scaling of attention temperature to stabilize signal propagation across layers.
The goal is to maintain the well-conditioned, isotropic mixing needed for trainability while retaining the near-Gaussian activations that make residual-free models quantization-friendly. 
We note that this is not intended to be a unique strategy, but we found it to be effective in our experiments.

\textbf{$(i)$ Orthogonal initialization.}
We begin with an initialization that ensures each layer performs well-conditioned dense mixing.
Following \citep{ji2026cutting}, we initialize the $\mathbf{W}^\text{V}, \mathbf{W}^\text{O}$ to be scaled orthogonal matrix, such that $\mathbf{W}^\text{V} \mathbf{W}^\text{O} \approx \alpha_{VO} \mathbf{I}$, where $\alpha_{VO}$ is a hyperparameter. We initialize $\mathbf{W}^\text{U}$ and $\mathbf{W}^\text{D}$ to be scaled-corrected uniform orthogonal (SUO) distribution \citep{martens2021rapid}, such that $\mathbf{W}^\text{U}\mathbf{W}^{\text{U}^\top} \approx \alpha_{U}\mathbf{I}$ and $\mathbf{W}^\text{D}\mathbf{W}^{\text{D}^\top} \approx \alpha_{D}\mathbf{I}$ . For attention, we initialize each head with $\mathbf{W}^\text{Q}_i \approx  \mathbf{W}^\text{K}_i$ to be scaled semi-orthogonal matrix, such that  $\mathbf{W}^\text{Q}_i  \mathbf{W}^{\text{K}^\top}_i \approx \alpha_{QK} \mathbf{I}$, where $\alpha_{QK}$ is a hyperparameter. 
This ensures that, at initialization, each layer performs well-conditioned dense mixing, which is critical for maintaining stable signal propagation in the absence of residual paths.

\textbf{$(ii)$ Spectral or second order optimization.} 
We pair orthogonal initialization with spectral-type optimizations that preserve isotropy during training,  
whereas sign-like updates introduce anisotropy at a rate that grows with dimension (\Cref{lem:gram_drift}). Since the Gaussianizing effect in \Cref{lem:ex_kurt} relies on dense, well-conditioned mixing, preserving near-isometry is essential for maintaining stable activation statistics. In practice, we therefore use Muon, SOAP, KL Shampoo optimizers, which empirically show to maintain near-Gaussianity of activations, as discussed in \Cref{sec:exp}. 

\textbf{$(iii)$ Depth-aware scaling of attention temperature.}
Even with appropriate initialization and optimization, residual-free models require careful scaling to prevent signal attenuation or amplification across layers. Therefore, we introduce a scaling that modifies the attention to $\softmax\!\left(\frac{\tau_\ell}{\sqrt{d}} \rmQ\rmK^\top\right)$, where $\tau_\ell = \beta^{-\ell}$ with $\beta > 1$ is a depth-dependent temperature scaling.
This controls the sharpness of attention distributions and prevents saturation or collapse of attention patterns as depth increases. 

Together, these components provide a simple recipe for training residual-free transformers. In the experiments that follow in \Cref{sec:exp}, we show that this combination enables residual-free transformers to retain activation statistics that are substantially more favorable for low-precision quantization. 

\section{Experimental Results}
\label{sec:exp}

We evaluate both residual and residual-free transformers on language tasks. After training, we quantize model parameters and activations from BF16/full precision to low-bit integer formats and report downstream accuracy. 
Detailed experimental configurations are provided in \Cref{app:exp_details}.
We pretrain all models on FineWeb-Edu \citep{penedo2024fineweb} for 50k steps using a batch size of 512 sequences, sequence length of 512, and the next-token prediction objective. 
We evaluate the resulting models on eight commonsense reasoning benchmarks: \textit{arc\_c}, \textit{arc\_e} \citep{clark2018think}, \textit{boolq} \citep{clark2019boolq}, \textit{hellaswag} \citep{zellers2019hellaswag}, \textit{openbookqa} \citep{mihaylov2018can}, \textit{piqa} \citep{bisk2020piqa}, \textit{social\_iqa} \citep{sap2019social}, and \textit{winogrande} \citep{sakaguchi2021winogrande}.  Residual transformers are trained using AdamW and KL Shampoo, while residual-free transformers are trained primarily with KL Shampoo and orthogonal initialization unless otherwise specified. All models are trained under comparable settings to isolate architectural effects. We first apply simple uniform n-bit per-channel quantization of weights, and then m-bit per-tensor quantization of activations, and denote it as WnAm. We report accuracy task-wise and averaged across tasks with more results in \Cref{app:kurtosis_exp}. 
Code is provided in \url{https://github.com/yipingji/residual-free-quantization}.

\begin{figure}[t]
    \centering
    \includegraphics[width=0.32\linewidth]{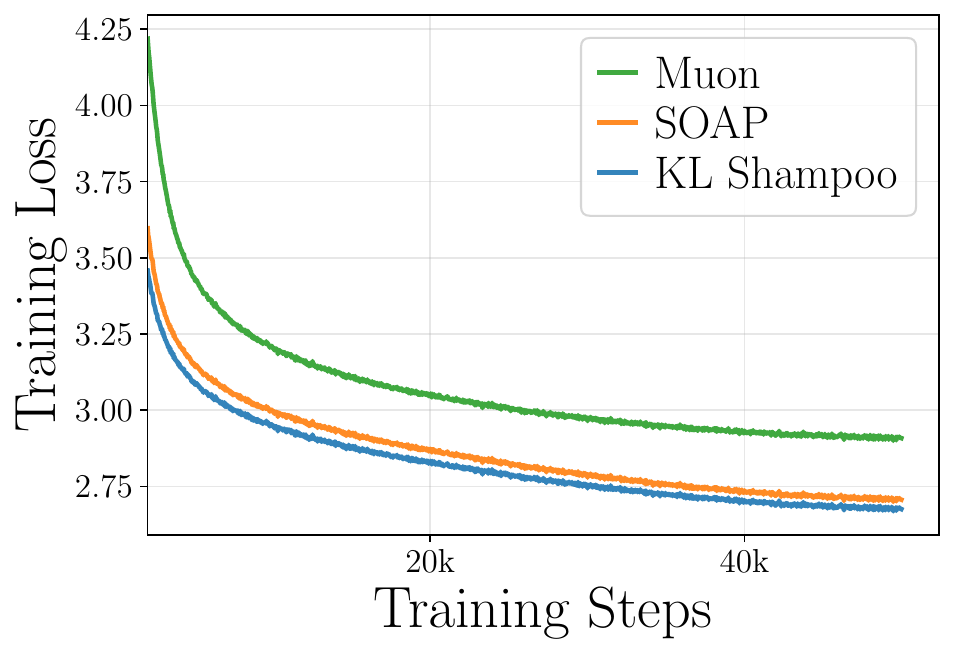}
        \includegraphics[width=0.33\linewidth]{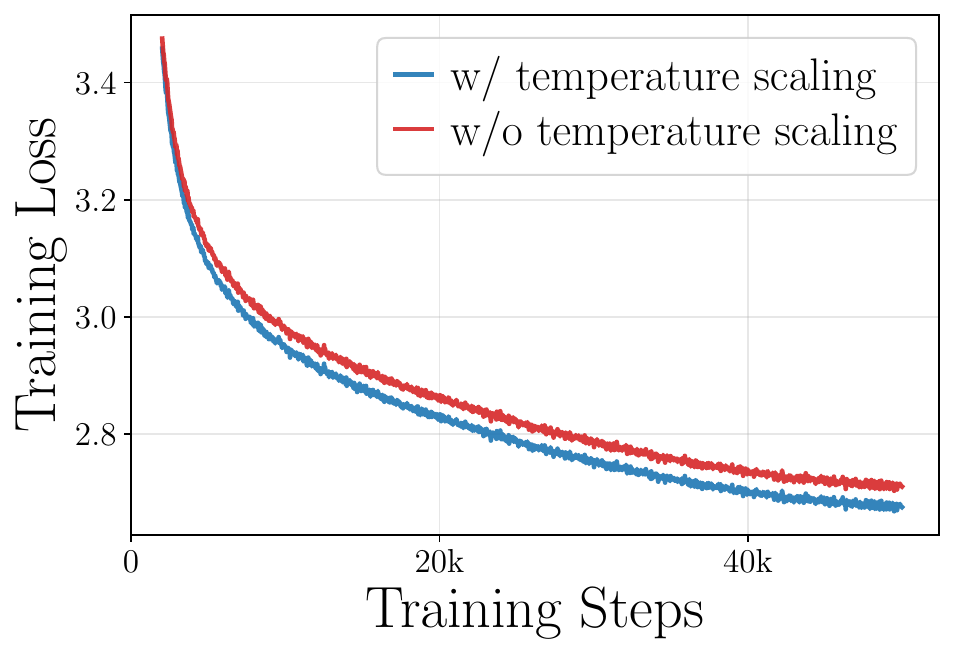}
        \includegraphics[width=0.33\linewidth]{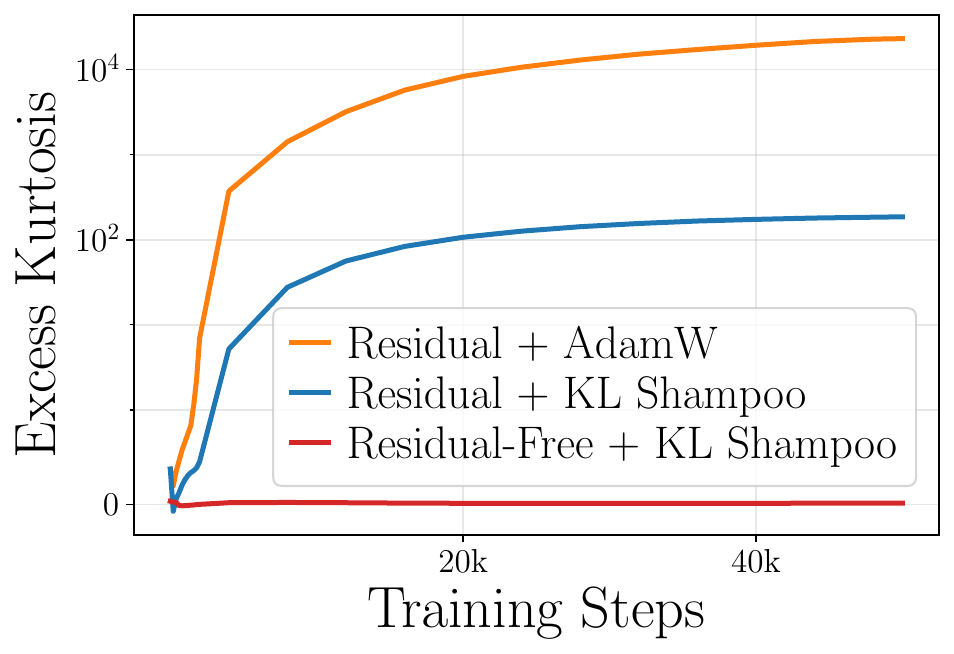}
    \vspace{-0.1cm}
    \caption{Comparison of optimizers in Residual-Free transformers (\textbf{Left}). Temperature scaling improves training convergence (\textbf{Middle}). Comparison of excess kurtosis over training(\textbf{Right}).}
    \label{fig:residualfree-optimizers}
    \vspace{-0.2cm}
\end{figure}

\begin{figure}[t]
    \centering
    \includegraphics[width=0.9\linewidth]{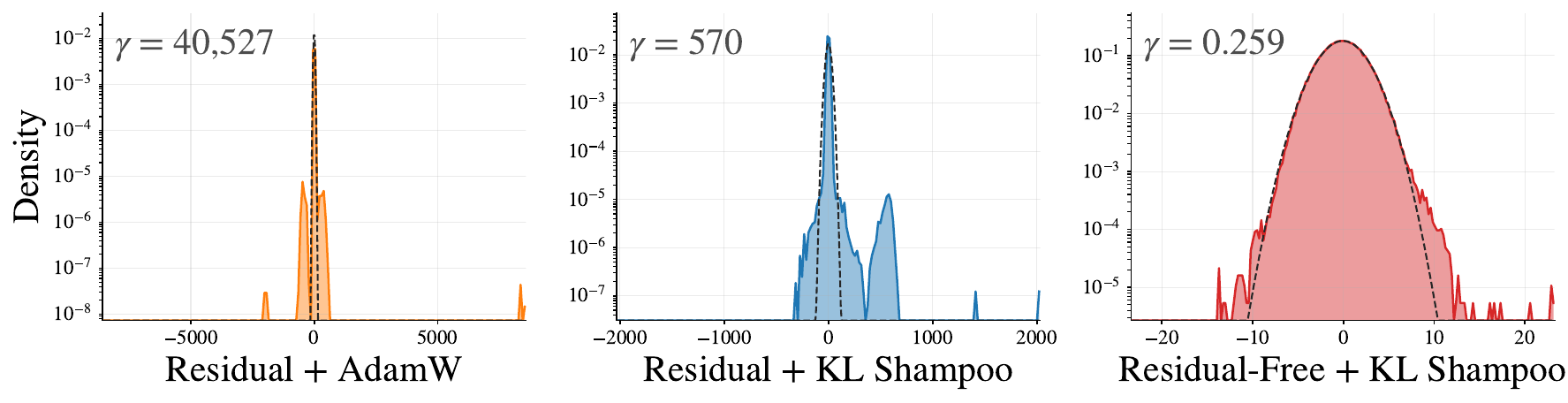}
    \caption{Activation histograms for layer 16 at 50k step showing excess kurtosis ($\gamma$). \textbf{Left and middle:} residual transformers show heavy-tailed, non-Gaussian activations with large excess kurtosis exhibiting sharp peaks and extreme outliers, leading to high quantization error. \textbf{Right:} residual-free transformers show near-Gaussian activations ($\gamma = 0.259$), thus resulting is low quantization error.}
    \label{fig:hist_l16}
    \vspace{-0.3cm}
\end{figure}

\textbf{Residual-free with KL Shampoo and temperature scaling performs the best.}
We first study the optimization behavior of residual-free transformers. In practice, we observe that residual-free transformers become increasingly difficult to train with AdamW as depth increases. 
We then evaluate several second-order optimizers, including SOAP, Muon, and KL Shampoo. As shown in Figures~\ref{fig:residualfree-optimizers} (left) and \ref{fig:residual_residual_free_loss}, all optimizers enable trainability of residual-free models, which is consistent with the theory in \Cref{sec:theory}. In addition, \emph{KL Shampoo achieves the lowest training loss, and SOAP is slightly worse, while Muon converges slowly.}
We further find that the \emph{depth-aware attention temperature scaling leads to lower training loss} (\Cref{fig:residualfree-optimizers} (middle)). 
This supports our hypothesis that controlling the sharpness of attention distributions is critical in deep residual-free networks, where repeated dense mixing can otherwise lead to signal attenuation or amplification.
In the remainder of this section, we use KL Shampoo
with temperature scaling as our default configuration for residual-free models.

\textbf{Residual-free transformers yield near-Gaussian activations.} 
We now compare the activation distribution of residual and residual-free transformers. For a fair comparison, we include residual transformers trained with both AdamW and KL Shampoo. 
As shown in \Cref{fig:residualfree-optimizers} (right), residual transformers trained with AdamW exhibit rapidly increasing excess kurtosis, which continues to grow throughout training. This indicates that activations become increasingly heavy-tailed and deviate from Gaussianity. Training with KL Shampoo significantly reduces this growth, 
but excess kurtosis remains high, thus indicating non-Gaussian activations. In contrast, \emph{residual-free transformers trained with KL Shampoo maintain excess kurtosis close to zero across the entire training trajectory}, indicating that activations remain near-Gaussian.
These results highlight that optimization alone can partially mitigate non-Gaussianity in residual models, but cannot eliminate it. Removing the residual connection 
allows dense mixing to consistently contract higher-order moments and stabilize activation distributions, consistent with our theory (\Cref{sec:theory}). 
\Cref{fig:hist_l16} shows the activation distributions where the residual-free transformer exhibits a near-Gaussian distribution with excess kurtosis close to $0$, whereas both residual models produce sharply peaked, heavy-tailed distributions with extreme outliers, even with KL Shampoo.
This confirms that KL Shampoo mitigates but does not remove the non-Gaussian activation structure induced by residual architectures. More analyses across training are in \Cref{app:kurtosis_exp} including complex learned residual aggregation in \Cref{app:learned_residual_aggregation}.

\textbf{Residual-free transformers are substantially quantization robust.}
We next evaluate how these differences in activation distribution actually affect quantization. 
\Cref{tab:quant_compare} shows that, at full precision (BF16), the residual-free model incurs a small performance drop, and Residual + KL Shampoo slightly outperforms the others. 
As precision decreases, the gap between architectures becomes pronounced. The Residual + AdamW model degrades rapidly, with accuracy dropping from $0.449$ at BF16 to 
$0.369$ at W8A8 ($-8\%$). This behavior is consistent with its heavy-tailed activation distributions, which introduce large quantization errors. 
Using KL Shampoo appears to improve robustness for residual models in reasonably high-bit quantization (W8A12). However, note that it degrades substantially from $0.452$ to $0.348$ when aggressively quantized to W8A6. Thus, it shows that KL Shampoo can reduce activation non-Gaussianity, though it does not eliminate it.
In contrast, \emph{the residual-free transformer is consistently the most robust across all low-precision settings.} While slightly weaker at BF16, it surpasses residual models under quantization, achieving higher average accuracy by $+2.7\%$ points at W8A8 and $+6.6\%$ points at W8A6. Notably, its performance degrades much more slowly as precision decreases, indicating that it is more compatible with this uniform low-bit quantization protocol.
Additional results at a higher precision of W16 are provided in \Cref{app:kurtosis_exp}.
We also show our results are robust to an outlier-aware post-training quantization method such as SmoothQuant \citep{xiao2023smoothquant} in \Cref{app:smoothquant}.


\begin{table}[t]
\caption{\textbf{Residual-free transformers are substantially quantization robust.} Zero-shot performance (average accuracy) comparison across models under BF16, W8A12, W8A8, and W8A6.}
\centering
\resizebox{\textwidth}{!}{%
\begin{tabular}{@{}l|l|cccccccc|c@{}}
\toprule
\textbf{Quant.} & \textbf{Model} & \textbf{arc\_c} & \textbf{arc\_e} & \textbf{boolq} & \textbf{hellaswag} & \textbf{openbookqa} & \textbf{piqa} & \textbf{social\_iqa} & \textbf{winogrande} & \textbf{avg} \\ \midrule
\multirow{3}{*}{BF16}  & Residual + AdamW           & 0.257 & 0.488 & 0.583 & 0.370 & 0.330 & 0.653 & 0.383 & 0.524 & 0.449 \\
                       & Residual + KL Shampoo      & 0.265 & 0.513 & 0.582 & 0.385 & 0.312 & 0.650 & 0.396 & 0.514 & \textbf{0.452} \\
                       & Residual-Free + KL Shampoo & 0.257 & 0.463 & 0.606 & 0.343 & 0.316 & 0.616 & 0.371 & 0.511 & 0.435 \\ \midrule
\multirow{3}{*}{W8A12} & Residual + AdamW           & 0.232 & 0.334 & 0.500 & 0.295 & 0.264 & 0.526 & 0.347 & 0.521 & 0.378 \\
                       & Residual + KL Shampoo      & 0.265 & 0.510 & 0.596 & 0.377 & 0.318 & 0.641 & 0.387 & 0.526 & \textbf{0.453} \\
                       & Residual-Free + KL Shampoo & 0.252 & 0.450 & 0.589 & 0.345 & 0.318 & 0.617 & 0.366 & 0.524 & 0.433 \\ \midrule
\multirow{3}{*}{W8A8}  & Residual + AdamW           & 0.273 & 0.284 & 0.507 & 0.263 & 0.272 & 0.500 & 0.346 & 0.510 & 0.369 \\
                       & Residual + KL Shampoo      & 0.222 & 0.401 & 0.578 & 0.311 & 0.270 & 0.565 & 0.349 & 0.521 & 0.402 \\
                       & Residual-Free + KL Shampoo & 0.247 & 0.448 & 0.590 & 0.344 & 0.292 & 0.611 & 0.370 & 0.529 & \textbf{0.429} \\ \midrule
\multirow{3}{*}{W8A6}  & Residual + AdamW           & 0.279 & 0.262 & 0.619 & 0.263 & 0.272 & 0.503 & 0.326 & 0.515 & 0.380 \\
                       & Residual + KL Shampoo      & 0.277 & 0.263 & 0.380 & 0.258 & 0.244 & 0.505 & 0.348 & 0.513 & 0.348 \\
                       & Residual-Free + KL Shampoo & 0.243 & 0.415 & 0.553 & 0.331 & 0.296 & 0.582 & 0.366 & 0.527 & \textbf{0.414} \\ \bottomrule
\end{tabular}%
}
\vspace{-0.2cm}
\label{tab:quant_compare}
\end{table}

Additional ablations clarify the contributions of residual connectivity and the training recipe (\Cref{sec:architecture_recipe_ablations}) and compare QK normalization with depth-aware temperature scaling (\Cref{sec:qk_normalization}).
Further experiments show that the main quantization findings remain consistent across calibration seeds (\Cref{app:calibration_seed_variability}), while training stability and quantization robustness at moderate precision persist at a larger model size (\Cref{app:larger_models}).
\Cref{app:exp_details} reports training throughput and memory consumption for all configurations on four H100 GPUs. In our implementation, the residual-free model trained with KL Shampoo achieves $1.70$ steps per second, compared with $4.77$ for the residual model trained with AdamW, while using approximately $9.7\%$ less memory per GPU.
Training the residual-free model with Muon increases throughput to $4.60$ steps per second but yields a higher final loss than KL Shampoo, highlighting the tradeoff between training speed, memory consumption, and optimization performance.

\section{Discussion and Conclusion}
\label{sec:conclusion}

To the best of our knowledge, this work is the first to identify and study the quantization-relevant consequences of residual-free transformer architectures. Our central finding is that quantization difficulty is not only a property of the quantizer, but also of the activation distributions induced by the architecture. In particular, residual connections preserve and can amplify non-Gaussian activation statistics during training, producing heavy-tailed activations that are poorly matched to simple low-bit quantization. In contrast, when residual-free transformers are combined with orthogonal initialization and isometry-preserving optimization, 
their activations remain close to Gaussian across depth and training time.
We deliberately evaluate this effect using simple uniform quantization. Modern quantization pipelines often include rotations, clipping, equivalent transformations, per-channel or per-token scaling, and mixed-precision outlier handling, all of which can compensate for, and therefore obscure, the architectural source of quantization error. Our goal is not to replace such methods, but to isolate whether the transformer architecture itself produces quantization-unfriendly activation statistics. Under this diagnostic, residual-free transformers show substantially improved robustness to low-bit quantization, despite a modest full-precision performance gap.

More broadly, our results suggest that quantization robustness should be viewed not only as a post-training compression problem, but also as an architectural and optimization problem. Residual-free models close much of the trainability gap reported in prior work, extending from shallow regimes to moderately deep transformers, while producing activation statistics that are intrinsically more compressible. Important limitations remain, including the full-precision accuracy gap, scaling to greater depths and model sizes, and understanding how residual-free architectures interact with stronger quantization pipelines. Nevertheless, our findings provide a step toward quantization-friendly foundation models and suggest that future work should jointly consider architecture, optimization, and compression rather than treating them in isolation.



\section*{Acknowledgments}
Mahalakshmi Sabanayagam, Hemanth Saratchandran and Simon Lucey acknowledge support from the Commonwealth Bank of Australia through the CommBank Centre for Foundational AI Research. This funding was essential for the completion of the research
described in this publication.

{
\bibliographystyle{abbrvnat}   
\bibliography{references}
}

\newpage
\appendix
\section*{Supplementary Material}
\tableofcontents
\newpage
\section{Related Work}
\label{app:related_work}

\textbf{Residual-free transformers.}
Prior work \citep{he2023deep} has investigated training transformers without residual connections or normalization by modifying the self-attention block. Based on their observation that residual-free transformers suffer from rank collapse, where the kernel matrix converges in depth to have rank 1, they modified the Self-Attention Block to maintain well-behaved kernels at initialization. However, their techniques require 5 times more training steps to achieve a comparable training loss than the residual-based methods. Further, \cite{ji2026cutting} showed that residual-free transformers can be trained as efficiently as their residual-based counterparts through a principled initialization of the self-attention block alone, without any architectural modification. Their approach is also hardware-friendly and remains compatible with optimizations such as Flash Attention \cite{dao2022flashattention, dao2023flashattention2}. They demonstrated the benefits of residual-free architectures in the vision domain, where learned features are more abstract and semantically consistent. However, their method still struggles to scale to deeper networks, for example, beyond 12 layers.

\textbf{Training deep networks without residuals.}
The challenge of training very deep networks without explicit skip connections predates transformers. Highway networks and residual networks were introduced to address optimization difficulties in depth \citep{srivastava2015training,he2016deep}. Other approaches attempt to recover trainability through parameterization or initialization, including identity and Dirac parameterizations \citep{hardt2017identity,zagoruyko2017diracnets}, orthogonal initialization and dynamical-isometry analyses \citep{saxe2014exact,pennington2017resurrecting,schoenholz2017deep,xiao2018dynamical}. These works primarily focus on optimization, signal propagation, or representation learning. Our focus is different, and we ask how the presence or absence of residual paths affects activation distributions and their suitability for low-bit quantization.

\textbf{Activation outliers and transformer quantization.}
A large body of work has shown that transformer quantization is often limited by activation outliers and high dynamic range rather than by weights alone. LLM.int8() identifies systematic emergent outlier features in large language models and handles them through mixed-precision decomposition \citep{dettmers2022gpt3}. SmoothQuant observes that activations are harder to quantize than weights and migrates activation outlier difficulty into weights through an equivalent transformation \citep{xiao2023smoothquant}. Outlier Suppression and Outlier Suppression+ mitigate activation outliers through shifting and scaling transformations \citep{wei2022outlier,wei2023outlier}. \citet{bondarenko2021understanding} show that transformer activations have high dynamic range and structured outliers that make low-bit fixed-point quantization difficult. \citet{bondarenko2023quantizable} further connect strong activation outliers to attention heads that learn no-op or partial updates in the residual stream. \citet{he2024understanding} study outlier features during transformer training and show that architectural and optimization choices affect their emergence, using metrics including kurtosis over activation norms. \citet{nrusimha2024mitigating} show that outlier channels emerge early in language model training, occur more frequently in layers with residual streams, and can be mitigated through activation regularization. These works motivate our focus on architectural sources of quantization difficulty. Our contribution is to isolate, in controlled residual versus residual-free comparisons, how residual addition affects activation Gaussianity and quantization error.

\textbf{Rotation and outlier-removal methods.}
Recent quantization methods also reduce activation outliers by exploiting rotational invariances. QuaRot applies orthogonal rotations to hidden states, feedforward activations, attention components, and the KV cache to obtain outlier-free 4-bit inference without changing the full-precision function \citep{ashkboos2024quarot}. SpinQuant learns such rotations to improve quantized accuracy beyond random rotations \citep{liu2025spinquant}. TurboQuant similarly uses randomized rotations and vector quantization to reduce outlier-induced distortion in low-bit compression \citep{zandieh2025turboquant}. These methods are closely related in spirit to our use of dense orthogonal mixing: all suggest that distributing information across coordinates can make activations easier to quantize. The distinction is that rotation-based methods modify or transform a trained residual model for quantization, whereas our work studies whether quantization-friendly activation statistics can be encouraged during training through residual-free architecture, orthogonal initialization, and isometry-preserving optimization.

\textbf{Positioning of our work.}
Overall, prior work has established that activation outliers are a central obstacle for transformer quantization and has proposed increasingly effective methods to compensate for them after or during training. In contrast, we study a source-level architectural problem and understand the mechanism that results in non-Gaussian activations in the first place.
\section{Preliminary Definitions}
\label{app:prelims}

\subsection{Notations}
We denote vectors and matrices using boldface notation. Specifically, bold lowercase letters such as $\rv$ denote vectors, and bold uppercase letters such as $\rmM$ denote matrices. 
For a vector $\rv \in \mathbb{R}^d$, we denote its $i$-th coordinate by $v_i$. For a matrix $\rmM \in \mathbb{R}^{m \times n}$, we denote its $(i,j)$-th entry by $M_{ij}$.
We use $\|\cdot\|$ to denote the Euclidean norm for vectors and the spectral norm for matrices unless otherwise specified. The identity matrix of dimension $d$ is denoted by $\rmI_d$.
For a function $f$, we use $\nabla f$ to denote its gradient. 
Throughout the paper, we consider sequences of representations $\{\rmX_\ell\}_{\ell=0}^L$, where $\ell$ indexes the layer of the network.

\subsection{Residual and Residual-Free blocks in Transformers} 
\label{app:transformers}
We write out the definitions of residual and residual-free blocks in \Cref{sec:kurtosis_analysis} for completeness.
Let $\rmX_\ell\in\mathbb{R}^{T\times d}$ denote the token representation at layer $\ell$, where $T$ is the sequence length and $d$ is the model dimension. Each transformer block consists of a self-attention map $\text{SA}(\cdot)$ and a position-wise feedforward map $\MLP(\cdot)$, both applied to normalized inputs through $\LN(\cdot)$.
For a residual transformer,
\begin{align}
    \rmY_\ell &= \rmX_{\ell-1} + \SA\!\left(\LN(\rmX_{\ell-1})\right), \nonumber \\ 
    \rmX_\ell  &= \rmY_\ell + \MLP\!\left(\LN(\rmY_\ell)\right). \label{eq:residual_app}
\end{align}
The residual-free counterpart removes the additive skip paths,
\begin{align}
    \rmY_\ell &= \SA\!\left(\LN(\rmX_{\ell-1})\right), \nonumber \\
    \rmX_\ell
    &= \MLP\!\left(\LN(\rmY_\ell)\right). \label{eq:residual_free_app}
\end{align}
This distinction, the presence or absence of the additive identity path, is the architectural feature we isolate in the analysis.
Here $\LN(\cdot)$ denotes either layer normalization \cite{ba2016layer} or RMS normalization \citep{zhang2019root}, and serves to control the scale of the input to each sublayer. The self-attention operation $\SA(\cdot)$ for any input $\rmX$ is defined as
\begin{equation}
    \text{SA}(\mathbf{X}) = \mathbf{A}\mathbf{V}\mathbf{W}^{\text{O}}, \nonumber
\end{equation}
where $\mathbf{Q} = \mathbf{X}\mathbf{W}^{\text{Q}}, \mathbf{K} = \mathbf{X}\mathbf{W}^{\text{K}}, \mathbf{V} = \mathbf{X}\mathbf{W}^{\text{V}},$
and the attention matrix is $\mathbf{A} = \softmax\big(\frac{1}{\sqrt{d}}\mathbf{Q}\mathbf{K}^\top\big)$. 
The parameter matrices 
$\mathbf{W}^{\text{Q}}, \mathbf{W}^{\text{K}}, \mathbf{W}^{\text{V}}, \mathbf{W}^{\text{O}} \in \mathbb{R}^{d \times d}$ 
are learnable. In practice, $\SA(\cdot)$ is multi-head attention with $h$ heads and has the form
$$ \SA(\rmX) =
\mathrm{Concat}\bigl(\rmA_1\rmV_1,\ldots,\rmA_h\rmV_h\bigr)\rmW^{O},
$$
where $\rmA_i$ are attention matrices with projection matrices $\mathbf{W}^{\text{Q}}_i, \ \mathbf{W}^{\text{K}}_i, \ \mathbf{W}^{\text{V}}_i 
    \in \mathbb{R}^{d \times d_h},  d_h = \tfrac{d}{h}$ and $\rmW^O$ is the output projection.
The MLP block has the form
$$ \MLP(\rmX) = \phi(\rmX\rmW^U)\rmW^D,$$
where $\phi$ is a pointwise nonlinearity such as $\gelu$, and $\mathbf{W}^{\text{U}} \in \mathbb{R}^{d \times d_f}$ and 
$\mathbf{W}^{\text{D}} \in \mathbb{R}^{d_f \times d}$ are the up-projection  and down-projection learnable matrices, respectively, $d_f$ is the hidden dimension (typically $d_f = 4d$).
Thus, at the level of a single coordinate or token, each sublayer repeatedly applies a nonlinear transformation and mixes coordinates through matrices such as $\rmW^O$ and $\rmW^D$.

\subsection{Initializations}
\paragraph{Gaussian initialization.}
In standard deep learning frameworks such as \texttt{PyTorch}, linear layers are typically initialized using variance-preserving schemes such as Xavier (Glorot) \citep{glorot2010understanding} or Kaiming \citep{he2015delving} initializations. 
For example, for a weight matrix $\rmW \in \mathbb{R}^{m \times n}$, Xavier initialization samples
$W_{ij} \sim \mathcal{N}\left(0, \frac{2}{m+n}\right),$
while Kaiming initialization (for ReLU activations) uses
$W_{ij} \sim \mathcal{N}\left(0, \frac{2}{n}\right).$
These schemes aim to preserve the variance of activations across layers at initialization.

\paragraph{Orthogonal initialization.}
The weight matrix $\rmW \in \mathbb{R}^{d \times d}$ is initialized such that $\rmW \rmW^\top = \rmI_d$ \citep{saxe2014exact}.
For rectangular matrices, one constructs a semi-orthogonal matrix whose columns (or rows) are orthonormal. Often, a scaling factor is applied:
$\rmW = \alpha \tilde{\rmW}, \quad \tilde{\rmW}^\top \tilde{\rmW} = \rmI,$
where $\alpha$ controls the variance of activations.
Orthogonal initialization preserves norms and helps maintain dynamical isometry, which is beneficial for training very deep networks.

\subsection{Optimizers}
\paragraph{Sign gradient descent.}
Sign gradient descent is a simplified optimization method that updates parameters using only the sign of the gradient \citep{bernstein2018signsgd}. It serves as a useful proxy for adaptive methods such as Adam \citep{kingma2015adam}, AdamW \citep{loshchilov2018decoupled}. 
Given parameters $\rmW$ and loss $\mathcal{L}(\rmW)$, the update rule is
\begin{align}
    \rmW_{t+1} = \rmW_t - \eta \, \text{sign}\big(\nabla \mathcal{L}(\rmW_t)\big), \label{eq:signGD}
\end{align}

where $\eta > 0$ is the learning rate and the sign function is applied elementwise.
This method captures the directionality of adaptive optimizers while discarding magnitude information.

\paragraph{Spectral gradient descent.}
Spectral gradient descent is a simplified model of modern second-order optimization methods such as Shampoo \citep{gupta2018shampoo}, Muon \citep{jordan2024muon}, SOAP \citep{vyassoap}, and KL-Shampoo \citep{lin2025understanding}. These methods precondition the gradient using curvature information.
Let $\rmW \in \mathbb{R}^{m \times n}$ be a weight matrix with gradient $\rmG = \nabla \mathcal{L}(\rmW)$. A generic spectral gradient descent updates $\rmW$ as
$$
\rmW_{t+1} = \rmW_t - \eta \, \rmP_t^{-1/4}\, \rmG_t \,\rmQ_t^{-1/4},
$$
where $\rmP_t, \rmQ_t$ are positive definite preconditioning matrices.
In matrix form, a common approximation is
$\rmP_t = \rmG_t\rmG_t^\top$ and $\rmQ_t = \rmG_t^\top\rmG_t$.
Considering the eigenvalue decomposition of $\rmG_t = \rmU \rmS \rmV^\top$, thus $\rmG_t\rmG_t^\top = \rmU \rmS^2 \rmU^\top$ and $\rmG_t^\top \rmG_t = \rmV \rmS^2 \rmV^\top$  hence the update is
\begin{align}
    \rmW_{t+1} &= \rmW_t - \eta \, \rmU \rmS^{-1/2} \rmS \rmS^{-1/2} \rmV^\top \nonumber \\
    &= \rmW_t - \eta \,\rmU\rmV^\top \label{eq:spectralGD}
\end{align}
This effectively rescales updates along different eigendirections, normalizing by curvature. Spectral gradient descent captures the key behavior of matrix-preconditioned optimizers while remaining analytically tractable.
In particular, it preserves invariances under linear reparameterizations and tends to equalize progress across different directions in parameter space.

\section{Analysis of Layer-Wise Activations of Residual and Residual-Free Transformers}
\label{app:kurtosis}

We start by analyzing the need for normalization $\LN(\cdot)$, as it controls the second moments in both residual and residual-free transformers.

\subsection{Normalization is Necessary to Control Variance}
\label{app:variance_layer}
We consider the layerwise recursions defined in \Cref{app:transformers} and analyze the behavior of the variance across layers. We expand and include further technical details in \Cref{lem:variance} as follows: 

\begin{lemma}[Normalization controls variance collapse and growth]
\label{lem:variance_app}
Consider the residual and residual-free update in \Cref{eq:residual,eq:residual_free} without normalization, that is, $\LN(\rmX) = \rmX$.
Assume coordinates of $\rmX_{\ell}$ are centered with variance $q_\ell$, and weight matrices are variance-preserving. Then,
\begin{itemize}[leftmargin=*]
    \item \textbf{Residual-free case:}
    The variance of $\rmX_{\ell+1}$,  $q_{\ell+1} \approx \mathbb{E}_{z \sim \mathcal{N}(0,q_\ell)}[\phi(z)^2],$
    which yields $q_{\ell+1} = c q_\ell$ with $c < 1$ for common nonlinearities that contracts the input. Hence, the variance of deep layers $q_L \to 0$ exponentially with depth.
    \item \textbf{Residual case:} 
    The variance of $\rmX_{\ell+1}$,
    $q_{\ell+1} \approx (1+c) q_\ell$ with $c>0$ under weak correlation, so $q_L$ grows exponentially with depth.
\end{itemize}
\end{lemma}

\begin{proof}[Proof of \Cref{lem:variance_app}]
We give the proof for a single normalized-width sublayer, since the same argument
applies to both the attention output projection and the MLP projection under the
variance-preserving approximation. Write the sublayer as
$$ F_\ell(\rmX_\ell)=\rmW_\ell \phi(\rmX_\ell),$$
where $\rmW_\ell$ is independent of $\rmX_\ell$ and variance-preserving. For a
single coordinate, this means
$$ \mathbb{E}\bigl[(F_\ell(\rmX_\ell))_i^2 \mid \phi(\rmX_\ell)\bigr] =
\frac{1}{d}\sum_{j=1}^d \phi(X_{\ell,j})^2. $$
Taking the expectation gives
$$ \operatorname{Var}\bigl((F_\ell(\rmX_\ell))_i\bigr) \approx
\mathbb{E}_{z\sim \mathcal N(0,q_\ell)}[\phi(z)^2]. $$
Thus, if the coordinates of $\rmX_\ell$ are centered with variance $q_\ell$, the
residual-free recursion satisfies
$$ q_{\ell+1} \approx
\mathbb{E}_{z\sim \mathcal N(0,q_\ell)}[\phi(z)^2].$$
For common pointwise nonlinearities such as $\relu$ and $\gelu$, and vectorwise linearities like $\softmax$, the variance contracts as these functions contract the input value. Therefore, 
$$ q_{\ell+1} = c q_{\ell},$$ 
where $c<1$. Therefore, by induction,
$$q_\ell = c^{\ell} q_0.$$
Exemplary, in the case of ReLU, $c=\frac{1}{2}$, which can be derived using the symmetry of the Gaussian. 
Hence, without normalization or gain correction, the variance of a residual-free
network collapses exponentially with depth.
For the residual recursion,
$$ \rmX_{\ell+1} = \rmX_\ell + F_\ell(\rmX_\ell).$$
For one coordinate,
\begin{align*}
q_{\ell+1}
&= \operatorname{Var}(X_{\ell i}+F_{\ell i})
= \operatorname{Var}(X_{\ell i})
+ \operatorname{Var}(F_{\ell i})
+ 2\operatorname{Cov}(X_{\ell i},F_{\ell i}) \\
&= q_\ell
+ \operatorname{Var}(F_{\ell,i})
+ 2\operatorname{Cov}(X_{\ell,i},F_{\ell,i})
\end{align*}
Under the standard mean-field approximation that the residual branch is weakly correlated with the input coordinate, the covariance term is negligible, giving
$$ q_{\ell+1} \approx q_\ell + \mathbb{E}_{z\sim \mathcal N(0,q_\ell)}[\phi(z)^2]
= (1+c)q_\ell.
$$
Hence, in this case, the variance grows exponentially with depth.
\end{proof}

\begin{remark}
Layer normalization prevents both effects by re-centering and rescaling the input
to each sublayer, whereas RMSNorm only rescales the input to each sublayer. Both ensure that the variance of the inputs to $\SA(\cdot)$ and $\MLP(\cdot)$ remains $\mathcal O(1)$ across layers.
\end{remark}

\subsection{Auxiliary Lemmas on Excess Kurtosis}
We now introduce the excess kurtosis \Cref{def:ex_kurt}, which is a scalar proxy for deviation from Gaussianity \citep{pearson1905fehlergesetz}. Excess kurtosis does not characterize a distribution completely, but it is a natural and tractable summary statistic for tail-heaviness and departure from Gaussian fourth moments. 
Notice that the input to each layer $\text{SA}(\cdot)$ or $\text{MLP}(\cdot)$ is normalized using $\LN(\cdot)$ which keep the variance controlled and $\mathcal{O}(1)$.

\begin{definition}[Excess Kurtosis]
    Let $v$ be a random variable with $\expectation{}{v} = \mu$ and variance $\expectation{}{(v-\mu)^2} = \sigma^2$. Its excess kurtosis is defined as
    $$\gamma_v = \frac{\expectation{}{(v-\mu)^4}}{\lp\expectation{}{(v-\mu)^2}\rp^2} - 3.$$
    \label{def:ex_kurt_app}
\end{definition}
For a standardized random variable $\rv$, excess kurtosis is $\gamma_v = \expectation{}{v^4} - 3.$
A standard Gaussian random variable has excess kurtosis equal to $0$. Thus, values of $\gamma_v$ close to $0$ indicate that $v$ is close to Gaussian in the fourth-moment sense.

We first record several elementary facts about the propagation of excess kurtosis through linear combinations. These lemmas will be used repeatedly in the later analysis.

\begin{lemma}[Excess kurtosis of isotropic variance mixing]
    Let $y = \sum_{i=1}^d a_i u_i$ where $\ru \in \mathbb{R}^d$ has independent coordinates and each random variable $u_i$ is centered and has isotropic variance, that is, $\expectation{}{u_i}=0$ and $\expectation{}{u_i^2}=1$, and $\ra \in \mathbb{R}^d$ is a deterministic vector with unit norm $\sum_{i=1}^d a_i^2 = 1$. Assume further that for each $u_i$ has the same excess kurtosis $\gamma_{u}$. Then excess kurtosis of $y$, $\gamma_y$ is
    $$\gamma_y = \gamma_{u} \sum_{i=1}^d a_i^4.$$
    \label{lem:ex_kurt_y}
\end{lemma}
\begin{proof}
The mean of $y$ is $\sum_{i=1}^d a_i \expectation{}{u_i} = 0$ and the variance of $y$ is $\sum_{i=1}^d a_i^2 \expectation{}{u_i^2} = \sum_{i=1}^d a_i^2 = 1$. Therefore, the excess kurtosis of $y$ is
\begin{align}
    \gamma_y &= \expectation{}{\lp \sum_{i=1}^d a_i u_i \rp^4} - 3 \nonumber \\
    &= \mathbb{E} \bigg[ \sum_{i=1}^d a_i^4 u_i^4 + 4\sum_{i\ne j}^d a_i^3 u_i^3 a_j u_j + 6\sum_{i\ne j}^d a_i^2 u_i^2 a_j^2 u_j^2 +  12\sum_{i\ne j\ne k}^d a_i^2 u_i^2 a_j u_j a_k u_k + \nonumber \\ 
    &\qquad \qquad 24\sum_{i\ne j\ne k\ne l}^d a_i u_i a_j u_j a_k u_k a_l u_l \bigg] - 3 \nonumber \\
    &= \expectation{}{\sum_{i=1}^d a_i^4 u_i^4 + 6\sum_{i\ne j}^d a_i^2 u_i^2 a_j^2 u_j^2 } - 3 \nonumber \\
    &= \lp \gamma_u+3 \rp \sum_{i=1}^d a_i^4 + 6 \sum_{i\ne j}^d a_i^2 a_j^2 - 3 \qquad \qquad ; \gamma_u = \expectation{}{ u_i^4} - 3, \expectation{}{u_i^2} = 1 \,\, \forall i \label{eq:kurt_y_}
\end{align}
Since $\sum_{i=1}^d a_i^2 = 1$, squaring it will give
\begin{align}
    \lp \sum_{i=1}^d a_i^2 \rp^2 &= \sum_{i=1}^d a_i^4 + 2\sum_{i \ne j} a_i^2 a_j^2 = 1 \nonumber \\
    2\sum_{i \ne j} a_i^2 a_j^2 &= 1-\sum_{i=1}^d a_i^4 \label{eq:aiajsq}
\end{align}
Substituting \Cref{eq:aiajsq} in \Cref{eq:kurt_y_}, we obtain
\begin{align}
    \gamma_y &= \lp \gamma_u+3 \rp \sum_{i=1}^d a_i^4 + 3 \lp 1-\sum_{i=1}^d a_i^4 \rp - 3 \nonumber\\
    &= \gamma_u \sum_{i=1}^d a_i^4 \label{eq:kurt_y}
\end{align}
Thus proving the lemma.
\end{proof}

\begin{lemma}[Dense mixing implies $1/d$ excess kurtosis shrinkage]
    Consider the setting of \Cref{lem:ex_kurt_y}. Assume in addition that $\ra$ is dense, 
    then 
    $$\gamma_y = \frac{\gamma_u}{d}.$$
    \label{lem:ex_kurt_y_ortho_a}
\end{lemma}
\begin{proof}
When $\ra$ is dense unit norm, it means $\sum_{i=1}^d a_i^2 =1$ with  the scale of $a_i$ is roughly $\frac{1}{\sqrt{d}}$ for all $i$. Therefore, $$\sum_{i=1}^d a_i^4 = \frac{1}{d}.$$
The claim then follows immediately from \Cref{eq:kurt_y} as 
$\gamma_y = \frac{\gamma_u}{d}.$ Thus, if $\ru$ is non-Gaussian, then dense mixing suppresses it.
\end{proof}

\subsection{Excess Kurtosis Behavior of Different Initializations: Orthogonal vs Gaussian}
\label{app:init_analysis}
\paragraph{Excess kurtosis behavior of orthogonal initialization.}
In \Cref{lem:ex_kurt_y,lem:ex_kurt_y_ortho_a}, orthogonality guarantees unit norm, that is $\sum_{i=1}^d a_i^2=1$. The conclusion of \Cref{lem:ex_kurt_y_ortho_a} does not hold for an arbitrary orthogonal $\ra$ without the additional dense condition. To obtain $\sum_i a_i^4 = O(1/d)$, one needs $\ra$ to be dense or incoherent, for example when $\max_i |a_i| \lesssim d^{-1/2}$. The case above gives the cleanest exact statement and captures the regime relevant for \textbf{dense orthogonal initializations}.

\begin{corollary}[Near-Gaussianity after dense orthogonal mixing]
Under the assumptions of \Cref{lem:ex_kurt_y_ortho_a}, $\gamma_y = \frac{\gamma_u}{d}.$
In particular, $y$ is close to Gaussian in the fourth-moment sense whenever $\gamma_u \ll d.$
Equivalently, as $d \to \infty$, if $\gamma_u = o(d)$, then $\gamma_y \to 0$.
    
\label{cor:gaussian_y}
\end{corollary}
\begin{proof}
$y$ is close to Gaussian in the fourth-moment sense when its excess kurtosis $\gamma_y$ is close to zero.
It is straightforward to see that $\gamma_y$ is close to $0$ only if $\gamma_u$ is much smaller than $d$. The asymptotic statement follows immediately.
\end{proof}

\paragraph{Excess kurtosis behavior of default Gaussian initialization.}
A natural alternative to dense orthogonal mixing is default Gaussian initialization, for example, $a_i \sim \mathcal N(0,1/d)$. This also produces dense rows with high probability, and one might therefore expect a similar $1/d$ excess kurtosis shrinkage. Indeed, a Gaussian row is dense in the sense that $$\sum_{i=1}^d a_i^4 = \frac{3}{d}.$$
For a row normalized to unit norm, \(a/\|a\|\) is approximately uniform on the sphere, and \(\mathbb{E}\sum_i (a_i/\|a\|)^4 = 3/(d+2)\). Thus Gaussian rows are dense up to constants, and still yield \(O(1/d)\) kurtosis shrinkage. The constant differs from the ideal equal-magnitude row.
However, default Gaussian initialization is weaker than dense orthogonal initialization in two important ways. First, the row norm $\sum_{i=1}^d a_i^2$ is random for each row $i$, rather than exactly equal to one. Consequently, even if the input $\ru$ is Gaussian, the output is a scale mixture of Gaussians over different rows, producing an additional finite-width excess kurtosis term of order $1/d$. For example, if $u_i$ are independent standardized Gaussians and $a_i \sim \mathcal N(0,1/d)$, then
$$ \gamma_y = \frac{6}{d},$$
whereas a dense unit-norm (orthogonal) row gives $\gamma_y=0$ for Gaussian inputs.

Second, and more importantly for deep networks, a Gaussian matrix preserves norms only in \textbf{expectation}. Its singular values are spread, so repeated composition can create anisotropic amplification and contraction across layers. Orthogonal initialization instead gives exact norm preservation at the linear level, and therefore avoids injecting layerwise variance disorder. Thus Gaussian initialization is a plausible candidate for one-layer kurtosis shrinkage, but dense orthogonal initialization provides stronger per-layer geometric control, which is crucial in deep compositions, attention blocks, and during training.

\subsection{Excess Kurtosis Behavior of Optimizers: Spectral GD vs Sign GD}
\label{app:gd_analysis}
The key point in the excess kurtosis analysis is that dense approximately orthogonal mixing shrinks excess kurtosis and makes the activations more Gaussian. 
In this section, we analyze the excess kurtosis behavior of different optimizers from this standpoint. Spectral updates tend to maintain orthogonal mixing, whereas sign-like updates can destroy it much more rapidly. Once this dense near-isometric structure is lost, the excess kurtosis shrinkage from \Cref{lem:ex_kurt_y_ortho_a} weakens, and activations become peakier.

\paragraph{Setup.}
Let $\rmW_t \in \mathbb{R}^{d \times d}$ denote a weight matrix at iteration $t$, and let $\rmG_t = \nabla_{\rmW}\mathcal{L}(\rmW_t)$ be its gradient. We are interested in how the update rule affects the near-isometry condition
\begin{align}
\rmW_t \rmW_t^\top - \rmI = \rmE_t,
\end{align}
where $\|\rmE_t\| = \epsilon_t$ is very small.

\paragraph{Spectral GD updates.}
The spectral-gradient descent update has the form
\begin{align}
\rmW_{t+1} = \rmW_t - \eta \, \rmU_t \rmV_t^\top,
\label{eq:spectral_gd_update}
\end{align}
where $\rmG_t = \rmU_t \rmS_t \rmV_t^\top$.
\paragraph{Sign GD updates.}
The sign-gradient descent update step is
\begin{align}
\rmW_{t+1} = \rmW_t - \eta \, \sign(\rmG_t),
\label{eq:sign_gd_update}
\end{align}
where $\sign(\rmG_t)$ is applied entrywise.

\begin{lemma}[Spectral gradient descent preserves near-isometry]
\label{lem:spectral_gram_drift}
Consider the spectral update \Cref{eq:spectral_gd_update}, and assume $\rmW_t$ satisfies near-isometry. Then
\begin{align}
\|\rmW_{t+1}\rmW_{t+1}^\top - \rmI\|_2 \le \mathcal{O}(\eta).
\end{align}
\end{lemma}

\begin{proof}
Lets evaluate $\rmW_{t+1}\rmW_{t+1}^\top$ by expanding the update,
\begin{align}
\rmW_{t+1}\rmW_{t+1}^\top - \rmI &=
(\rmW_t-\eta \rmU_t \rmV_t^\top)(\rmW_t-\eta \rmU_t \rmV_t^\top)^\top - \rmI \nonumber \\
&= \rmW_t\rmW_t^\top - \eta(\rmW_t \rmV_t \rmU_t^\top + \rmU_t \rmV_t^\top \rmW_t^\top)
+
\eta^2 \rmU_t \rmV_t^\top \rmV_t \rmU_t^\top - \rmI \nonumber \\
&= \rmE_t - 2\eta \rmW_t \rmV_t \rmU_t^\top + \eta^2 \rmI \qquad ; \rmU, \rmV, \rmW \text{ are orthogonal} \nonumber
\end{align}
Now the spectral norm can be bounded as,
\begin{align}
\|\rmW_{t+1}\rmW_{t+1}^\top - \rmI\|_2 &\le \epsilon_t + 2\eta \|\rmW_t\|_2\|\rmU_t \rmV_t^\top \|_2 + \eta^2 \nonumber \\
&\le \epsilon_t + 2\eta + \eta^2 = \mathcal{O}(\eta) \nonumber
\end{align}
\end{proof}

\begin{remark}
\Cref{lem:spectral_gram_drift} shows that spectral updates can preserve near-isometry for a sufficiently small learning rate. Therefore, they can preserve the dense orthogonal mixing regime needed for the $1/d$ excess kurtosis shrinkage in residual-free networks.
\end{remark}

Now, let's analyze Sign GD updates in a similar fashion.
\begin{lemma}[Large operator norm of dense sign updates]
\label{lem:sign_spectral_norm}
Let $\rmS \in \{\pm 1\}^{d \times d}$ be a dense sign matrix. Then generically
\begin{align}
\|\rmS\|_2 = O(\sqrt{d}). \nonumber
\end{align}
\end{lemma}

\begin{proof}
A $d\times d$ random dense sign matrix has the largest singular value of order $\sqrt{d}$ with high probability. The same scaling applies to dense sign patterns. Thus, a dense entrywise sign update has operator norm $O(\sqrt{d})$ rather than $O(1)$.
\end{proof}

Substituting this into the $\rmW_{t+1} \rmW_{t+1}^\top$ expansion gives much larger drift from identity $\rmI$.

\begin{lemma}[Sign GD destroys near-isometry faster]
\label{prop:sign_gram_drift}
Consider the spectral update \Cref{eq:sign_gd_update}, and assume $\rmW_t$ satisfies near-isometry. Then
\begin{align}
\|\rmW_{t+1}\rmW_{t+1}^\top - \rmI\|_2 \le \mathcal{O}(\eta \sqrt{d} + \eta^2 d),
\end{align}
implying sign-like updates can destroy orthogonality much faster than spectral updates.
\end{lemma}

\begin{proof}
Lets evaluate $\rmW_{t+1}\rmW_{t+1}^\top$ by expanding the update,
\begin{align}
\rmW_{t+1}\rmW_{t+1}^\top - \rmI &=
(\rmW_t-\eta \sign(\rmG_t))(\rmW_t-\eta \sign(\rmG_t))^\top - \rmI \nonumber \\
&= \rmW_t\rmW_t^\top - \eta(\rmW_t \sign(\rmG_t)^\top + \sign(\rmG_t)  \rmW_t^\top)
+
\eta^2 \sign(\rmG_t) \sign(\rmG_t)^\top - \rmI \nonumber \\
&= \rmE_t +  \eta(\rmW_t \sign(\rmG_t)^\top + \sign(\rmG_t)  \rmW_t^\top)
+ \eta^2 \sign(\rmG_t) \sign(\rmG_t)^\top \nonumber
\end{align}
Now the spectral norm can be bounded as,
\begin{align}
\|\rmW_{t+1}\rmW_{t+1}^\top - \rmI\|_2 &\le \epsilon_t + 2\eta \|\rmW_t\|_2\|\sign(\rmG_t) \|_2 + \eta^2 \|\sign(\rmG_t) \|_2^2 \nonumber \\
&\le \epsilon_t + 2 \eta \sqrt{d} + \eta^2 d = \mathcal{O}(\eta \sqrt{d} + \eta^2 d) \qquad ; \|\rmW_t\|_2=1, \text{\Cref{lem:sign_spectral_norm}} \nonumber
\end{align}
\end{proof}

\begin{remark}
Compared to the $O(\eta)$ drift from spectral updates, sign GD exhibits a much larger $O(\eta\sqrt{d}+\eta^2 d)$ drift. Hence, in high dimensions, sign-like methods can quickly destroy the near-isometric dense-mixing structure, even when the initialization is orthogonal.
\end{remark}

\subsection{Excess Kurtosis Behavior in Residual Settings}
\label{app:skip}

\begin{lemma}[Residual addition preserves excess kurtosis (no shrinkage)]
\label{lem:skip_vector_kurtosis}
Let $\rx$ have i.i.d.\ centered unit variance coordinates with excess kurtosis $\gamma_x$, \ie{} $\expectation{}{x_i}=0$, $\expectation{}{x_i^2}=1$ and $\expectation{}{x_i^4}=3+\gamma_x$.
Let $\rmW \in \mathbb{R}^{d\times d}$ have i.i.d.\ Gaussian entries $\mathcal{N}(0,\frac{1}{d})$ independent of $\rx$.
Define the residual update
\begin{align}
y_j = x_j + \sum_{i=1}^d \phi(x_i) W_{ij}, \nonumber
\end{align}
for a fixed coordinate $j$. Let
\begin{align}
m_2 = \expectation{}{\phi(x)^2},
\qquad
m_4 = \expectation{}{\phi(x)^4},
\qquad
c_{22} = \expectation{}{x^2 \phi(x)^2}, \nonumber
\end{align}
where $x$ denotes a generic coordinate distributed as any $x_i$. Then $y_j$ is centered with variance
\begin{align}
\Var(y_j)=1+ m_2, \nonumber
\end{align}
and excess kurtosis
\begin{align}
\gamma_{y_j}
=
\frac{3+\gamma_x + \frac{6}{d}\big(c_{22}+(d-1)m_2\big) + \frac{3}{d^2}\big(dm_4+d(d-1)m_2^2\big)}
{(1+m_2)^2}
-3 = \mathcal{O}\lp \gamma_x + \frac{1}{d} \rp.
\label{eq:gamma_yj_vector}
\end{align}
\end{lemma}

\begin{proof}
Fix $j$ and define
\begin{align}
z_j = \sum_{i=1}^d \phi(x_i) W_{ij}. \nonumber
\end{align}
Conditioned on $\rx$, the variable $z_j$ is Gaussian with mean $0$ and variance
\begin{align}
\Var(z_j \mid \rx) = \frac{1}{d} \sum_{i=1}^d \phi(x_i)^2. \nonumber
\end{align}
Let $S = \sum_{i=1}^d \phi(x_i)^2.$ Then, conditional on $\rx$,
\begin{align}
y_j \mid \rx \sim \mathcal{N}\lp x_j, \frac{S}{d} \rp. \nonumber
\end{align}
Hence
\begin{align}
\expectation{}{y_j \mid \rx} = x_j,
\qquad
\expectation{}{y_j^2 \mid \rx} = x_j^2 + \frac{S}{d},
\qquad
\expectation{}{y_j^4 \mid \rx} = x_j^4 + 6\frac{S}{d} x_j^2 + 3\frac{S^2}{d^2}. \nonumber
\end{align}

Taking the expectation gives
\begin{align}
\expectation{}{y_j} = \expectation{}{x_j} = 0. \nonumber
\end{align}
Similarly,
\begin{align}
\expectation{}{y_j^2} &=
\expectation{}{x_j^2} + \frac{1}{d} \expectation{}{S}. \nonumber
\end{align}
Since the coordinates are i.i.d.,
\begin{align}
\expectation{}{S} = \sum_{i=1}^d \expectation{}{\phi(x_i)^2} = dm_2, \nonumber
\end{align}
and therefore
\begin{align}
\Var(y_j)=\expectation{}{y_j^2}=1+m_2. \nonumber
\end{align}

For the fourth moment,
\begin{align}
\expectation{}{y_j^4} = \expectation{}{x_j^4} + 6\frac{1}{d}\,\expectation{}{x_j^2 S} + 3\frac{1}{d^2} \expectation{}{S^2}. \nonumber
\end{align}
Now $\expectation{}{x_j^4}=3+\gamma_x$. Also,
\begin{align}
\expectation{}{x_j^2 S} &=
\expectation{}{x_j^2 \phi(x_j)^2} +
\sum_{i\neq j} \expectation{}{x_j^2 \phi(x_i)^2} \nonumber \\
&= c_{22} + (d-1)\expectation{}{x_j^2}\expectation{}{\phi(x_i)^2} = c_{22} + (d-1)m_2. \nonumber
\end{align}
Finally,
\begin{align}
\expectation{}{S^2} &=
\expectation{}{\left(\sum_{i=1}^d \phi(x_i)^2\right)^2} \nonumber \\
&= \sum_{i=1}^d \expectation{}{\phi(x_i)^4}
+ \sum_{i\neq r} \expectation{}{\phi(x_i)^2 \phi(x_r)^2} \nonumber \\
&= dm_4 + d(d-1)m_2^2. \nonumber
\end{align}
Substituting these identities into $\expectation{}{y_j^4}$, we obtain
\begin{align}
\expectation{}{y_j^4}
=
3+\gamma_x
+
\frac{6}{d}\big(c_{22}+(d-1)m_2\big)
+
\frac{3}{d^2}\big(dm_4+d(d-1)m_2^2\big). \nonumber
\end{align}
Therefore
\begin{align}
\gamma_{y_j}
=
\frac{\expectation{}{y_j^4}}{\Var(y_j)^2}-3, \nonumber
\end{align}
which gives \Cref{eq:gamma_yj_vector}.
\end{proof}

\begin{remark}
Unlike dense orthogonal mixing, the residual update of the form $\ry=\rx+\phi(\rx)\rmW$ does not average over many coordinates, and therefore does not produce a $1/d$ kurtosis shrinkage. Instead, \Cref{eq:gamma_yj_vector} shows that the excess kurtosis of $y$ remains $\mathcal{O}(\gamma_x + 1/d)$ since $m_2$ is $\mathcal{O}(1)$ (\Cref{lem:relu_excess_kurtosis}). Thus, the non-Gaussianity already present in $\rx$ is preserved up to a perturbative correction. This explains why residual connections tend to retain and accumulate non-Gaussian activations rather than Gaussianizing them.
\end{remark}

\subsection{When Residual-Free Can Remain Near-Gaussian While Residual Transformer Need Not}
\label{app:skipless}
The previous lemmas show that a dense orthogonal mixing suppresses excess kurtosis by a factor of order $1/d$. We now use this fact to analyze the qualitative difference between residual and residual-free architectures.

At a high level, the residual-free architecture repeatedly \emph{replaces} the current representation by a transformed version. Thus, if each layer contains a sufficiently dense, approximately orthogonal mixing, then any non-Gaussianity produced at intermediate coordinates is averaged out at the next mixing stage. By contrast, the residual architecture updates the representation by \emph{adding} a new transformed term to the previous state. Consequently, non-Gaussianity already present in the hidden state is carried forward instead of being discarded, and may accumulate across layers (\Cref{lem:skip_kurt}).

%
We first idealize the residual-free layer as a composition of dense linear mixing and a coordinate-wise nonlinearity. This abstraction captures both the self-attention and MLP blocks at the level needed for our kurtosis analysis.

\begin{theorem}[Excess kurtosis shrinkage in residual-free layers]
\label{prop:skipless_kurt_contract}
Consider a residual-free network with hidden states $\{\rx_\ell\}_{\ell \ge 0}$ satisfying $\rx_{\ell+1} = \rmW_\ell \, \phi_\ell(\LN(\rx_\ell)),$
where:
\begin{enumerate}
    \item $\rmW_\ell \in \mathbb{R}^{d \times d}$ satisfies
    $\sum_{i=1}^d \lp \rmW_\ell \rp_{ji}^2 = 1$, and $\sum_{i=1}^d \lp \rmW_\ell \rp_{ji}^4 \le \frac{\mu_\ell}{d}$, and $\rx_\ell \in \mathbb{R}^{d}$;
    \item $\phi_\ell$ acts coordinate-wise;
    \item conditional on $\rmW_\ell$, the coordinates of $\phi_\ell(\LN(\rx_\ell))$ have common excess kurtosis $\tilde{\gamma}_\ell$, and the cross-coordinate dependence is negligible or controlled.
\end{enumerate}
Then each coordinate of $\rx_{\ell+1}$ has same mean and variance as $\phi(\LN(\rx_\ell))$, and its excess kurtosis $\gamma_{\ell+1}$ satisfies
$
|\gamma_{\ell+1}| \le \frac{\mu_\ell}{d} |\tilde{\gamma}_\ell|.
$
In particular, if $\sup_\ell \mu_\ell \le \mu < \infty$ and $|\tilde{\gamma}_\ell| \le C$ uniformly in $\ell$, then
$|\gamma_{\ell+1}| \le \frac{\mu C}{d},$
so the representation remains near-Gaussian in the fourth-moment sense whenever $d$ is large, that is, $\mu C \ll d$.
\label{thm:skipless_kurt}
\end{theorem}

\begin{proof}
Let $\rv_\ell = \LN(\rx_\ell)$ and fix a coordinate $j$ of $\rx_{\ell+1}$. By definition,
$(\rx_{\ell+1})_j=\sum_{i=1}^d \lp \rmW_\ell \rp_{ji} \, \phi_\ell((\rv_\ell)_i).$ Since $\rmW_\ell$ is orthogonal, it preserves the norm/variance. In the case of layer normalization, the mean is zero and variance is one. 
Applying \Cref{lem:ex_kurt_y_ortho_a} to the vector $\phi_\ell(\rv_\ell)$ with coefficients $\lp \rmW_\ell \rp_{j}$ yields
$|\gamma_{\ell+1}|\le\frac{\mu_\ell}{d} |\tilde{\gamma}_\ell|.$
The second claim follows immediately if $\mu_\ell$ and $\tilde{\gamma}_\ell$ are uniformly bounded.
\end{proof}

\begin{remark}
\Cref{prop:skipless_kurt_contract} formalizes the mechanism favorable to residual-free transformers. If each layer performs a dense, approximately orthogonal mixing, then even if the nonlinearity creates nonzero excess kurtosis at intermediate coordinates, the next mixing step suppresses it by a factor of order $1/d$. Thus, provided training does not destroy the dense-mixing structure or cause $|\tilde{\gamma}_\ell|$ to grow too rapidly with $d$, the hidden states remain nearly Gaussian.
\end{remark}

\begin{lemma}[Excess kurtosis of ReLU/GELU-like activations is $\mathcal{O}(1)$]
    Let $\rg \sim \mathcal{N}(0,\rmI) \in \mathbb{R}^d$ and $\phi(\rg)$ denote the centered and standardized ReLU-like activation output applied entrywise. Then excess kurtosis of $\phi(\rg)$ coordinatewise is $\mathcal{O}(1)$ with respect to dimension $d$.
    \label{lem:relu_excess_kurtosis}
\end{lemma}

\begin{proof}
As the activations are applied entrywise, there is no dependence on the dimension $d$. In the case of ReLU, we can explicitly compute the excess kurtosis using the higher-order moments $m_k$. Since  
\begin{align}
m_1 &= \frac{1}{\sqrt{2\pi}}, \qquad
m_2 = \frac{1}{2}, \qquad
m_3 = \sqrt{\frac{2}{\pi}}, \qquad
m_4=\frac{3}{2}, \nonumber
\end{align}
\begin{align}
\gamma_{\relu}
\approx 2.4076. \nonumber
\end{align}
Similarly, $\gelu$ also induces positive excess kurtosis when applied to Gaussian pre-activations. Numerically, the moments of $\gelu$ are $m_1 \approx 0.2821,
m_2 \approx 0.4252,
m_3 \approx 0.6751, 
m_4 \approx 1.3047.$ Therefore, $\gamma_\gelu \approx 3.0847$.
\end{proof}

\begin{remark}
If the normalized preactivations entering a nonlinearity like ReLU/GELU are already approximately Gaussian, then the output excess kurtosis introduced by that nonlinearity can be treated as a fixed constant.
However, the nonlinearity itself continually injects non-Gaussianity. 
\end{remark}

\paragraph{Excess kurtosis of self-attention.} 
While the previous discussion covers MLP explicitly, it also extends to the self-attention block. Let the normalized input be $\rx_\ell$, and the output of a token $i$ through self-attention be 
$$\lp \rx_{\ell+1} \rp_i = \softmax \lp \frac{1}{\sqrt{d}} \lp \rx_\ell \rp_i \rmQ_\ell \rmK_\ell^\top \lp \rx_\ell \rp_j^\top \rp \lp \rx_\ell \rp_j  \rmV_\ell \rmO_\ell,$$
where $\lp \rx_\ell \rp_i$ is the $\ell$-th layer representation of token $i$, $\rmQ_\ell, \rmK_\ell, \rmV_\ell, \rmO_\ell$ are weight matrices of layer $\ell$ that satisfy dense orthogonal condition in \Cref{thm:skipless_kurt}.
The key point is that orthogonality in the attention projections through $\rmQ_\ell$ and $\rmK_\ell$ controls both the scale of the attention scores and the isotropy of the latter projections through $\rmV_\ell$ and $\rmO_\ell$. Together, these effects help maintain near-Gaussian activations. If the excess kurtosis of softmax output is controlled, then the output $\rx_{\ell+1}$ is also near-Gaussian due to the orthogonality of $\rmV_\ell$ and $\rmO_\ell$. Therefore, the only thing left to be analyzed is the excess kurtosis induced by softmax, whether it increases or shrinks, similar to ReLU/GELU.

\begin{lemma}[Excess kurtosis of softmax activation can be $\mathcal{O}(d)$]
\label{lem:softmax_kurtosis_order}
    Let $\rg \sim \mathcal{N}(0,\rmI) \in \mathbb{R}^d$ and $\ry = \softmax(\tau \rg)$ denotes the softmax activation for some scale parameter $\tau > 0 $, that is 
    $$y_j = \frac{\exp \lp \tau g_j \rp}{\sum_{i=1}^d \exp \lp \tau g_i \rp}.$$ 
    Then excess kurtosis of $y_j$ coordinatewise is not dimension independent in general:
    \begin{enumerate}
        \item in the non-saturated regime, it is $\mathcal{O}(1)$;
        \item in the saturated regime, it is $\mathcal{O}(d)$.
    \end{enumerate}
    Hence the excess kurtosis of a softmax coordinate is controlled only when the logits remain in a non-peaky regime.
\end{lemma}

\begin{proof}
The key distinction is whether the logits $\tau g_j$ are small enough that softmax behaves approximately linearly, or large enough that softmax is close to an argmax selector/one-hot vector.

In the non-saturated regime, first-order Taylor's approximation of softmax around $\mathbf{0}$ vector gives
$$y_j \approx \frac{1}{d} + \frac{\tau}{d}\lp g_j - \frac{1}{d}\lp \sum_{i=1}^d g_i \rp \rp.$$
The centered $y_j$ is $\approx \frac{\tau}{d}g_j$ which is approximately Gaussian.  Therefore, its excess kurtosis is $\mathcal{O}(1)$.

In the saturated case, softmax becomes a nearly one-hot vector. Therefore, the coordinates are either $1$ with probability $\frac{1}{d}$ or $0$ with probability $1- \frac{1}{d}$. Therefore, a Bernoulli random variable. A Bernoulli random variable with parameter $p$ has excess kurtosis $\frac{1-6p(1-p)}{p(1-p)}$. Hence, in the case of softmax, the excess kurtosis is $$\frac{d^2-6d+6}{d-1} = \mathcal{O}(d).$$

Therefore, in the saturated regime, the coordinatewise excess kurtosis of softmax can grow linearly with the softmax dimension $d$.
\end{proof}

\begin{remark}
Unlike ReLU/GELU-like activations, softmax is not applied entrywise, and its kurtosis behavior depends on the geometry of the whole logit vector. If the attention scores remain controlled, then each softmax coordinate behaves like a smooth perturbation of $1/d$ and has $\mathcal{O}(1)$ excess kurtosis. However, if the scores become too large, softmax saturates, the attention weights become peaky, and the coordinatewise excess kurtosis can increase to $\mathcal{O}(d)$. Thus, for attention, controlling the score variance is essential to prevent softmax from becoming a source of strong non-Gaussianity.
\end{remark}

\begin{lemma}[Orthogonal $\rmQ_\ell,\rmK_\ell$ do not amplify normalized attention scores and controls scale]
Let's consider $\rmQ_\ell,\rmK_\ell$ to be orthogonal and the input to softmax,
\begin{align}
\lp \rs_\ell \rp_{ij} = \frac{1}{\sqrt d} (\rx_\ell)_i \rmQ_\ell \rmK_\ell^\top (\rx_\ell)_j^\top. \nonumber
\end{align}
The token representations entering attention are centered and variance-controlled through the normalization $\LN(\cdot)$, thus 
their covariance operators satisfy
\begin{align}
\|\Sigma_{\ell,i}\|_2 = \expectation{}{\| \lp \rx_\ell \rp_i \|_2^2} \le C, \quad \forall i
\end{align}
for a constant $C$ independent of $d$. Note that $C=1$ in the case of Layer Normalization. Then $\Var \lp \lp s_\ell \rp_{ij} \rp = \mathcal{O}(1).$ Consequently, the row-wise softmax operates in a non-saturated regime, and its coordinatewise excess kurtosis remains $\mathcal{O}(1)$.
\end{lemma}

\begin{proof}
Let $\rmM_\ell = \rmQ_\ell \rmK_\ell^\top.$ Then
$\lp s_\ell\rp_{ij} = \frac{1}{\sqrt d} (\rx_\ell)_i \rmM_\ell (\rx_\ell)_j^\top. $
Since $(\rx_\ell)_i$ and $(\rx_\ell)_j$ are centered, $\expectation{}{\lp s_\ell \rp_{ij}}=0$, and hence
\begin{align}
\Var\!\left(\lp s_\ell \rp_{ij}\right) &= \frac{1}{d}
\expectation{}{\left((\rx_\ell)_i \rmM_\ell (\rx_\ell)_j^\top\right)^2} \nonumber \\
&= \frac{1}{d} \expectation{}{\trace \lp (\rx_\ell)_i \rmM_\ell (\rx_\ell)_j^\top (\rx_\ell)_j \rmM_\ell (\rx_\ell)_i^\top \rp} \nonumber \\
&= \frac{1}{d} \expectation{}{\trace \lp (\rx_\ell)_i^\top (\rx_\ell)_i \rmM_\ell (\rx_\ell)_j^\top (\rx_\ell)_j \rmM_\ell \rp} \nonumber \\
&= \frac{1}{d} \trace \lp \expectation{}{(\rx_\ell)_i^\top (\rx_\ell)_i \rmM_\ell (\rx_\ell)_j^\top (\rx_\ell)_j \rmM_\ell }\rp \nonumber \\
&\le \frac{1}{d} \expectation{}{\| (\rx_\ell)_i^\top (\rx_\ell)_i \|_2} \trace \lp \rmM_\ell \expectation{}{(\rx_\ell)_j^\top (\rx_\ell)_j} \rmM_\ell \rp \nonumber \\
&\le \frac{1}{d} \|\Sigma_{\ell,i}\|_2 \|\Sigma_{\ell,j}\|_2 \trace \lp \rmM_\ell^\top \rmM_\ell \rp \qquad ; \rmM_\ell^\top \rmM_\ell = \rmI \nonumber \\
&\le C^2 = \mathcal{O}(1)   \nonumber
\end{align}

Thus, if the input to the attention block is already normalized or otherwise controlled so that the raw score scale is stable, then the projected attention scores remain stable as well. In particular, the logits do not systematically blow up with dimension, so the softmax remains away from the saturated nearly one-hot regime. By \Cref{lem:softmax_kurtosis_order}, this implies that each softmax coordinate has $\mathcal{O}(1)$ excess kurtosis. This proves the claim.
\end{proof}

\begin{remark}
The main role of orthogonal $\rmQ_\ell$ and $\rmK_\ell$ is not to make the softmax output exactly Gaussian, but to prevent the attention scores from entering the peaky saturated regime. Once the score variance is controlled at $\mathcal{O}(1)$ scale, the softmax coordinates remain in the benign regime where their excess kurtosis is also $\mathcal{O}(1)$ rather than growing with the sequence length.
\end{remark}

The preceding analysis explains the empirical behavior observed in the residual-free setting. The key mechanism is that, in the absence of residual addition, each layer \emph{replaces} the current representation by a newly mixed representation rather than carrying forward the previous hidden state. As a result, whether activations remain close to Gaussian is determined by two interacting properties: $(i)$ whether the layer map performs dense near-isometric mixing, and $(ii)$ whether optimization preserves this structure over training.

First, orthogonal initialization is favorable because it induces dense norm-preserving mixing. By \Cref{lem:ex_kurt_y_ortho_a}, such mixing suppresses excess kurtosis by a factor of order $1/d$. Thus, although nonlinearities such as ReLU, GeLU, and softmax can inject $\mathcal{O}(1)$ non-Gaussianity locally, the next orthogonal mixing step averages this out and keeps the activations close to Gaussian. In attention, orthogonal $\rmQ_\ell,\rmK_\ell$ control the score variance and prevent softmax saturation, while orthogonal $\rmV_\ell$ and $\rmO_\ell$ preserve isotropy after softmax. 
Altogether, this makes the residual-free block behave as a stable averaging-and-mixing operator.

Second, spectral optimizers complement this by preserving the near-isometric structure during training. By \Cref{lem:spectral_gram_drift}, their orthogonality deviation is only $\mathcal{O}(\eta)$, so an initially orthogonal weight matrix remains close to orthogonal for a small learning rate. Hence the dense-mixing Gaussianization mechanism continues to hold throughout training.
By contrast, default initialization does not enforce isometry. Although activations may be locally Gaussian at initialization, variance is not preserved uniformly across directions, so anisotropy can accumulate across layers. Similarly, sign GD or Adam/AdamW deviates substantially from orthogonality (\Cref{prop:sign_gram_drift}) even when training starts from an orthogonal initialization, leading to more concentrated coordinate directions and weakening the kurtosis shrinkage.

Taken together, the analysis yields the following result for residual-free transformers. 
\begin{enumerate}
    \item Orthogonal initialization provides the correct geometric starting point by enforcing dense norm-preserving mixing;
    \item spectral optimization preserves this geometry over training.
\end{enumerate}
 Either ingredient alone is insufficient. Therefore, among the settings considered, the combination of \textbf{residual-free + orthogonal initialization + spectral optimization} is the one that naturally maintains near-Gaussian activations.

\paragraph{Interpretation of theorems.}
Putting everything together yields the following interpretation of the theorems.

\begin{enumerate}[leftmargin=*]
    \item In \textbf{residual-free transformers}, if the learned attention and MLP maps remain sufficiently dense and approximately orthogonal, then each layer re-mixes coordinates and suppresses excess kurtosis introduced by nonlinearities. This creates a self-correcting fourth-moment mechanism toward Gaussian-like activations. 
    \item In \textbf{residual transformers}, the residual stream carries previous non-Gaussianity forward. Even if the branch output were itself close to Gaussian, the sum need not become more Gaussian because the old hidden state is preserved rather than replaced.
    \item  Between \textbf{spectral gradient descent} and sign gradient descent, spectral GD is more likely to preserve dense orthogonal mixing, and sign GD is coordinatewise and weakens $1/d$ kurtosis contraction.
    \item Consequently, the combination of \textbf{orthogonal initialization} and \textbf{spectral optimizers} is especially favorable in the residual-free setting. It both promotes dense norm-preserving mixing at initialization and tends to preserve it during training.
\end{enumerate}

\section{Experimental Details}
\label{app:exp_details}
\paragraph{Model.}
We use LLaMA-style decoder-only transformers with GELU activations in the feedforward blocks. The models have 24 layers with 12 attention heads, 768 hidden dimension, and 3072 FFN dimension, totaling 194 million parameters. We apply RMSNorm before each Self-Attention Block and Feed-Forward Network.

\paragraph{Data.} 
We use FineWeb-Edu \citep{penedo2024fineweb} with LlamaTokenizer and a 32000 vocabulary size. The batch size is 512 sequences, the sequence length is 512, and the seed is 42. No document filtering is applied. Therefore, the total tokens come to 262,144 tokens per step and approximately 13.1 billion training tokens over 50k steps. 

\paragraph{Training.}
All models are pretrained on FineWeb-Edu for 50k steps using the next-token prediction objective under BF16. We set 2k steps of linear warmup followed by cosine learning rate decay, and gradient clip norm to 1.  Unless otherwise specified, weight decay is set to $0.01$. For residual-free models, we use the orthogonal initialization described
in \Cref{sec:method}, with $\alpha_{QK}=0.9$,  $\alpha_{VO}=3$, and  $\alpha_{U}=\alpha_{D}=1.5$. For depth-aware attention
temperature scaling, we use base scaling factor $\beta=1.1$. The parameters of AdamW are $\beta_1 = 0.9, \beta_2 = 0.95$, and KL-Shampoo are $\beta_1 = 0.9, \beta_2= 0.95$ with preconditioner updates every 10 steps, initialization factor 0.1, no damping or grafting. The same optimizer settings are used for residual and residual-free models. 
We sweep learning rates in the range $\{1e-4, 3e-4, 5e-4, 7e-4, 1e-3, 3e-5, 5e-3\}$  and report results using the best stable setting for each optimizer and architecture. The selected learning rates are the
one that achieves the minimum NLL loss:
\begin{itemize}
        \item Residual + AdamW is 1e–3.
        \item Residual + KL Shampoo is 1e–3 
        \item Residual-free + KL Shampoo is 5e–4.
        \item Residual-free + Muon is 7e–4.
        \item Residual-free + SOAP is 5e–4.
\end{itemize}
Residual and residual-free models use identical dimensions, tokenizer, data order, training budget, and evaluation pipeline.

\textbf{Quantization protocol.}
After pretraining, we quantize tensors from full precision (BF16) to integer formats. Unless otherwise stated, we apply simple uniform quantization to both weights and activations, without any additional outlier handling, reconstruction, or learned rotations. This choice is intentional: our goal is to measure the intrinsic quantization robustness of the activation distributions produced by different architectures and optimizers. We adopt the notation WnAm to denote n-bit weight quantization and m-bit activation quantization (\eg{} W8A12 corresponds to 8-bit weights and 12-bit activations). Weights are quantized symmetrically first, followed by activations asymmetrically, with one scale per tensor. Token embeddings, weights, and activations are quantized. Since our focus is activation outliers, we quantize activations at the inputs to the normalization layers and the SA/MLP. For weights, we use per-channel asymmetric static quantization; for activations, we use per-tensor quantization. RMSNorm parameters and the final language model head remain in BF16. The calibration set consists of 32k tokens sampled from FineWeb-Edu.

\textbf{Evaluation metrics.}
For activation Gaussianity, we report excess kurtosis as defined in \Cref{def:ex_kurt}. We additionally measure negentropy, a standard information-theoretic proxy for non-Gaussianity. 
\begin{definition}[Negentropy]
    For a random variable $x$ with covariance matched to a Gaussian random variable $g$, negentropy is $$  J(x) = H(g) - H(x),$$ where $H(\cdot)$ denotes differential entropy.
\end{definition}
Since a Gaussian maximizes entropy among distributions with fixed covariance, $J(x) \ge 0$, with equality only for a Gaussian distribution.

For quantization, we report downstream accuracy after quantizing weights and activations to integer formats. We also compute normalized mean-squared error (MSE) and signal-to-quantization-noise ratio (SQNR) between full-precision tensors $\rx$ and their quantized-original versions $\hat{\rx}$:
$$\mathrm{Normalized \, MSE}(\rx,\hat{\rx}) = \frac{\|\rx-\hat{\rx}\|_2^2}{\|\rx\|_2^2}, \quad
\mathrm{SQNR}(\rx,\hat{\rx}) = 10 \log_{10} \lp \frac{\|\rx\|_2^2}{\|\rx-\hat{\rx}\|_2^2} \rp.
$$

\textbf{Downstream evaluation.}
We use lm-evaluation-harness 0.4.12.dev0, seed 42, and evaluate pretrained language models on eight zero-shot commonsense reasoning benchmarks: \textit{arc\_c}, \textit{arc\_e}, \textit{boolq}, \textit{hellaswag}, \textit{openbookqa}, \textit{piqa}, \textit{social\_iqa}, and \textit{winogrande}. We report task accuracy and the unweighted average accuracy across all tasks. Residual and residual-free models use identical dimensions, tokenizer, data order, training budget and evaluation pipeline. We do not report training-seed error bars due to the cost of pretraining, but we provide
per-task results and tensor-level distortion metrics to support the average trends.
Also, because zero-shot task accuracy is noisy and quantization can perturb answer rankings non-monotonically on individual tasks, we focus on average trends across tasks and complement accuracy with tensor-level SQNR/MSE.

\textbf{Hardware.} 
Each experiment is run on 4 H100 GPUs. For Residual with KL Shampoo, it takes 90622 MB of memory on each GPU and 1.56 steps per second. For Residual with AdamW, it takes 90328 MB of memory on each GPU and 4.77 steps per second. For Residual-Free with KL Shampoo, it takes 81598 MB of memory on each GPU and 1.7 steps per second. For Residual-Free with SOAP, it takes 81596 MB of memory on each GPU and 0.9 steps per second. For Residual-Free with MUON, it takes 81304 MB of memory and 4.6 steps per second.

\textbf{Code.} A functioning code repository is provided in \url{https://github.com/yipingji/residual-free-quantization}.

\FloatBarrier



\section{Additional Experiments}
\label{app:kurtosis_exp}

In this section, we provide additional empirical evidence supporting the statistical
and quantization advantages of residual-free transformers.

\subsection{Comparison of Residual-free and Residual Training}

\Cref{fig:residual_residual_free_loss} shows the training loss comparison of residual and residual-free transformers. The training curve for residual networks is added to \Cref{fig:residualfree-optimizers} (left).

\begin{figure}[h!]
    \centering
    \includegraphics[width=0.6\linewidth]{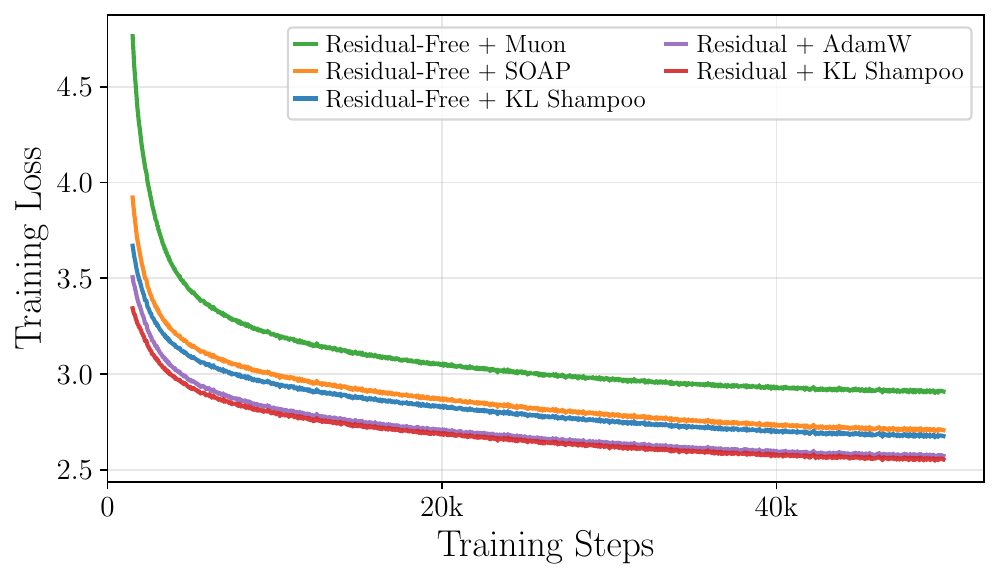}
    \caption{Training curves of residual and residual-free transformers using different optimizers.}
    \label{fig:residual_residual_free_loss}
\end{figure}
\FloatBarrier

\subsection{Residual-free Transformers With Second-Order Optimizations}

We study how different optimizers affect the activation statistics of residual-free
transformers. \Cref{fig:skipless_opt_krt_neg} shows that residual-free models when trained with second-order optimizers result in activations that have near-zero excess kurtosis and negentropy, indicating Gaussian-like behavior under
these diagnostics. 
We further illustrate the activation histograms of all residual-free models trained with different second-order optimizers for layers 8, 16, and 22 after 50k steps in \Cref{fig:skipless_opt_hist}, which confirms the well-concentrated, symmetric distributions with very low excess kurtosis, matching near-Gaussian distributions.

\begin{figure}[h!]
    \vspace{0.5cm}
    \centering
    \includegraphics[width=0.47\linewidth]{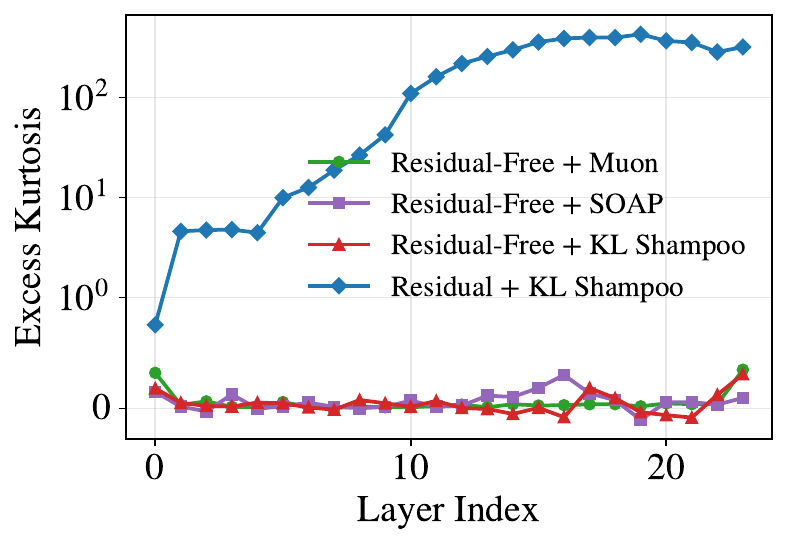}
    \includegraphics[width=0.47\linewidth]{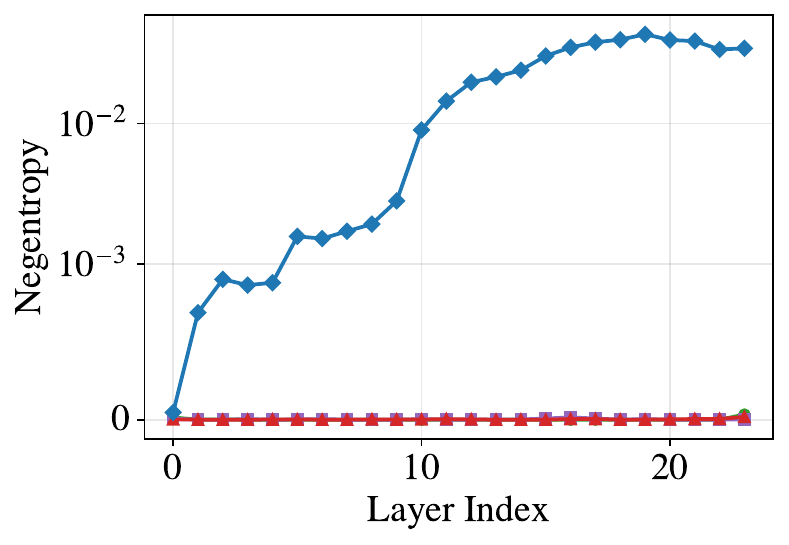}
    \caption{Excess kurtosis and negentropy comparison of residual-free models pretrained with different optimizers and a residual model trained with KL Shampoo. Residual-free models show near-zero excess kurtosis and negentropy, therefore near-Gaussian activations.}
    \label{fig:skipless_opt_krt_neg}
\end{figure}

\begin{figure}[h!]
    \centering
    \includegraphics[width=1\linewidth]{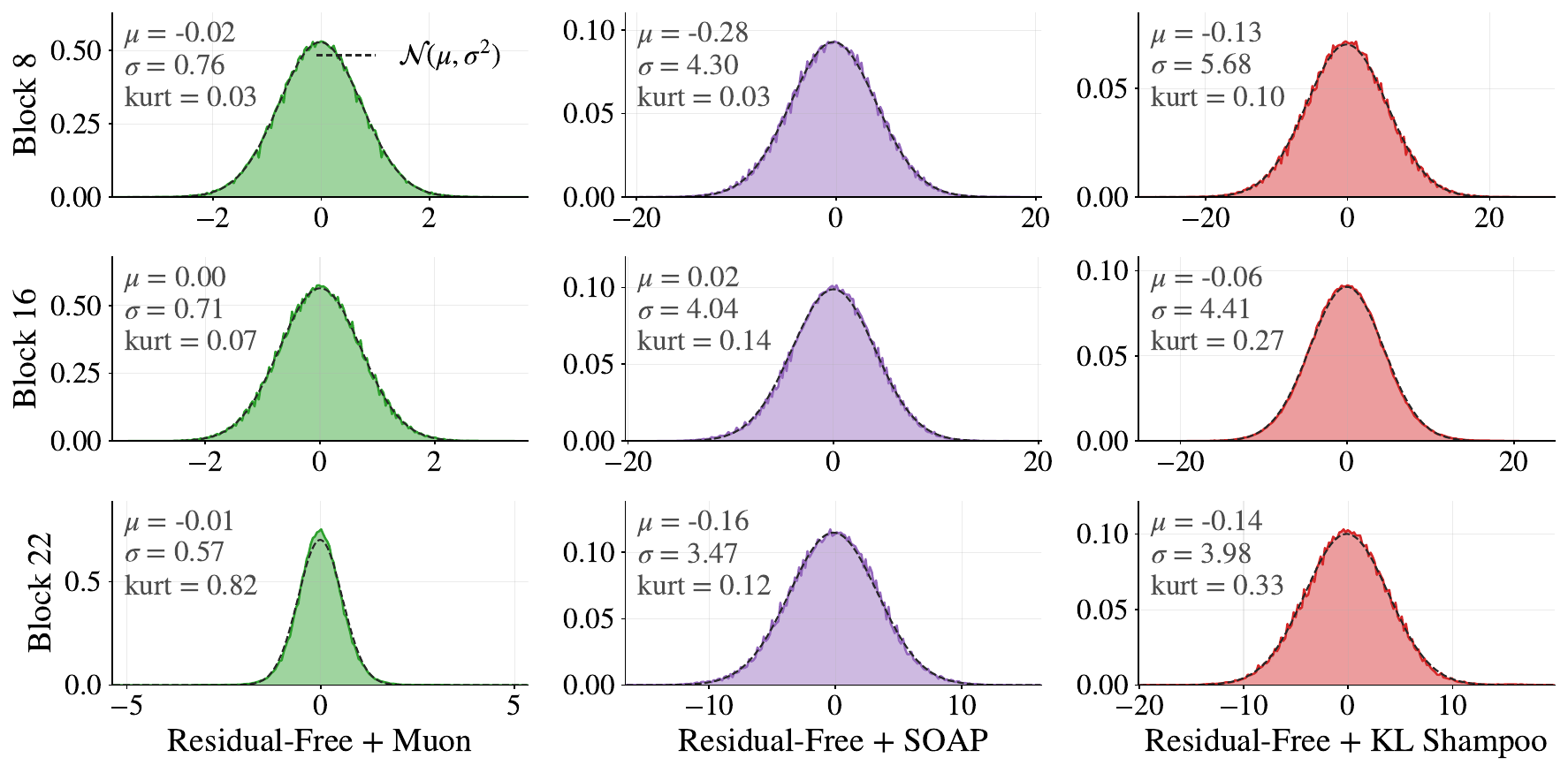}
    \caption{Activation distribution comparison of residual-free models pretrained with different second-order optimizers, such as Muon, SOAP, KL Shampoo, at step 50k. All models produce near-Gaussian activations, with low excess kurtosis across depth.}
    \label{fig:skipless_opt_hist}
\end{figure}
\FloatBarrier

\subsection{Residual and Residual-Free Transformers: Activation Distribution}
We now illustrate the distribution of activations during training for residual and residual-free models. \Cref{fig:activation_dist_training8,fig:activation_dist_training16,fig:activation_dist_training22} show the evolution of activation distributions during training across residual models trained with AdamW and KL Shampoo, and residual models with KL Shampoo for layers 8, 16, and 22, respectively. 
\begin{itemize}
    \item Consistently, the residual model with KL Shampoo maintains near-Gaussian activations across layers and across training steps. 
    \item AdamW rapidly loses Gaussianity during training and becomes extremely peaky across layers, with excess kurtosis decreasing across layers. That is, distributions peak throughout training; the peaking is high in initial layers. For instance, in 50k steps, the excess kurtosis decreases from $73,998$ to $40,527$ to $553$. While they all remain non-Gaussian, this shows the earlier layers are far more peaky and difficult to quantize, which is also evident from their value ranges.
    \item In the case of residual and KL Shampoo, it also becomes non-Gaussian similar to AdamW, but it is more gradual during training. Across layers, it maintains similar excess kurtosis, suggesting that it can be due to the isometry updates of spectral-type optimizations.
\end{itemize}

\begin{figure}[h!]
    \centering
    \includegraphics[width=0.95\linewidth]{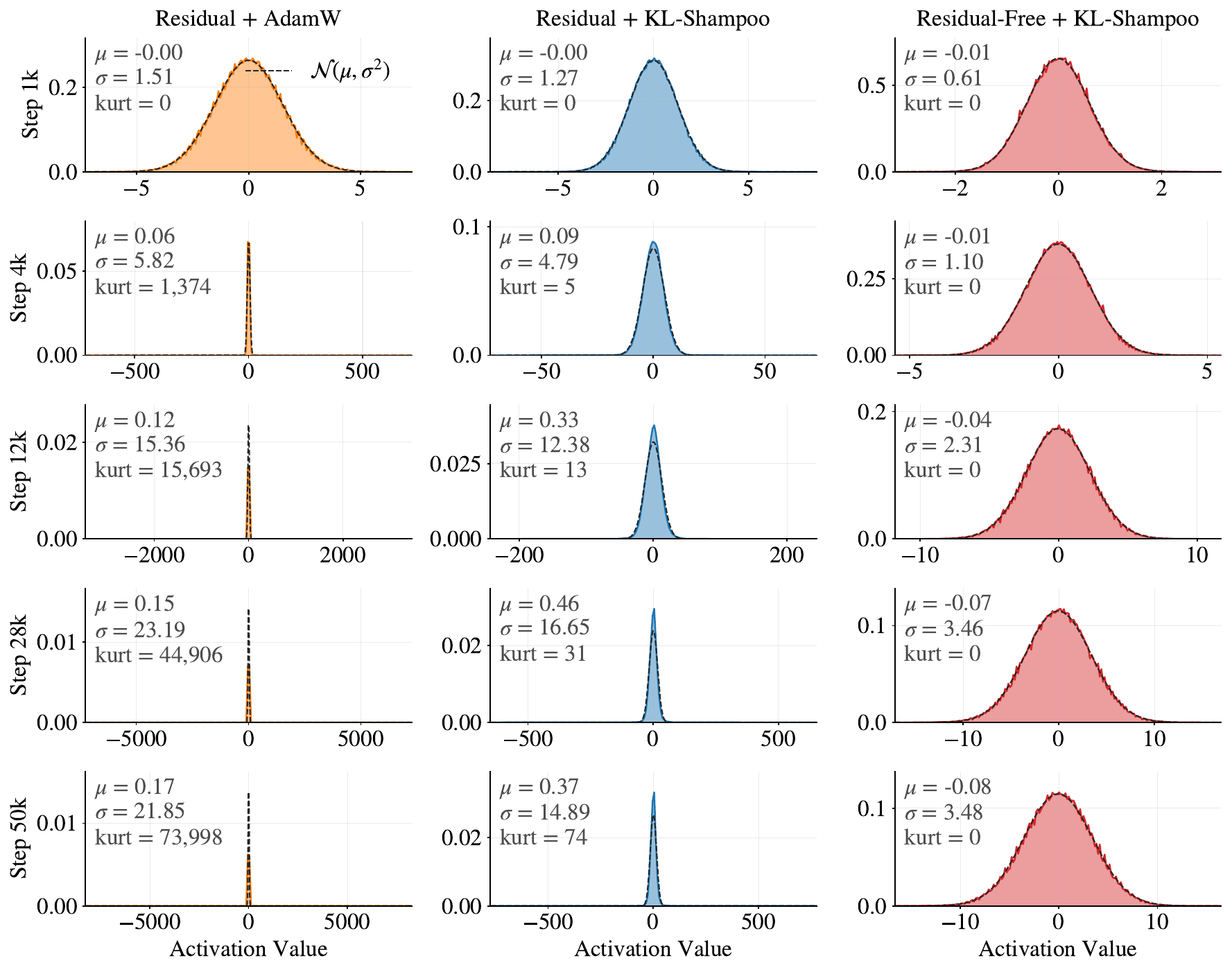}
        \caption{Activation distribution comparison of residual and residual-free for layer 8 during training.}

    \label{fig:activation_dist_training8}
\end{figure}

\begin{figure}[h!]
    \centering
    \includegraphics[width=0.95\linewidth]{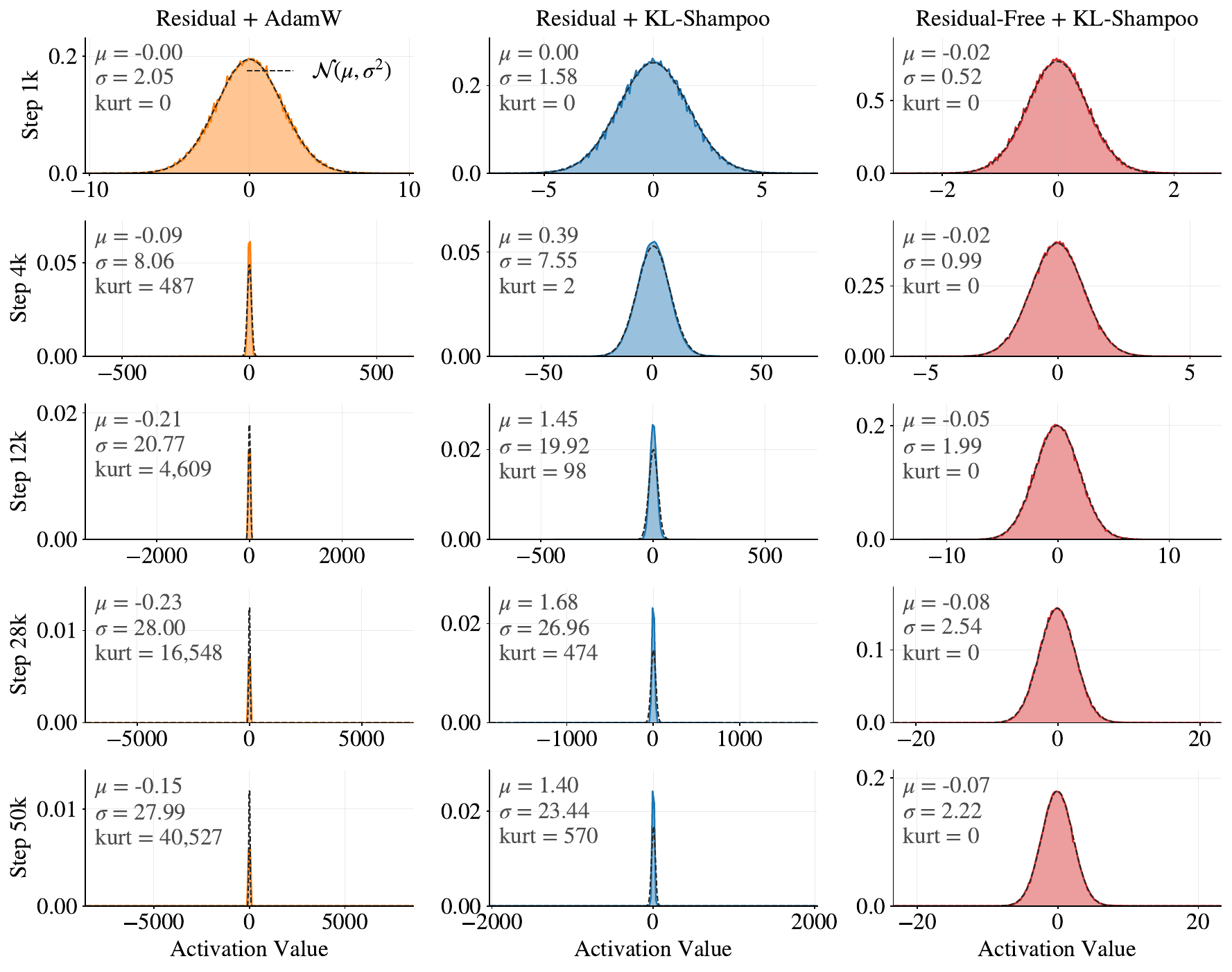}
        \caption{Activation distribution comparison of residual and residual-free for layer 16 during training.}
    \label{fig:activation_dist_training16}
\end{figure}

\begin{figure}[t]
    \centering
    \includegraphics[width=1\linewidth]{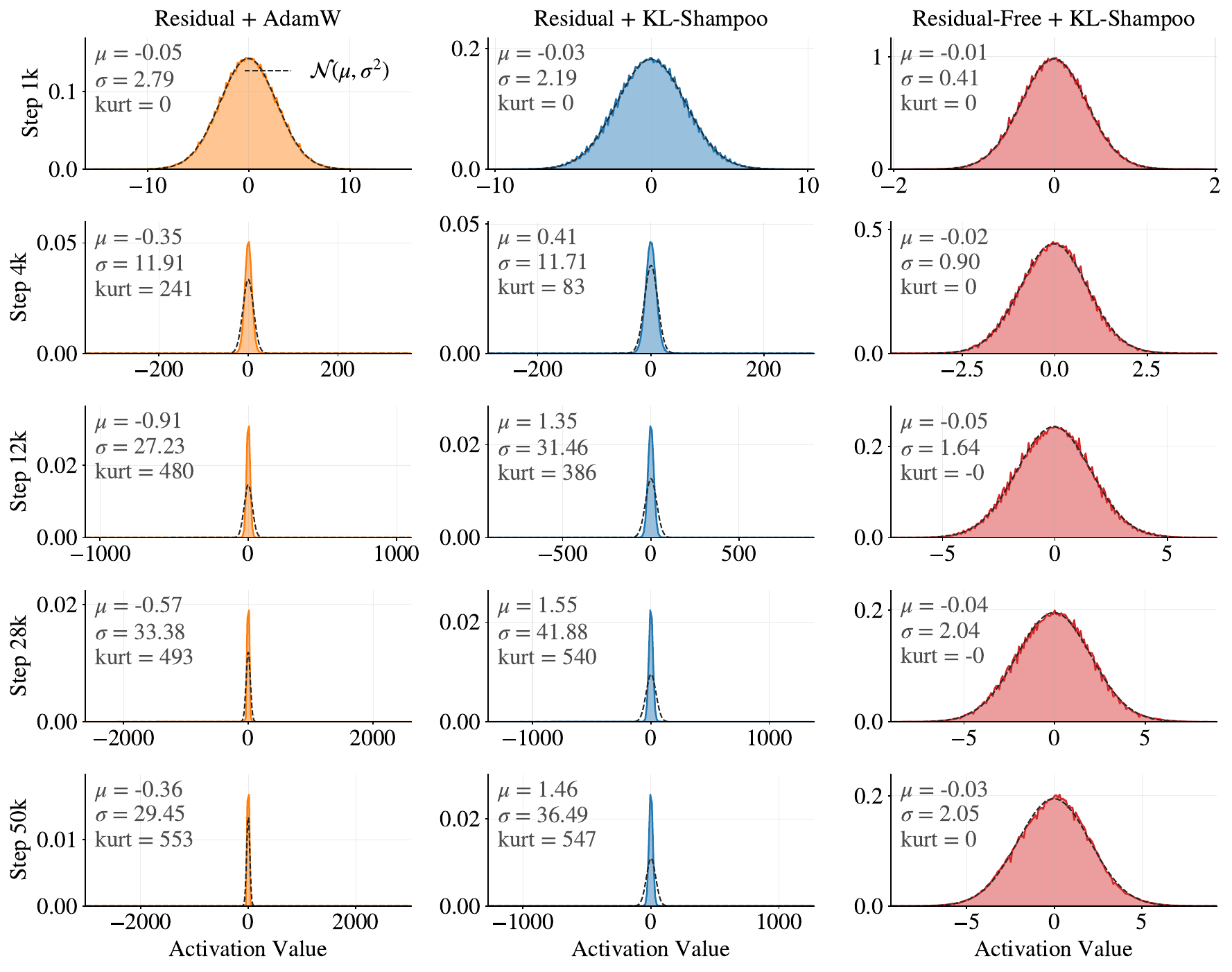}
        \caption{Activation distribution comparison of residual and residual-free for layer 22 during training.}
    \label{fig:activation_dist_training22}
\end{figure}
\FloatBarrier

\subsection{Residual and Residual-Free Transformers: W8 Quantization Results}
We further analyze the relationship between activation statistics and quantization performance across different precisions, focusing on 8-bit quantization of weights in this section. \Cref{fig:acc_vs_bits_w8} shows that residual-free models maintain stable accuracy even at low activation bit-widths, while residual models degrade significantly, especially when trained with AdamW. This behavior is consistent with the distortion metrics in \Cref{fig:quant_metrics_w8}, where residual-free models achieve higher SQNR and substantially lower normalized MSE across all bit-widths. The improvement is particularly pronounced at 4--8 bits, where quantization error is most sensitive to activation outliers. Finally, the layerwise statistics confirm the underlying mechanism that residual-free models maintain near-zero excess kurtosis and negentropy across depth, while residual models accumulate heavy-tailed, non-Gaussian activations, leading to higher quantization error. \Cref{tab:w8_results} shows the complete results.

\begin{figure}[h!]
    \centering
    \includegraphics[width=0.5\linewidth]{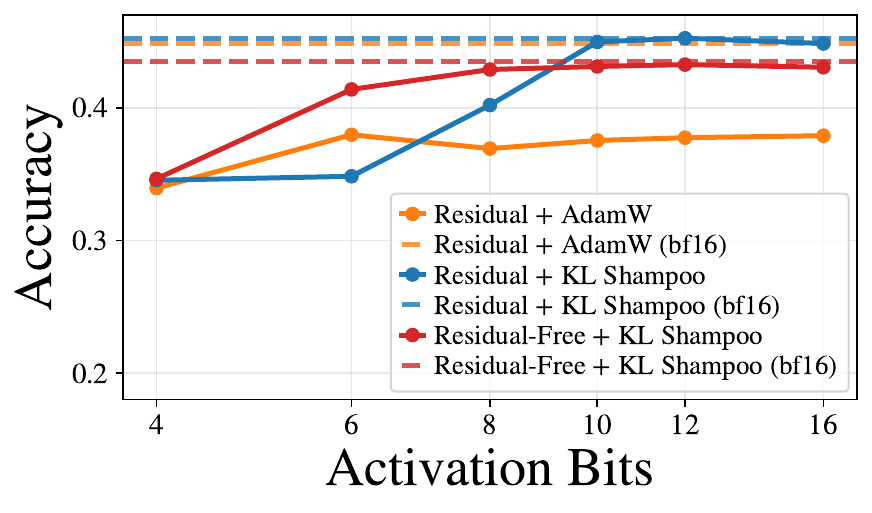}
     \caption{Accuracy as a function of activation bit-width and 8-bit weights. Residual-free transformers maintain stable performance even at low precision, while residual models exhibit significant degradation, particularly when trained with AdamW. Dashed lines indicate full-precision (BF16) baselines.}
    \label{fig:acc_vs_bits_w8}
\end{figure}

\begin{figure}[h!]
    \centering
    \includegraphics[width=0.49\linewidth]{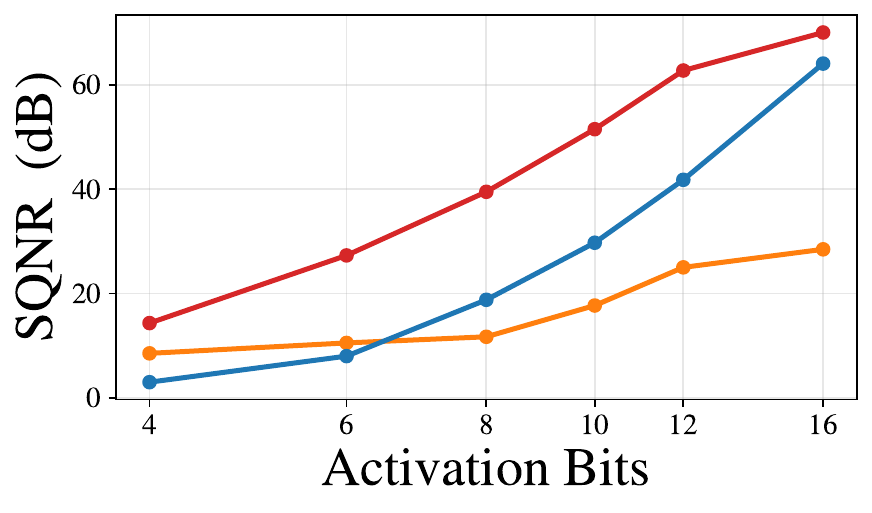}
    \includegraphics[width=0.49\linewidth]{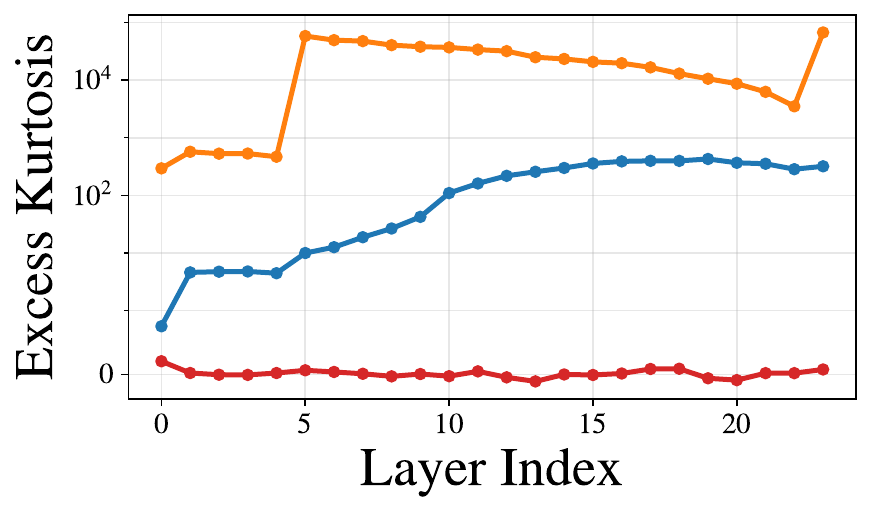}
    \includegraphics[width=\linewidth]{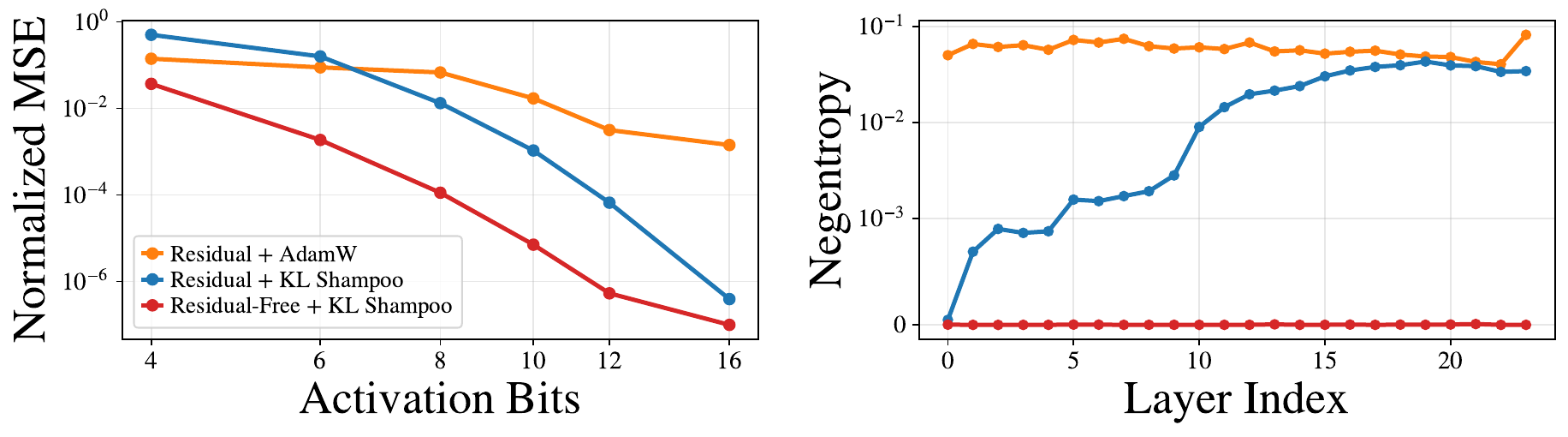}
    \caption{Quantization distortion and activation statistics. \textbf{Top-left:} SQNR improves steadily with bit-width, with residual-free models achieving the highest signal fidelity. \textbf{Bottom-left:} normalized MSE decreases significantly faster for residual-free models, especially at low precision. \textbf{Right:} layerwise excess kurtosis and negentropy show that residual models accumulate heavy-tailed, non-Gaussian activations, while residual-free models remain close to Gaussian. These statistical differences explain the improved quantization robustness.}
    \label{fig:quant_metrics_w8}
\end{figure}
\FloatBarrier

\begin{table}[h!]
\caption{\textbf{Residual-free transformers are substantially quantization robust.} Zero-shot performance (average accuracy) on 8 downstream tasks comparison across models under different quantization configurations.}
\centering
\resizebox{\textwidth}{!}{%
\begin{tabular}{@{}l|l|cccccccc|c@{}}
\toprule
\textbf{Model} & \textbf{Quantization} & \textbf{arc\_c} & \textbf{arc\_e} & \textbf{boolq} & \textbf{hellaswag} & \textbf{openbookqa} & \textbf{piqa} & \textbf{social\_iqa} & \textbf{winogrande} & \textbf{avg} \\ \midrule
\multirow{7}{*}{\rotatebox[origin=c]{90}{\footnotesize\shortstack{Residual \\ AdamW}}}           & BF16  & 0.257 & 0.488 & 0.583 & 0.370 & 0.330 & 0.653 & 0.383 & 0.524 & 0.449 \\
                                                                                                & W8A16 & 0.230 & 0.367 & 0.476 & 0.291 & 0.262 & 0.545 & 0.347 & 0.513 & 0.379 \\
                                                                                                & W8A12 & 0.232 & 0.334 & 0.500 & 0.295 & 0.264 & 0.526 & 0.347 & 0.521 & 0.378 \\
                                                                                                & W8A10 & 0.230 & 0.343 & 0.495 & 0.290 & 0.268 & 0.529 & 0.348 & 0.500 & 0.375 \\
                                                                                                & W8A8  & 0.273 & 0.284 & 0.507 & 0.263 & 0.272 & 0.500 & 0.346 & 0.510 & 0.369 \\
                                                                                                & W8A6  & 0.279 & 0.262 & 0.619 & 0.263 & 0.272 & 0.503 & 0.326 & 0.515 & 0.380 \\
                                                                                                & W8A4  & 0.272 & 0.261 & 0.380 & 0.258 & 0.240 & 0.487 & 0.343 & 0.475 & 0.340 \\ \midrule
\multirow{7}{*}{\rotatebox[origin=c]{90}{\footnotesize\shortstack{Residual \\ KL Shampoo}}}      & BF16  & 0.265 & 0.513 & 0.582 & 0.385 & 0.312 & 0.650 & 0.396 & 0.514 & 0.452 \\
                                                                                                & W8A16 & 0.255 & 0.509 & 0.595 & 0.374 & 0.304 & 0.640 & 0.388 & 0.522 & 0.449 \\
                                                                                                & W8A12 & 0.265 & 0.510 & 0.596 & 0.377 & 0.318 & 0.641 & 0.387 & 0.526 & 0.453 \\
                                                                                                & W8A10 & 0.271 & 0.509 & 0.588 & 0.375 & 0.310 & 0.631 & 0.390 & 0.523 & 0.450 \\
                                                                                                & W8A8  & 0.222 & 0.401 & 0.578 & 0.311 & 0.270 & 0.565 & 0.349 & 0.521 & 0.402 \\
                                                                                                & W8A6  & 0.277 & 0.263 & 0.380 & 0.258 & 0.244 & 0.505 & 0.348 & 0.513 & 0.348 \\
                                                                                                & W8A4  & 0.286 & 0.258 & 0.382 & 0.253 & 0.238 & 0.492 & 0.343 & 0.511 & 0.345 \\ \midrule
\multirow{7}{*}{\rotatebox[origin=c]{90}{\footnotesize\shortstack{Residual-Free \\ KL Shampoo}}} & BF16  & 0.257 & 0.463 & 0.606 & 0.343 & 0.316 & 0.616 & 0.371 & 0.511 & 0.435 \\
                                                                                                & W8A16 & 0.245 & 0.454 & 0.587 & 0.345 & 0.302 & 0.616 & 0.369 & 0.526 & 0.431 \\
                                                                                                & W8A12 & 0.252 & 0.450 & 0.589 & 0.345 & 0.318 & 0.617 & 0.366 & 0.524 & 0.433 \\
                                                                                                & W8A10 & 0.241 & 0.460 & 0.585 & 0.343 & 0.326 & 0.618 & 0.361 & 0.515 & 0.431 \\
                                                                                                & W8A8  & 0.247 & 0.448 & 0.590 & 0.344 & 0.292 & 0.611 & 0.370 & 0.529 & 0.429 \\
                                                                                                & W8A6  & 0.243 & 0.415 & 0.553 & 0.331 & 0.296 & 0.582 & 0.366 & 0.527 & 0.414 \\
                                                                                                & W8A4  & 0.272 & 0.269 & 0.394 & 0.260 & 0.238 & 0.497 & 0.342 & 0.500 & 0.346 \\ \bottomrule
\end{tabular}%
}
\label{tab:w8_results}
\end{table}

\FloatBarrier

\subsection{Residual and Residual-Free Transformers: W16 Quantization Results}
Similar to the previous section, we do 16-bit quantization of weights and analyze the performance of residual and residual free models. The results are consistent with 8-bit weight quantization. \Cref{fig:acc_vs_bits_w16} shows that residual-free models maintain stable accuracy even at low activation bit-widths, while residual models degrade significantly, especially when trained with AdamW. This behavior is consistent with the distortion metrics in \Cref{fig:quant_metrics_w16}, where residual-free models achieve higher SQNR and substantially lower normalized MSE across all bit-widths, along with confirming layerwise excess kurtosis and negentropy. 
\Cref{tab:w16_results} shows the complete results. 

\begin{figure}[h!]
    \centering
    \includegraphics[width=0.5\linewidth]{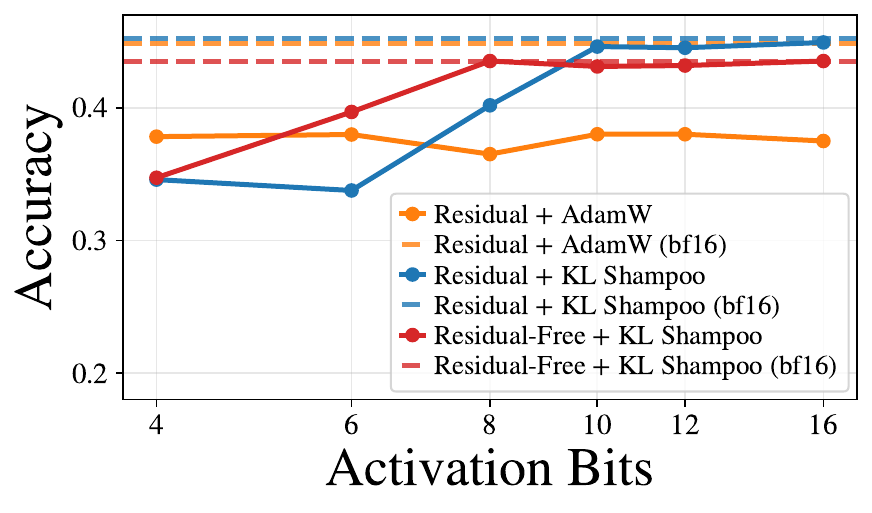}
    \caption{Accuracy as a function of activation bit-width and 16-bit weights. Residual-free transformers maintain stable performance even at low precision, while residual models exhibit significant degradation, particularly when trained with AdamW. Dashed lines indicate full-precision (BF16) baselines.}
    \label{fig:acc_vs_bits_w16}
\end{figure}

\begin{figure}[h!]
    \centering
    \includegraphics[width=0.49\linewidth]{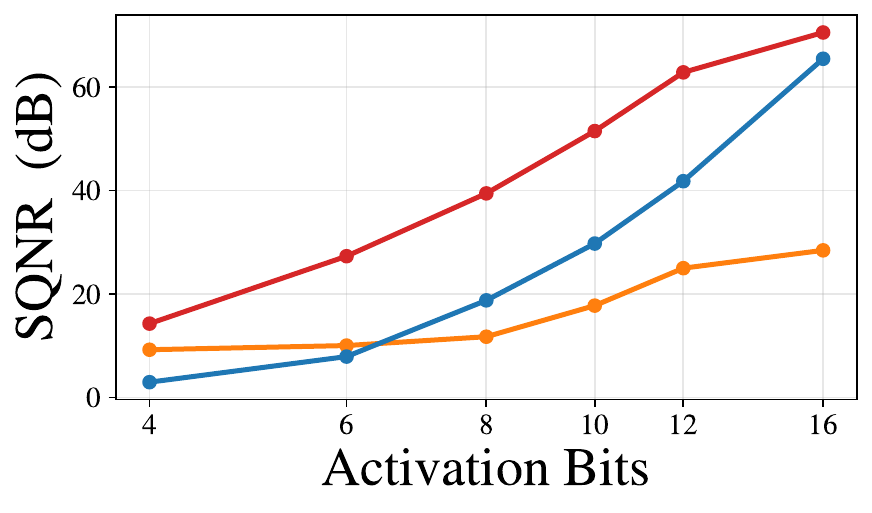}
    \includegraphics[width=0.49\linewidth]{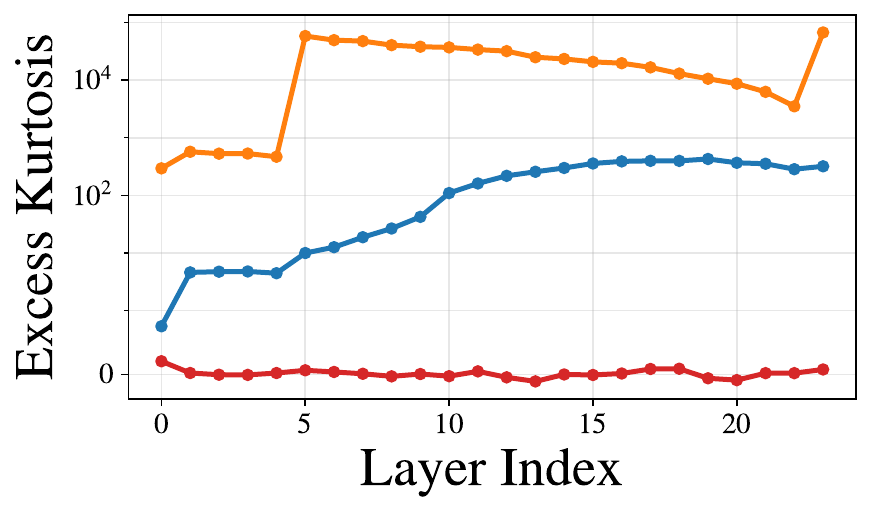}
    \includegraphics[width=\linewidth]{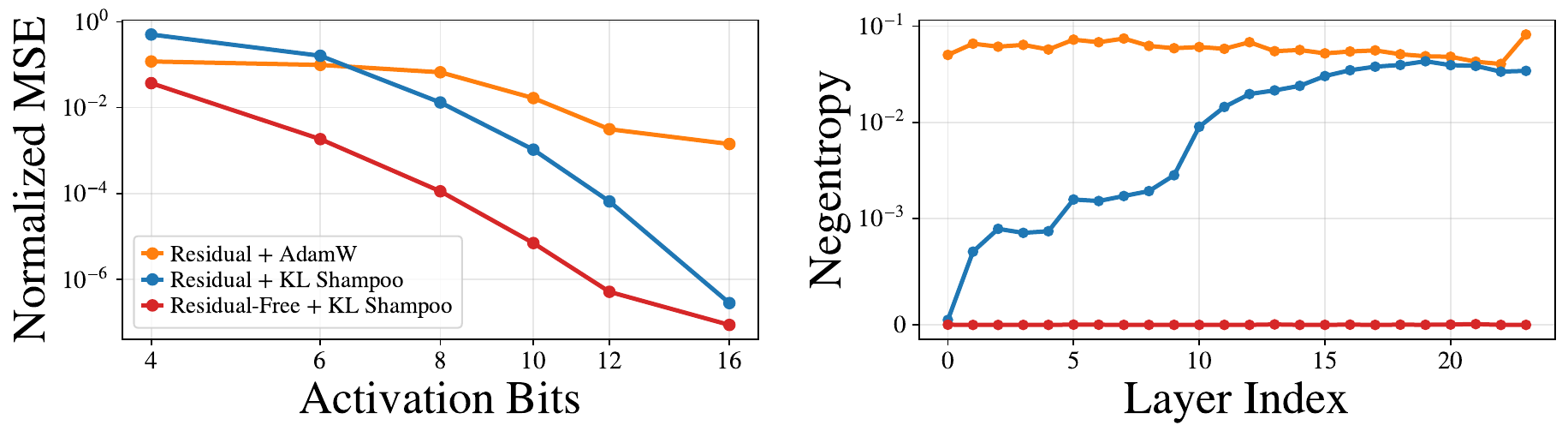}

    \caption{Quantization distortion and activation statistics. \textbf{Top-left:} SQNR improves steadily with bit-width, with residual-free models achieving the highest signal fidelity. \textbf{Bottom-left:} normalized MSE decreases significantly faster for residual-free models, especially at low precision. \textbf{Right:} layerwise excess kurtosis and negentropy show that residual models accumulate heavy-tailed, non-Gaussian activations, while residual-free models remain close to Gaussian. These statistical differences explain the improved quantization robustness.}
    \label{fig:quant_metrics_w16}
\end{figure}
\FloatBarrier

\begin{table}[h!]
\label{tab:w16_results}
\caption{\textbf{Residual-free transformers are substantially quantization robust.} Zero-shot performance (average accuracy) on 8 downstream tasks comparison across models under different quantization configurations.}
\centering
\resizebox{\textwidth}{!}{%
\begin{tabular}{@{}l|l|cccccccc|c@{}}
\toprule
\textbf{Model} & \textbf{Quantization} & \textbf{arc\_c} & \textbf{arc\_e} & \textbf{boolq} & \textbf{hellaswag} & \textbf{openbookqa} & \textbf{piqa} & \textbf{social\_iqa} & \textbf{winogrande} & \textbf{avg} \\ \midrule
\multirow{7}{*}{\rotatebox[origin=c]{90}{\footnotesize\shortstack{Residual \\ AdamW}}}           & BF16   & 0.257 & 0.488 & 0.583 & 0.370 & 0.330 & 0.653 & 0.383 & 0.524 & 0.449 \\
                                                                                                & W16A16 & 0.231 & 0.347 & 0.471 & 0.290 & 0.262 & 0.542 & 0.346 & 0.510 & 0.375 \\
                                                                                                & W16A12 & 0.231 & 0.338 & 0.542 & 0.288 & 0.240 & 0.549 & 0.346 & 0.507 & 0.380 \\
                                                                                                & W16A10 & 0.229 & 0.372 & 0.478 & 0.287 & 0.276 & 0.537 & 0.344 & 0.518 & 0.380 \\
                                                                                                & W16A8  & 0.260 & 0.278 & 0.489 & 0.260 & 0.264 & 0.504 & 0.351 & 0.515 & 0.365 \\
                                                                                                & W16A6  & 0.285 & 0.269 & 0.614 & 0.265 & 0.266 & 0.496 & 0.337 & 0.507 & 0.380 \\
                                                                                                & W16A4  & 0.282 & 0.259 & 0.621 & 0.264 & 0.264 & 0.496 & 0.342 & 0.499 & 0.378 \\ \midrule
\multirow{7}{*}{\rotatebox[origin=c]{90}{\footnotesize\shortstack{Residual \\ KL Shampoo}}}      & BF16   & 0.265 & 0.513 & 0.582 & 0.385 & 0.312 & 0.650 & 0.396 & 0.514 & 0.452 \\
                                                                                                & W16A16 & 0.262 & 0.478 & 0.609 & 0.356 & 0.318 & 0.637 & 0.385 & 0.511 & 0.445 \\
                                                                                                & W16A12 & 0.265 & 0.479 & 0.609 & 0.352 & 0.322 & 0.636 & 0.380 & 0.517 & 0.445 \\
                                                                                                & W16A10 & 0.259 & 0.469 & 0.595 & 0.352 & 0.304 & 0.635 & 0.369 & 0.510 & 0.437 \\
                                                                                                & W16A8  & 0.207 & 0.388 & 0.485 & 0.301 & 0.292 & 0.557 & 0.349 & 0.517 & 0.387 \\
                                                                                                & W16A6  & 0.271 & 0.274 & 0.392 & 0.263 & 0.260 & 0.492 & 0.354 & 0.493 & 0.350 \\
                                                                                                & W16A4  & 0.295 & 0.263 & 0.382 & 0.254 & 0.242 & 0.520 & 0.344 & 0.511 & 0.351 \\ \midrule
\multirow{7}{*}{\rotatebox[origin=c]{90}{\footnotesize\shortstack{Residual-Free \\ KL Shampoo}}} & BF16   & 0.257 & 0.463 & 0.606 & 0.343 & 0.316 & 0.616 & 0.371 & 0.511 & 0.435 \\
                                                                                                & W16A16 & 0.254 & 0.458 & 0.600 & 0.342 & 0.312 & 0.619 & 0.371 & 0.526 & 0.435 \\
                                                                                                & W16A12 & 0.252 & 0.457 & 0.595 & 0.342 & 0.316 & 0.614 & 0.374 & 0.504 & 0.432 \\
                                                                                                & W16A10 & 0.261 & 0.460 & 0.594 & 0.340 & 0.300 & 0.619 & 0.377 & 0.500 & 0.431 \\
                                                                                                & W16A8  & 0.265 & 0.459 & 0.607 & 0.344 & 0.312 & 0.618 & 0.376 & 0.503 & 0.435 \\
                                                                                                & W16A6  & 0.241 & 0.389 & 0.498 & 0.313 & 0.278 & 0.579 & 0.356 & 0.522 & 0.397 \\
                                                                                                & W16A4  & 0.271 & 0.268 & 0.418 & 0.256 & 0.238 & 0.505 & 0.333 & 0.490 & 0.347 \\ \bottomrule
\end{tabular}%
}
\label{tab:w16_results}
\end{table}
\FloatBarrier

\newpage
\subsection{Residual and Residual-Free Transformers: Weight Full Precision Quantization Results}
In this section, we do not quantize weights and analyze the performance of residual and residual-free models only under activation quantization. The results are interestingly consistent with full precision weights. \Cref{fig:acc_vs_bits_a} shows that residual-free models maintain stable accuracy even at low activation bit-widths, while residual models degrade significantly, especially when trained with AdamW. This behavior is consistent with the distortion metrics in \Cref{fig:quant_metrics_a}, where residual-free models achieve higher SQNR and substantially lower normalized MSE across all bit-widths, along with confirming layerwise excess kurtosis and negentropy. 
The evaluations across different weight quantization suggest that activation quantization affects the performance much more significantly than weight quantization, as the performances are similar, for example, A8, W8A8, and W16A8.

\begin{figure}[h!]
    \centering
    \includegraphics[width=0.5\linewidth]{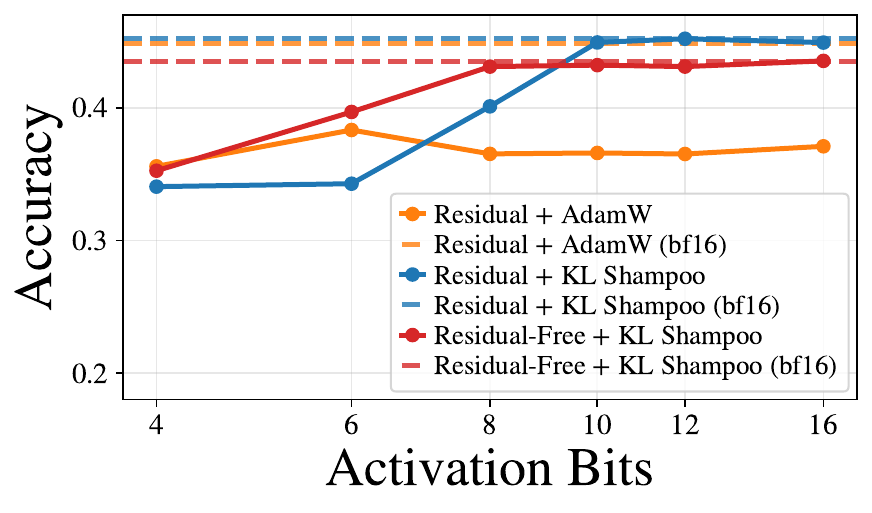}
    \caption{Accuracy as a function of activation bit-width and full precision. Residual-free transformers maintain stable performance even at low precision, while residual models exhibit significant degradation, particularly when trained with AdamW. Dashed lines indicate full-precision (BF16) baselines.}
    \label{fig:acc_vs_bits_a}
\end{figure}

\begin{figure}[h!]
    \centering
    \includegraphics[width=0.49\linewidth]{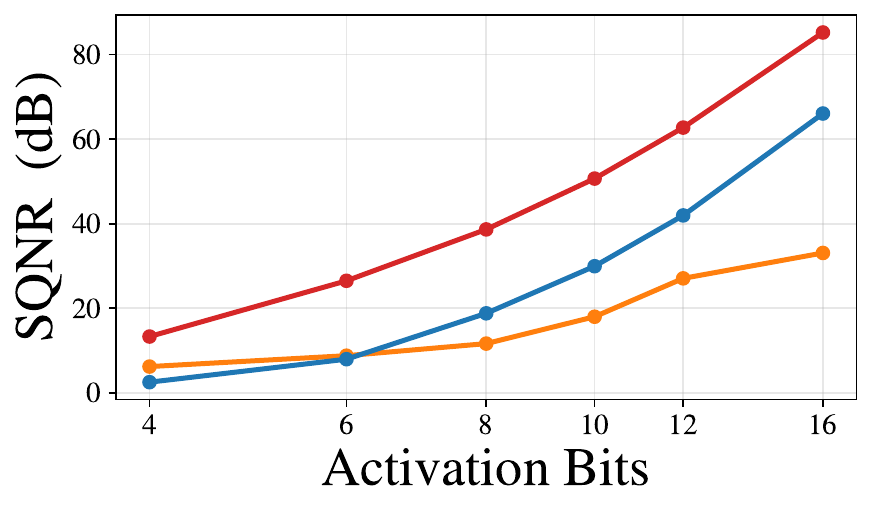}
    \includegraphics[width=0.49\linewidth]{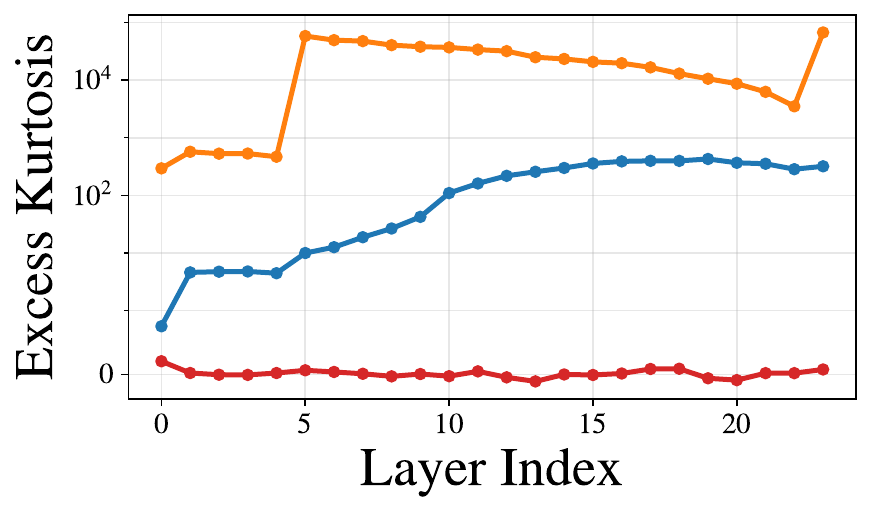}
    \includegraphics[width=0.49\linewidth]{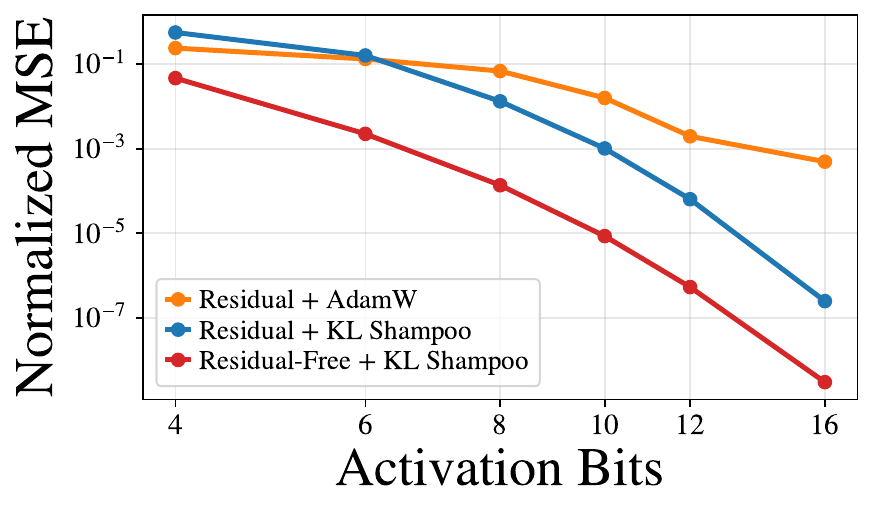}
        \includegraphics[width=0.49\linewidth]{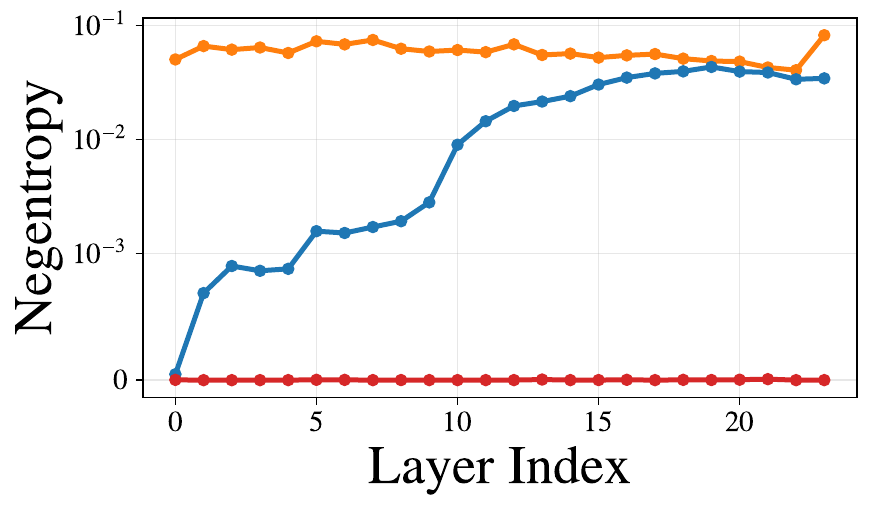}

    \caption{Quantization distortion and activation statistics. \textbf{Top-left:} SQNR improves steadily with bit-width, with residual-free models achieving the highest signal fidelity. \textbf{Bottom-left:} normalized MSE decreases significantly faster for residual-free models, especially at low precision. \textbf{Right:} layerwise excess kurtosis and negentropy show that residual models accumulate heavy-tailed, non-Gaussian activations, while residual-free models remain close to Gaussian. These statistical differences explain the improved quantization robustness.}
    \label{fig:quant_metrics_a}
\end{figure}

\begin{table}[h!]
\caption{Zero-shot performance (Accuracy $\%$) on 8 downstream tasks under \emph{activation-only} quantization (weights kept in BF16).}
\centering
\resizebox{\textwidth}{!}{%
\begin{tabular}{@{}l|l|llllllll|l@{}}
\toprule
\textbf{Model} & \textbf{Quantization} & \textbf{arc\_c} & \textbf{arc\_e} & \textbf{boolq} & \textbf{hellaswag} & \textbf{openbookqa} & \textbf{piqa} & \textbf{social\_iqa} & \textbf{winogrande} & \textbf{avg} \\ \midrule
\multirow{7}{*}{\rotatebox[origin=c]{90}{\footnotesize\shortstack{Residual \\ AdamW}}}           & BF16 & 0.257 & 0.488 & 0.583 & 0.370 & 0.330 & 0.653 & 0.383 & 0.524 & 0.449 \\
                                                                                                & A16  & 0.219 & 0.348 & 0.460 & 0.285 & 0.254 & 0.550 & 0.341 & 0.511 & 0.371 \\
                                                                                                & A12  & 0.241 & 0.342 & 0.402 & 0.280 & 0.268 & 0.529 & 0.354 & 0.506 & 0.365 \\
                                                                                                & A10  & 0.216 & 0.355 & 0.449 & 0.289 & 0.254 & 0.526 & 0.341 & 0.498 & 0.366 \\
                                                                                                & A8   & 0.271 & 0.274 & 0.487 & 0.263 & 0.276 & 0.495 & 0.352 & 0.504 & 0.365 \\
                                                                                                & A6   & 0.278 & 0.265 & 0.621 & 0.264 & 0.280 & 0.492 & 0.339 & 0.527 & 0.383 \\
                                                                                                & A4   & 0.285 & 0.252 & 0.476 & 0.252 & 0.262 & 0.487 & 0.343 & 0.490 & 0.356 \\ \midrule
\multirow{7}{*}{\rotatebox[origin=c]{90}{\footnotesize\shortstack{Residual \\ KL Shampoo}}}      & BF16 & 0.265 & 0.513 & 0.582 & 0.385 & 0.312 & 0.650 & 0.396 & 0.514 & 0.452 \\
                                                                                                & A16  & 0.266 & 0.509 & 0.589 & 0.376 & 0.312 & 0.638 & 0.387 & 0.516 & 0.449 \\
                                                                                                & A12  & 0.271 & 0.508 & 0.590 & 0.378 & 0.318 & 0.640 & 0.379 & 0.533 & 0.452 \\
                                                                                                & A10  & 0.269 & 0.498 & 0.588 & 0.378 & 0.330 & 0.640 & 0.380 & 0.511 & 0.449 \\
                                                                                                & A8   & 0.213 & 0.423 & 0.575 & 0.310 & 0.280 & 0.574 & 0.346 & 0.488 & 0.401 \\
                                                                                                & A6   & 0.274 & 0.265 & 0.381 & 0.259 & 0.228 & 0.503 & 0.342 & 0.491 & 0.343 \\
                                                                                                & A4   & 0.267 & 0.266 & 0.380 & 0.251 & 0.234 & 0.493 & 0.347 & 0.487 & 0.341 \\ \midrule
\multirow{7}{*}{\rotatebox[origin=c]{90}{\footnotesize\shortstack{Residual-Free \\ KL Shampoo}}} & BF16 & 0.257 & 0.463 & 0.606 & 0.343 & 0.316 & 0.616 & 0.371 & 0.511 & 0.435 \\
                                                                                                & A16  & 0.259 & 0.460 & 0.601 & 0.343 & 0.314 & 0.616 & 0.375 & 0.515 & 0.435 \\
                                                                                                & A12  & 0.247 & 0.457 & 0.601 & 0.344 & 0.310 & 0.619 & 0.370 & 0.503 & 0.431 \\
                                                                                                & A10  & 0.252 & 0.457 & 0.604 & 0.341 & 0.308 & 0.621 & 0.366 & 0.508 & 0.432 \\
                                                                                                & A8   & 0.247 & 0.456 & 0.610 & 0.342 & 0.308 & 0.610 & 0.369 & 0.505 & 0.431 \\
                                                                                                & A6   & 0.235 & 0.385 & 0.545 & 0.309 & 0.258 & 0.565 & 0.357 & 0.521 & 0.397 \\
                                                                                                & A4   & 0.300 & 0.270 & 0.415 & 0.262 & 0.234 & 0.499 & 0.339 & 0.501 & 0.353 \\ \bottomrule
\end{tabular}%
}
\label{tab:aonly_results}
\end{table}
\FloatBarrier

\subsection{Comparison of Quantization Techniques: SmoothQuant vs Uniform}
\label{app:smoothquant}
We evaluate SmoothQuant \citep{xiao2023smoothquant}, an outlier-aware post-training quantization (PTQ) method with smoothing strength 0.5 and the same evaluation pipeline as uniform quantization. The results show that SmoothQuant produces only small changes relative to uniform quantization, and \textbf{residual-free advantage remains after applying SmoothQuant}. This supports the insight that residual-free training produces activation statistics that are intrinsically more favorable before post-training corrections are applied. 

\begin{table}[h!]
\caption{Comparison of zero-shot performance between SmoothQuant and uniform post training quantization techniques, showing the results are robust to quantization techniques.}
\begin{tabular}{@{}lllllll@{}}
\toprule
\textbf{Model}             & \textbf{PTQ} & \textbf{BF16} & \textbf{W8A16} & \textbf{W8A8} & \textbf{W8A6} & \textbf{W8A4} \\ \midrule
Residual + AdamW           & uniform      & 0.449         & 0.379          & 0.369         & 0.380         & 0.340         \\
Residual + AdamW           & SmoothQuant  & 0.448         & 0.377          & 0.370         & 0.378         & 0.349         \\
Residual + KL-Shampoo      & uniform      & 0.452         & 0.449          & 0.402         & 0.348         & 0.345         \\
Residual + KL-Shampoo      & SmoothQuant  & 0.449         & 0.442          & 0.397         & 0.345         & 0.342         \\
Residual-free + KL-Shampoo & uniform      & 0.435         & 0.431          & 0.429         & 0.414         & 0.347         \\
Residual-free + KL-Shampoo & SmoothQuant  & 0.436         & 0.432          & 0.431         & 0.404         & 0.346        \\
\bottomrule
\end{tabular}
\label{tab:smoothquant}
\end{table}
\FloatBarrier

\subsection{Calibration Sensitivity in Zero Shot Evaluation}
\label{app:calibration_seed_variability}

We assess the sensitivity of post training quantization and zero shot performance to calibration sampling using three seeds, $\{0,1,2\}$.
For each pretrained checkpoint, we repeat calibration and quantization with these seeds and evaluate the resulting models on the same eight downstream tasks.
The pretrained weights and evaluation protocol remain fixed across runs.
The reported uncertainty therefore characterizes variability due to calibration sampling and does not capture variability across pretraining runs.
Table~\ref{tab:calibration_seed_summary} reports the average accuracy across tasks. \textbf{The uncertainty estimates preserve the main finding.} Residual Free + KL Shampoo stays the most robust model to quantization by losing only 0.5\%, 0.7\%, and 2.3\% relative to its mean BF16 accuracy under W8A16, W8A8 and W8A6, respectively.

\begin{table}[h!]
    \centering
    \caption{
        Average accuracy across eight downstream tasks.
        Quantized entries report the mean and associated uncertainty across three calibration seeds.
        Accuracies and $\Delta$ are expressed as fractions, with $\Delta$ denoting the change from the corresponding BF16 baseline.
        The BF16 values correspond to the fixed pretrained checkpoints.
    }
    \label{tab:calibration_seed_summary}
    \small
    \begin{tabular}{@{}lccccc@{}}
        \toprule
        \textbf{Model} & \textbf{BF16} & \textbf{W8A16} & \textbf{W8A8} & \textbf{W8A6} & \textbf{W8A4} \\
        \midrule
        Residual free + KL Shampoo & $0.435$ & $0.430 \pm 0.006$ & $0.428 \pm 0.008$ & $0.412 \pm 0.006$ & $0.347 \pm 0.005$ \\
        $\Delta$ to BF16 & $0$ & $-0.005$ & $-0.007$ & $-0.023$ & $-0.088$ \\
        \addlinespace
        Residual + AdamW & $0.449$ & $0.378 \pm 0.004$ & $0.366 \pm 0.005$ & $0.376 \pm 0.005$ & $0.348 \pm 0.016$ \\
        $\Delta$ to BF16 & $0$ & $-0.071$ & $-0.083$ & $-0.073$ & $-0.101$ \\
        \addlinespace
        Residual + KL Shampoo & $0.452$ & $0.448 \pm 0.011$ & $0.392 \pm 0.004$ & $0.345 \pm 0.002$ & $0.347 \pm 0.004$ \\
        $\Delta$ to BF16 & $0$ & $-0.004$ & $-0.060$ & $-0.107$ & $-0.105$ \\
        \bottomrule
    \end{tabular}
\end{table}

The residual model trained with AdamW exhibits a small nonmonotonic change, with mean accuracy increasing from $36.58\%$ at W8A8 to $37.62\%$ at W8A6.
The detailed results in Table~\ref{tab:calibration_seed_tasks} show that this increase is driven primarily by BoolQ, where accuracy rises from $51.46\%$ to $61.13\%$.
Changes on the other tasks are mixed, so this increase does not reflect a consistent benefit from lower precision across benchmarks.

\begin{table}[h!]
    \centering
    \caption{
        Accuracy on each downstream task for the fixed pretrained checkpoints.
        All accuracies are percentages, and quantized entries report the mean and associated uncertainty across three calibration seeds.
        Average denotes the unweighted mean across the eight tasks.
    }
    \label{tab:calibration_seed_tasks}
    \setlength{\tabcolsep}{3pt}
    \resizebox{\textwidth}{!}{%
    \begin{tabular}{@{}c|l|cccccccc|c@{}}
        \toprule
        \textbf{Model} & \textbf{Quant.} & \textbf{arc\_c} & \textbf{arc\_e} & \textbf{boolq} & \textbf{hellaswag} & \textbf{openbookqa} & \textbf{piqa} & \textbf{social\_iqa} & \textbf{winogrande} & \textbf{avg} \\
        \midrule
        \multirow{5}{*}{\rotatebox[origin=c]{90}{\shortstack{Residual\\AdamW}}}
        & BF16 & $25.68$ & $48.82$ & $58.29$ & $37.05$ & $33.00$ & $65.34$ & $38.28$ & $52.41$ & $44.86$ \\
        & W8A16 & $23.41 \pm 2.36$ & $36.25 \pm 1.31$ & $47.79 \pm 3.21$ & $29.12 \pm 0.51$ & $26.53 \pm 1.25$ & $54.43 \pm 1.17$ & $34.63 \pm 0.45$ & $50.46 \pm 4.92$ & $37.83 \pm 0.35$ \\
        & W8A8 & $26.91 \pm 1.59$ & $27.48 \pm 0.69$ & $51.46 \pm 2.41$ & $26.32 \pm 0.07$ & $25.80 \pm 0.99$ & $49.40 \pm 1.76$ & $34.72 \pm 0.57$ & $50.54 \pm 1.81$ & $36.58 \pm 0.52$ \\
        & W8A6 & $26.88 \pm 2.87$ & $25.91 \pm 1.05$ & $61.13 \pm 2.89$ & $26.21 \pm 0.82$ & $27.07 \pm 2.07$ & $49.98 \pm 1.76$ & $33.18 \pm 2.19$ & $50.62 \pm 2.56$ & $37.62 \pm 0.54$ \\
        & W8A4 & $27.56 \pm 0.00$ & $26.67 \pm 1.42$ & $39.61 \pm 7.41$ & $25.84 \pm 0.04$ & $24.47 \pm 3.79$ & $49.06 \pm 0.87$ & $34.61 \pm 1.03$ & $50.41 \pm 4.13$ & $34.78 \pm 1.60$ \\
        \midrule
        \multirow{5}{*}{\rotatebox[origin=c]{90}{\shortstack{Residual\\KL Shampoo}}}
        & BF16 & $26.45$ & $51.30$ & $58.17$ & $38.53$ & $31.20$ & $65.02$ & $39.61$ & $51.38$ & $45.21$ \\
        & W8A16 & $25.74 \pm 0.80$ & $51.15 \pm 0.60$ & $58.95 \pm 1.14$ & $37.26 \pm 0.19$ & $31.20 \pm 2.77$ & $63.57 \pm 1.22$ & $38.26 \pm 0.83$ & $52.22 \pm 3.85$ & $44.79 \pm 1.13$ \\
        & W8A8 & $21.84 \pm 1.48$ & $39.32 \pm 1.15$ & $53.29 \pm 1.89$ & $30.23 \pm 0.24$ & $27.87 \pm 1.15$ & $55.75 \pm 2.46$ & $34.54 \pm 0.71$ & $50.64 \pm 4.01$ & $39.19 \pm 0.36$ \\
        & W8A6 & $26.99 \pm 2.12$ & $26.12 \pm 1.02$ & $38.37 \pm 0.61$ & $26.13 \pm 0.87$ & $24.47 \pm 3.31$ & $49.85 \pm 1.35$ & $33.81 \pm 0.48$ & $50.59 \pm 3.76$ & $34.54 \pm 0.23$ \\
        & W8A4 & $28.19 \pm 1.28$ & $26.50 \pm 2.48$ & $37.96 \pm 0.34$ & $25.43 \pm 0.15$ & $25.33 \pm 1.52$ & $50.05 \pm 2.31$ & $34.61 \pm 2.04$ & $49.72 \pm 1.37$ & $34.72 \pm 0.42$ \\
        \midrule
        \multirow{5}{*}{\rotatebox[origin=c]{90}{\shortstack{Residual free\\KL Shampoo}}}
        & BF16 & $25.68$ & $46.25$ & $60.55$ & $34.31$ & $31.60$ & $61.64$ & $37.10$ & $51.07$ & $43.53$ \\
        & W8A16 & $24.60 \pm 0.44$ & $45.08 \pm 0.00$ & $58.98 \pm 0.52$ & $34.33 \pm 0.10$ & $31.00 \pm 1.31$ & $61.53 \pm 1.15$ & $36.68 \pm 1.10$ & $51.49 \pm 1.39$ & $42.96 \pm 0.64$ \\
        & W8A8 & $24.86 \pm 1.30$ & $44.67 \pm 0.79$ & $58.66 \pm 1.08$ & $34.38 \pm 0.55$ & $30.67 \pm 2.74$ & $60.74 \pm 0.56$ & $36.34 \pm 1.21$ & $51.96 \pm 2.67$ & $42.78 \pm 0.79$ \\
        & W8A6 & $25.88 \pm 2.76$ & $41.68 \pm 1.78$ & $54.17 \pm 2.87$ & $32.65 \pm 0.17$ & $30.33 \pm 2.74$ & $57.60 \pm 3.04$ & $36.30 \pm 1.60$ & $50.62 \pm 1.78$ & $41.15 \pm 0.56$ \\
        & W8A4 & $28.01 \pm 3.41$ & $26.82 \pm 2.06$ & $40.16 \pm 0.81$ & $25.80 \pm 0.51$ & $24.20 \pm 2.28$ & $49.09 \pm 1.78$ & $34.82 \pm 0.48$ & $49.01 \pm 1.09$ & $34.74 \pm 0.53$ \\
        \bottomrule
    \end{tabular}%
    }
\end{table}

\FloatBarrier

\subsection{QK Normalization as an Alternative to Temperature Scaling}
\label{sec:qk_normalization}

QK normalization provides an alternative way to control attention logit scales by normalizing the query and key vectors.
We compare this approach with the proposed depth-aware temperature schedule in the residual-free architecture.
Both configurations use orthogonal initialization, KL Shampoo, and a learning rate of $5 \times 10^{-4}$ under otherwise matched training conditions.
With QK normalization, the training negative log-likelihood plateaus near $7.4$ within the first few hundred steps and shows no subsequent improvement during the observed run.
The depth-aware temperature configuration reaches a substantially lower negative log-likelihood of approximately $2.67$.

The two approaches modify attention through different mechanisms.
QK normalization controls query and key magnitudes within each attention head, whereas the proposed temperature schedule explicitly prescribes how attention logit scales vary with depth.
Temperature scaling affects both the softmax distribution and the gradients propagated to the query and key projections.
Its depth dependence may therefore help regulate attention sensitivity and gradient scales across layers in residual-free networks.
QK normalization alone does not impose this prescribed depth schedule.

\FloatBarrier

\subsection{Ablations of Residual Removal vs the Proposed Training Recipe}
\label{sec:architecture_recipe_ablations}

We use matched controls and component ablations to distinguish the effect of residual connectivity on activation statistics from the contributions of initialization, optimizer choice, and temperature scaling to trainability.
The matched architecture comparison holds the training recipe and learning rate fixed.

\paragraph{Matched architecture comparison.}
We compare residual and residual free transformers using the same orthogonal initialization, KL Shampoo optimizer, depth aware temperature schedule, learning rate of $10^{-3}$, data, and training protocol.
Residual connectivity is the only experimental difference between this pair.

As shown in Table~\ref{tab:ablation_activation_kurtosis}, the difference in activation statistics emerges early and increases during training.
At step $12{,}000$, the mean excess kurtosis across layers is $39.53$ for the residual model and $0.02$ for the residual free model.
At step $50{,}000$, these values reach $115.35$ and $0.01$, respectively, a difference of approximately four orders of magnitude.
\textbf{The shared training recipe is therefore insufficient to reproduce near zero excess kurtosis when the residual path is present.}
This observation is consistent with the distinct kurtosis propagation regimes described in \Cref{lem:ex_kurt,lem:skip_kurt}.

\begin{table}[h!]
    \centering
    \caption{
        Mean activation excess kurtosis across all layers at selected training steps.
        Temp.\ indicates whether depth aware temperature scaling is enabled.
        The final two rows form the matched architecture comparison, with all settings other than residual connectivity held fixed.
    }
    \label{tab:ablation_activation_kurtosis}
    \setlength{\tabcolsep}{3pt}
    \begin{tabular}{@{}lllcrrrr@{}}
        \toprule
        & & & & \multicolumn{4}{c}{\textbf{Training step}} \\
        \cmidrule(lr){5-8}
        \textbf{Architecture} & \textbf{Optimizer} & \textbf{Initialization} & \textbf{Temp.} & \textbf{12k} & \textbf{20k} & \textbf{40k} & \textbf{50k} \\
        \midrule
        Residual & AdamW & Default & No & $3204.32$ & $8333.78$ & $19349.20$ & $23142.12$ \\
        Residual & KL Shampoo & Default & No & $56.33$ & $107.12$ & $173.73$ & $186.19$ \\
        Residual & KL Shampoo & Orthogonal & No & $39.44$ & $70.53$ & $107.78$ & $113.51$ \\
        Residual & KL Shampoo & Orthogonal & Yes & $39.53$ & $70.69$ & $108.94$ & $115.35$ \\
        Residual free & KL Shampoo & Orthogonal & Yes & $0.02$ & $0.01$ & $0.01$ & $0.01$ \\
        \bottomrule
    \end{tabular}
\end{table}

\paragraph{Effect of initialization and depth-aware temperature scaling.}
The residual controls further clarify the contributions of orthogonal initialization and temperature scaling.
Replacing the default initialization with orthogonal initialization reduces the mean excess kurtosis at step $50{,}000$ from $186.19$ to $113.51$.
It also increases W8A8 accuracy from $40.21\%$ to $43.75\%$, while BF16 accuracy changes from $45.21\%$ to $44.94\%$.
Consequently, the W8A8 accuracy loss from the corresponding BF16 baseline decreases from $5.00$ to $1.19$ percentage points.
\textbf{Orthogonal initialization therefore improves quantization robustness in these residual models}, although their excess kurtosis remains substantially above that of the residual-free model.

Adding depth-aware temperature scaling changes the final excess kurtosis from $113.51$ to $115.35$ and reduces W8A8 accuracy from $43.75\%$ to $41.47\%$.
It also provides no accuracy improvement at W8A6 or W8A4 in these controls.
These results show that \textbf{adding the temperature schedule to a residual model does not reproduce the activation statistics or quantization robustness} of the residual-free configuration.
Table~\ref{tab:residual_recipe_accuracy} reports the average accuracies, and Table~\ref{tab:residual_recipe_tasks} provides the detailed task results.

\begin{table}[h!]
    \centering
    \caption{
        Average zero shot accuracy across eight downstream tasks for residual transformers trained with KL Shampoo.
        Temperature scaling is disabled unless explicitly indicated.
    }
    \label{tab:residual_recipe_accuracy}
    \setlength{\tabcolsep}{4pt}
    \begin{tabular}{@{}lccccc@{}}
        \toprule
        \textbf{Configuration} & \textbf{BF16} & \textbf{W8A16} & \textbf{W8A8} & \textbf{W8A6} & \textbf{W8A4} \\
        \midrule
        Default initialization & $0.4521$ & $0.4485$ & $0.4021$ & $0.3484$ & $0.3454$ \\
        Orthogonal initialization & $0.4494$ & $0.4409$ & $0.4375$ & $0.3534$ & $0.3496$ \\
        Orthogonal + temperature scaling & $0.4457$ & $0.4410$ & $0.4147$ & $0.3497$ & $0.3479$ \\
        \bottomrule
    \end{tabular}
\end{table}

\paragraph{Trainability of residual-free models.}
We also ablate initialization and optimizer choice within the residual-free architecture.
With default initialization and KL Shampoo, all tested learning rates, $\{7 \times 10^{-5}, 10^{-4}, 3 \times 10^{-4}\}$, exhibit an early plateau near a negative log likelihood of $7.38$ within the first $1{,}000$ training steps.
With orthogonal initialization and AdamW, the best stable run reaches a negative log-likelihood of $6.48$ at step $20{,}000$, whereas the complete residual-free recipe achieves $2.67$.
These results support the \textbf{importance of orthogonal initialization and KL Shampoo for optimizing residual-free transformers} under the evaluated settings.
The temperature scaling ablation in \Cref{fig:residualfree-optimizers} (middle) additionally demonstrates a convergence benefit. In contrast, the matched residual comparison shows that \textbf{removing the residual path is the principal factor behind Gaussianity} by the contraction of higher order activation statistics.

Therefore, the ablations separate two roles. \textbf{Residual removal enables near-Gaussianity} by layerwise contraction of higher-order activation statistics, while \textbf{orthogonal initialization and second order optimization make this regime trainable} at depth. \textbf{Temperature scaling provides an additional convergence benefit} but does not explain the Gaussianity gap. 

\begin{table}[h!]
    \centering
    \caption{
        Accuracy on each downstream task for residual transformers trained with KL Shampoo.
        All accuracies are expressed as decimals, and average denotes the unweighted mean across the eight tasks.
        Default and Orthogonal specify the initialization, while Temp denotes the addition of depth aware temperature scaling.
    }
    \label{tab:residual_recipe_tasks}
    \small
    \setlength{\tabcolsep}{3pt}
    \resizebox{\textwidth}{!}{%
    \begin{tabular}{@{}c|l|cccccccc|c@{}}
        \toprule
        \textbf{Setup} & \textbf{Quantization} & \textbf{arc\_c} & \textbf{arc\_e} & \textbf{boolq} & \textbf{hellaswag} & \textbf{openbookqa} & \textbf{piqa} & \textbf{social\_iqa} & \textbf{winogrande} & \textbf{avg} \\
        \midrule
        \multirow{5}{*}{\rotatebox[origin=c]{90}{\shortstack{Residual\\Default}}}
        & BF16 & $0.2645$ & $0.5130$ & $0.5817$ & $0.3853$ & $0.3120$ & $0.6502$ & $0.3961$ & $0.5138$ & $0.4521$ \\
        & W8A16 & $0.2551$ & $0.5093$ & $0.5945$ & $0.3743$ & $0.3040$ & $0.6398$ & $0.3884$ & $0.5225$ & $0.4485$ \\
        & W8A8 & $0.2218$ & $0.4011$ & $0.5783$ & $0.3106$ & $0.2700$ & $0.5653$ & $0.3490$ & $0.5209$ & $0.4021$ \\
        & W8A6 & $0.2765$ & $0.2626$ & $0.3801$ & $0.2576$ & $0.2440$ & $0.5054$ & $0.3480$ & $0.5130$ & $0.3484$ \\
        & W8A4 & $0.2858$ & $0.2580$ & $0.3823$ & $0.2528$ & $0.2380$ & $0.4924$ & $0.3429$ & $0.5107$ & $0.3454$ \\
        \midrule
        \multirow{5}{*}{\rotatebox[origin=c]{90}{\shortstack{Residual\\Orthogonal}}}
        & BF16 & $0.2790$ & $0.4920$ & $0.5976$ & $0.3729$ & $0.3040$ & $0.6425$ & $0.3777$ & $0.5296$ & $0.4494$ \\
        & W8A16 & $0.2654$ & $0.4773$ & $0.5765$ & $0.3657$ & $0.3040$ & $0.6360$ & $0.3767$ & $0.5257$ & $0.4409$ \\
        & W8A8 & $0.2602$ & $0.4613$ & $0.5859$ & $0.3583$ & $0.3320$ & $0.6115$ & $0.3669$ & $0.5241$ & $0.4375$ \\
        & W8A6 & $0.2747$ & $0.2702$ & $0.3896$ & $0.2622$ & $0.2700$ & $0.5136$ & $0.3414$ & $0.5051$ & $0.3534$ \\
        & W8A4 & $0.2867$ & $0.2694$ & $0.3810$ & $0.2549$ & $0.2620$ & $0.5000$ & $0.3470$ & $0.4957$ & $0.3496$ \\
        \midrule
        \multirow{5}{*}{\rotatebox[origin=c]{90}{\shortstack{Residual\\Orthogonal\\Temp}}}
        & BF16 & $0.2662$ & $0.5080$ & $0.5691$ & $0.3736$ & $0.3120$ & $0.6371$ & $0.3849$ & $0.5146$ & $0.4457$ \\
        & W8A16 & $0.2662$ & $0.4865$ & $0.5691$ & $0.3719$ & $0.3160$ & $0.6295$ & $0.3833$ & $0.5051$ & $0.4410$ \\
        & W8A8 & $0.2253$ & $0.4440$ & $0.5333$ & $0.3359$ & $0.3080$ & $0.5958$ & $0.3593$ & $0.5162$ & $0.4147$ \\
        & W8A6 & $0.2773$ & $0.2694$ & $0.3872$ & $0.2583$ & $0.2620$ & $0.4951$ & $0.3501$ & $0.4980$ & $0.3497$ \\
        & W8A4 & $0.2807$ & $0.2567$ & $0.3804$ & $0.2614$ & $0.2640$ & $0.4973$ & $0.3373$ & $0.5051$ & $0.3479$ \\
        \bottomrule
    \end{tabular}%
    }
\end{table}
\FloatBarrier

\subsection{Scaling to Larger Models}
\label{app:larger_models}

We examine whether training stability and quantization robustness persist when increasing model size from 194M to 730M parameters.
The larger model doubles the hidden dimension from $768$ to $1536$, while retaining the same 24 layers and training and evaluation protocol.
For quantized evaluation, we use three calibration seeds to estimate variability due to calibration sampling.
This experiment provides an initial assessment of scaling with width at a fixed depth.

\begin{table}[h!]
    \centering
    \caption{
        \textbf{Average accuracy across eight downstream tasks at two model sizes.}
        All models are trained with KL Shampoo.
        Quantized results include uncertainty estimates obtained using three calibration seeds.
        $\Delta$ denotes the change from the corresponding model's BF16 accuracy.
        Both accuracy and $\Delta$ are expressed as fractions.
    }
    \label{tab:model_scaling_quantization}
    \setlength{\tabcolsep}{3pt}
    \begin{tabular}{@{}lccccc@{}}
        \toprule
        \textbf{Model} & \textbf{BF16} & \textbf{W8A16} & \textbf{W8A8} & \textbf{W8A6} & \textbf{W8A4} \\
        \midrule
        \multicolumn{6}{@{}l}{\textbf{194M parameters}} \\
        \addlinespace
        Residual + KL Shampoo & $0.452$ & $0.448 \pm 0.011$ & $0.392 \pm 0.004$ & $0.345 \pm 0.002$ & $0.347 \pm 0.004$ \\
        $\Delta$ from BF16 & $0$ & $\mathbf{-0.004}$ & $-0.060$ & $-0.107$ & $-0.105$ \\
        \addlinespace
        Residual free + KL Shampoo & $0.435$ & $0.430 \pm 0.006$ & $0.428 \pm 0.008$ & $0.412 \pm 0.006$ & $0.347 \pm 0.005$ \\
        $\Delta$ from BF16 & $0$ & $-0.005$ & $\mathbf{-0.007}$ & $\mathbf{-0.023}$ & $\mathbf{-0.088}$ \\
        \midrule
        \multicolumn{6}{@{}l}{\textbf{730M parameters}} \\
        \addlinespace
        Residual + KL Shampoo & $0.484$ & $0.389 \pm 0.009$ & $0.362 \pm 0.002$ & $0.349 \pm 0.005$ & $0.346 \pm 0.011$ \\
        $\Delta$ from BF16 & $0$ & $-0.095$ & $-0.122$ & $-0.135$ & $-0.138$ \\
        \addlinespace
        Residual free + KL Shampoo & $0.457$ & $0.462 \pm 0.002$ & $0.456 \pm 0.003$ & $0.350 \pm 0.003$ & $0.350 \pm 0.008$ \\
        $\Delta$ from BF16 & $0$ & $\mathbf{+0.005}$ & $\mathbf{-0.001}$ & $\mathbf{-0.107}$ & $\mathbf{-0.107}$ \\
        \bottomrule
    \end{tabular}
\end{table}

Training with the residual free recipe remains stable at 730M parameters.
As shown in Table~\ref{tab:model_scaling_quantization}, increasing model size improves BF16 accuracy for both architectures, from $0.452$ to $0.484$ for the residual model and from $0.435$ to $0.457$ for the residual free model.
Although the residual model retains higher BF16 accuracy, its performance deteriorates substantially under quantization.
At 730M parameters, W8A16 and W8A8 reduce its accuracy by $9.5$ and $12.2$ percentage points, respectively.
The corresponding changes for the residual free model are $+0.5$ and $-0.1$ percentage points, indicating that its mean accuracy remains close to the BF16 baseline at these precisions.

The precision at which the residual free model overtakes the residual model also changes with model size.
At 194M parameters, the residual model remains stronger at W8A16, where its accuracy decreases by only $0.4$ percentage points, and the residual free model first achieves higher mean accuracy at W8A8.
At 730M parameters, this crossover already occurs at W8A16.
Under the evaluated protocol, the larger residual model therefore incurs a greater quantization penalty, while the residual free model largely preserves its BF16 accuracy at W8A16 and W8A8.

At W8A6 and W8A4, both 730M models deteriorate to approximately $0.35$ average accuracy, substantially narrowing the difference between architectures.
The preservation of BF16 accuracy observed at W8A16 and W8A8 therefore does not extend uniformly to lower activation precisions.
Determining how training stability and the quantization crossover depend on depth, width, and training budget requires a broader experimental sweep.
\FloatBarrier

\subsection{Activation Outliers with Complex Learned Residual Aggregation}
\label{app:learned_residual_aggregation}

We examine how the kurtosis propagation mechanisms in \Cref{lem:ex_kurt,lem:skip_kurt} extend to architectures that aggregate multiple residual representations, such as attention residuals \citep{team2026attention} and hyperconnections \citep{zhu2025hyper}. Under its assumptions, \Cref{lem:ex_kurt} shows that dense mixing across $d$ coordinates contracts excess kurtosis by a factor of order $1/d$. In contrast, the identity path in \Cref{lem:skip_kurt} preserves a direct contribution from the previous activation, preventing the same contraction in general.
Learned residual aggregation introduces an intermediate
possibility: averaging across residual representations may
reduce excess kurtosis, depending on the mixing coefficients
and the dependence between representations.

\paragraph{An averaging model for residual aggregation.}
Consider one coordinate of a residual update,
\begin{align*}
    y = s + b,
    \qquad
    s = \sum_{r=1}^{m} a_r h_r,
\end{align*}
where $h_r$ are residual representations and $b$ is the new
branch output. Assume that the coefficients are deterministic
and not all zero, and that the $h_r$ are independent, centered
random variables with variances $\sigma_r^2 > 0$ and finite
fourth moments. Additivity of fourth cumulants gives
\begin{equation}
    \gamma_s
    =
    \frac{
        \sum_{r=1}^{m} a_r^4 \sigma_r^4 \gamma_{h_r}
    }{
        \left(\sum_{r=1}^{m} a_r^2 \sigma_r^2\right)^2
    }.
    \label{eq:residual_aggregation_kurtosis}
\end{equation}
For equal variances and a common excess kurtosis $\gamma_h$,
this simplifies to
\begin{equation}
    \gamma_s
    =
    \gamma_h
    \frac{\sum_{r=1}^{m} a_r^4}
         {\left(\sum_{r=1}^{m} a_r^2\right)^2}
    =
    \frac{\gamma_h}{m_{\mathrm{eff}}},
    \qquad
    m_{\mathrm{eff}}
    =
    \frac{\left(\sum_{r=1}^{m} a_r^2\right)^2}
         {\sum_{r=1}^{m} a_r^4},
    \label{eq:effective_residual_streams}
\end{equation}
where $1 \leq m_{\mathrm{eff}} \leq m$.
When one coefficient dominates, $m_{\mathrm{eff}}$ is close
to one, and the mixture largely preserves the excess kurtosis
of that representation. When all coefficients have equal
magnitude, $m_{\mathrm{eff}}=m$, yielding a contraction
factor of $1/m$. When $m \ll d$, this provides a weaker
nominal averaging effect than the $1/d$ contraction from
dense coordinate mixing in \Cref{lem:ex_kurt}.
These identities characterize the residual mixture $s$ under
the stated assumptions, rather than the full update $y$.
The branch output $b$, dependence between representations,
and coefficients that depend on activations can alter the
resulting kurtosis. The calculation therefore identifies
a possible averaging mechanism, rather than a general
contraction guarantee for trained architectures with learned
residual aggregation. Their behavior depends on whether the learned residual aggregation is concentrated, in which case non-Gaussianity is preserved, or diffused, in which case cross-stream averaging can reduce it by a factor related to the effective number of active streams.
 
\paragraph{Experimental setup.}
To validate the above theory, we evaluate three publicly released checkpoints from a Qwen3 style 0.6B model family, each with 28 layers and 655M tokens:
a standard residual model, Full Attention Residuals,
and Block Attention Residuals.
The corresponding Hugging Face repositories are:
\begin{itemize}
    \item Standard residual:
    \url{https://huggingface.co/Ethangou/attention-residuals-0.6B-baseline}
    \item Full Attention Residuals:
    \url{https://huggingface.co/Ethangou/attention-residuals-0.6B-full}
    \item Block Attention Residuals:
    \url{https://huggingface.co/Ethangou/attention-residuals-0.6B-block}
\end{itemize}
The model cards report 20,000 training steps for each
checkpoint, with a training sequence length of 1,024 for
Full Attention Residuals and 2,048 for the standard and
block variants. Consequently, this experiment compares
released checkpoints rather than isolating residual
aggregation under identical training conditions.
For each checkpoint, we measure excess kurtosis after
the attention and MLP normalization layers using the
same evaluation data.

\begin{table}[h!]
    \centering
    \caption{
        Activation excess kurtosis at selected layers of Qwen3 style checkpoints with different residual
        aggregation schemes, evaluated on the same data.
        Attention and MLP denote activations 
        after the corresponding normalization layers.
        A Gaussian distribution has zero excess kurtosis.
    }
    \label{tab:learned_residual_kurtosis}
    \setlength{\tabcolsep}{4pt}
    \begin{tabular}{@{}llrrrrrr@{}}
        \toprule
        & & \multicolumn{6}{c}{\textbf{Layer}} \\
        \cmidrule(lr){3-8}
        \textbf{Residual architecture} & \textbf{Activation}
        & \textbf{3} & \textbf{9} & \textbf{12} & \textbf{17} & \textbf{20} & \textbf{26} \\
        \midrule
        Standard residual
        & Attention
        & 192.23 & 113.42 & 69.39 & 14.05 & 3.64 & 2.32 \\
        & MLP
        & 205.78 & 114.12 & 65.80 & 11.15 & 3.37 & 2.30 \\
        \addlinespace
        Full Attention Residuals
        & Attention
        & 1.30 & 8.62 & 130.36 & 5.22 & 4.04 & 49.92 \\
        & MLP
        & 2.98 & 12.93 & 14.14 & 8.73 & 4.34 & 23.20 \\
        \addlinespace
        Block Attention Residuals
        & Attention
        & 227.42 & 24.05 & 67.69 & 108.37 & 43.15 & 61.78 \\
        & MLP
        & 565.19 & 44.72 & 36.91 & 21.79 & 18.64 & 19.71 \\
        \bottomrule
    \end{tabular}
\end{table}

Table~\ref{tab:learned_residual_kurtosis} shows that learned
residual aggregation does not uniformly increase or decrease
activation excess kurtosis. 
The observed variation is compatible with the conditional
averaging mechanism described above: aggregation can reduce
excess kurtosis when contributions are sufficiently diffuse
and dependence is sufficiently weak. However, the experiment
does not measure coefficient concentration or dependence
between representations, so it does not directly establish
this mechanism as the cause of the observed differences.
These measurements characterize activation fourth moments
and do not, by themselves, establish a ranking of
quantization robustness among the residual architectures.

\FloatBarrier

\section{Broader Impact Statement}
\label{app:broader_impact}

This work studies how architectural choices influence activation statistics and their implications for low-bit quantization. By enabling more effective use of simple uniform quantization, our approach can reduce the memory and bandwidth requirements of large-scale models, potentially making training and deployment more environmentally
efficient. At the same time, lowering deployment costs may accelerate both beneficial
and harmful uses of foundation models. Our work does not directly address issues such
as misuse, fairness, or safety, as our focus is on the foundational understanding of these models.
Thus, we believe the potential benefits of robustness research outweigh its risks. Additionally, we do not see any immediate risk stemming from our work.


\newpage
\section*{NeurIPS Paper Checklist}

\begin{enumerate}

\item {\bf Claims}
    \item[] Question: Do the main claims made in the abstract and introduction accurately reflect the paper's contributions and scope?
    \item[] Answer: \answerYes{} 
    \item[] Justification: The theoretical analysis in \Cref{sec:kurtosis_analysis} shows when the activations can be near-Gaussian and the experimental analysis in \Cref{sec:exp} supports the analysis as well as shows that residual-free models are substantially more quantization robust than residual ones.
    \item[] Guidelines:
    \begin{itemize}
        \item The answer \answerNA{} means that the abstract and introduction do not include the claims made in the paper.
        \item The abstract and/or introduction should clearly state the claims made, including the contributions made in the paper and important assumptions and limitations. A \answerNo{} or \answerNA{} answer to this question will not be perceived well by the reviewers. 
        \item The claims made should match theoretical and experimental results, and reflect how much the results can be expected to generalize to other settings. 
        \item It is fine to include aspirational goals as motivation as long as it is clear that these goals are not attained by the paper. 
    \end{itemize}

\item {\bf Limitations}
    \item[] Question: Does the paper discuss the limitations of the work performed by the authors?
    \item[] Answer: \answerYes{} 
    \item[] Justification: The limitations are discussed in the introduction (\Cref{sec:intro}), and discussion and conclusion (\Cref{sec:conclusion}). The drop in performance of residual models in full-precision is discussed in \Cref{sec:exp}. All assumptions in the theoretical analysis are clearly stated in the Lemmas and Theorems in \Cref{sec:kurtosis_analysis}.
    \item[] Guidelines:
    \begin{itemize}
        \item The answer \answerNA{} means that the paper has no limitation while the answer \answerNo{} means that the paper has limitations, but those are not discussed in the paper. 
        \item The authors are encouraged to create a separate ``Limitations'' section in their paper.
        \item The paper should point out any strong assumptions and how robust the results are to violations of these assumptions (e.g., independence assumptions, noiseless settings, model well-specification, asymptotic approximations only holding locally). The authors should reflect on how these assumptions might be violated in practice and what the implications would be.
        \item The authors should reflect on the scope of the claims made, e.g., if the approach was only tested on a few datasets or with a few runs. In general, empirical results often depend on implicit assumptions, which should be articulated.
        \item The authors should reflect on the factors that influence the performance of the approach. For example, a facial recognition algorithm may perform poorly when image resolution is low or images are taken in low lighting. Or a speech-to-text system might not be used reliably to provide closed captions for online lectures because it fails to handle technical jargon.
        \item The authors should discuss the computational efficiency of the proposed algorithms and how they scale with dataset size.
        \item If applicable, the authors should discuss possible limitations of their approach to address problems of privacy and fairness.
        \item While the authors might fear that complete honesty about limitations might be used by reviewers as grounds for rejection, a worse outcome might be that reviewers discover limitations that aren't acknowledged in the paper. The authors should use their best judgment and recognize that individual actions in favor of transparency play an important role in developing norms that preserve the integrity of the community. Reviewers will be specifically instructed to not penalize honesty concerning limitations.
    \end{itemize}

\item {\bf Theory assumptions and proofs}
    \item[] Question: For each theoretical result, does the paper provide the full set of assumptions and a complete (and correct) proof?
    \item[] Answer: \answerYes{} 
    \item[] Justification: All the theoretical results are presented in \Cref{sec:kurtosis_analysis} where the assumptions and considerations are clearly discussed as part of Lemmas and Theorems. Corresponding proofs are provided in \Cref{app:kurtosis} and referenced appropriately in the main paper.
    \item[] Guidelines:
    \begin{itemize}
        \item The answer \answerNA{} means that the paper does not include theoretical results. 
        \item All the theorems, formulas, and proofs in the paper should be numbered and cross-referenced.
        \item All assumptions should be clearly stated or referenced in the statement of any theorems.
        \item The proofs can either appear in the main paper or the supplemental material, but if they appear in the supplemental material, the authors are encouraged to provide a short proof sketch to provide intuition. 
        \item Inversely, any informal proof provided in the core of the paper should be complemented by formal proofs provided in appendix or supplemental material.
        \item Theorems and Lemmas that the proof relies upon should be properly referenced. 
    \end{itemize}

    \item {\bf Experimental result reproducibility}
    \item[] Question: Does the paper fully disclose all the information needed to reproduce the main experimental results of the paper to the extent that it affects the main claims and/or conclusions of the paper (regardless of whether the code and data are provided or not)?
    \item[] Answer: \answerYes{} 
    \item[] Justification: The experimental configurations and details to reproduce the results are fully provided in \Cref{sec:exp,app:exp_details}. The implementation details and step by step instructions are described in \Cref{sec:method}.
    \item[] Guidelines:
    \begin{itemize}
        \item The answer \answerNA{} means that the paper does not include experiments.
        \item If the paper includes experiments, a \answerNo{} answer to this question will not be perceived well by the reviewers: Making the paper reproducible is important, regardless of whether the code and data are provided or not.
        \item If the contribution is a dataset and\slash or model, the authors should describe the steps taken to make their results reproducible or verifiable. 
        \item Depending on the contribution, reproducibility can be accomplished in various ways. For example, if the contribution is a novel architecture, describing the architecture fully might suffice, or if the contribution is a specific model and empirical evaluation, it may be necessary to either make it possible for others to replicate the model with the same dataset, or provide access to the model. In general. releasing code and data is often one good way to accomplish this, but reproducibility can also be provided via detailed instructions for how to replicate the results, access to a hosted model (e.g., in the case of a large language model), releasing of a model checkpoint, or other means that are appropriate to the research performed.
        \item While NeurIPS does not require releasing code, the conference does require all submissions to provide some reasonable avenue for reproducibility, which may depend on the nature of the contribution. For example
        \begin{enumerate}
            \item If the contribution is primarily a new algorithm, the paper should make it clear how to reproduce that algorithm.
            \item If the contribution is primarily a new model architecture, the paper should describe the architecture clearly and fully.
            \item If the contribution is a new model (e.g., a large language model), then there should either be a way to access this model for reproducing the results or a way to reproduce the model (e.g., with an open-source dataset or instructions for how to construct the dataset).
            \item We recognize that reproducibility may be tricky in some cases, in which case authors are welcome to describe the particular way they provide for reproducibility. In the case of closed-source models, it may be that access to the model is limited in some way (e.g., to registered users), but it should be possible for other researchers to have some path to reproducing or verifying the results.
        \end{enumerate}
    \end{itemize}

\item {\bf Open access to data and code}
    \item[] Question: Does the paper provide open access to the data and code, with sufficient instructions to faithfully reproduce the main experimental results, as described in supplemental material?
    \item[] Answer: \answerYes{} 
    \item[] Justification: The anonymized code repository as a zip file is submitted as supplementary material to the submission.
    \item[] Guidelines:
    \begin{itemize}
        \item The answer \answerNA{} means that paper does not include experiments requiring code.
        \item Please see the NeurIPS code and data submission guidelines (\url{https://neurips.cc/public/guides/CodeSubmissionPolicy}) for more details.
        \item While we encourage the release of code and data, we understand that this might not be possible, so \answerNo{} is an acceptable answer. Papers cannot be rejected simply for not including code, unless this is central to the contribution (e.g., for a new open-source benchmark).
        \item The instructions should contain the exact command and environment needed to run to reproduce the results. See the NeurIPS code and data submission guidelines (\url{https://neurips.cc/public/guides/CodeSubmissionPolicy}) for more details.
        \item The authors should provide instructions on data access and preparation, including how to access the raw data, preprocessed data, intermediate data, and generated data, etc.
        \item The authors should provide scripts to reproduce all experimental results for the new proposed method and baselines. If only a subset of experiments are reproducible, they should state which ones are omitted from the script and why.
        \item At submission time, to preserve anonymity, the authors should release anonymized versions (if applicable).
        \item Providing as much information as possible in supplemental material (appended to the paper) is recommended, but including URLs to data and code is permitted.
    \end{itemize}

\item {\bf Experimental setting/details}
    \item[] Question: Does the paper specify all the training and test details (e.g., data splits, hyperparameters, how they were chosen, type of optimizer) necessary to understand the results?
    \item[] Answer: \answerYes{} 
    \item[] Justification: All the details are provided in \Cref{sec:exp,app:exp_details}.
    \item[] Guidelines:
    \begin{itemize}
        \item The answer \answerNA{} means that the paper does not include experiments.
        \item The experimental setting should be presented in the core of the paper to a level of detail that is necessary to appreciate the results and make sense of them.
        \item The full details can be provided either with the code, in appendix, or as supplemental material.
    \end{itemize}

\item {\bf Experiment statistical significance}
    \item[] Question: Does the paper report error bars suitably and correctly defined or other appropriate information about the statistical significance of the experiments?
    \item[] Answer: \answerNo{} 
    \item[] Justification: Due to the computational cost of pretraining transformer models, we report single-run results, which is common practice for large-scale model training studies. To improve reliability, we use controlled comparisons under matched training and evaluation settings and report consistent trends across multiple metrics.
    \item[] Guidelines:
    \begin{itemize}
        \item The answer \answerNA{} means that the paper does not include experiments.
        \item The authors should answer \answerYes{} if the results are accompanied by error bars, confidence intervals, or statistical significance tests, at least for the experiments that support the main claims of the paper.
        \item The factors of variability that the error bars are capturing should be clearly stated (for example, train/test split, initialization, random drawing of some parameter, or overall run with given experimental conditions).
        \item The method for calculating the error bars should be explained (closed form formula, call to a library function, bootstrap, etc.)
        \item The assumptions made should be given (e.g., Normally distributed errors).
        \item It should be clear whether the error bar is the standard deviation or the standard error of the mean.
        \item It is OK to report 1-sigma error bars, but one should state it. The authors should preferably report a 2-sigma error bar than state that they have a 96\% CI, if the hypothesis of Normality of errors is not verified.
        \item For asymmetric distributions, the authors should be careful not to show in tables or figures symmetric error bars that would yield results that are out of range (e.g., negative error rates).
        \item If error bars are reported in tables or plots, the authors should explain in the text how they were calculated and reference the corresponding figures or tables in the text.
    \end{itemize}

\item {\bf Experiments compute resources}
    \item[] Question: For each experiment, does the paper provide sufficient information on the computer resources (type of compute workers, memory, time of execution) needed to reproduce the experiments?
    \item[] Answer: \answerYes{} 
    \item[] Justification: It is provide in \Cref{app:exp_details}.
    \item[] Guidelines:
    \begin{itemize}
        \item The answer \answerNA{} means that the paper does not include experiments.
        \item The paper should indicate the type of compute workers CPU or GPU, internal cluster, or cloud provider, including relevant memory and storage.
        \item The paper should provide the amount of compute required for each of the individual experimental runs as well as estimate the total compute. 
        \item The paper should disclose whether the full research project required more compute than the experiments reported in the paper (e.g., preliminary or failed experiments that didn't make it into the paper). 
    \end{itemize}
    
\item {\bf Code of ethics}
    \item[] Question: Does the research conducted in the paper conform, in every respect, with the NeurIPS Code of Ethics \url{https://neurips.cc/public/EthicsGuidelines}?
    \item[] Answer: \answerYes{} 
    \item[] Justification: Our work conforms with the NeurIPS Code of Ethics.
    \item[] Guidelines:
    \begin{itemize}
        \item The answer \answerNA{} means that the authors have not reviewed the NeurIPS Code of Ethics.
        \item If the authors answer \answerNo, they should explain the special circumstances that require a deviation from the Code of Ethics.
        \item The authors should make sure to preserve anonymity (e.g., if there is a special consideration due to laws or regulations in their jurisdiction).
    \end{itemize}

\item {\bf Broader impacts}
    \item[] Question: Does the paper discuss both potential positive societal impacts and negative societal impacts of the work performed?
    \item[] Answer: \answerYes{} 
    \item[] Justification: Broader impact statement is provided in \Cref{app:broader_impact}. As our work is about understanding foundational models, we do not see any immediate risks from our research, and the potential benefits outweigh the risks.
    \item[] Guidelines:
    \begin{itemize}
        \item The answer \answerNA{} means that there is no societal impact of the work performed.
        \item If the authors answer \answerNA{} or \answerNo, they should explain why their work has no societal impact or why the paper does not address societal impact.
        \item Examples of negative societal impacts include potential malicious or unintended uses (e.g., disinformation, generating fake profiles, surveillance), fairness considerations (e.g., deployment of technologies that could make decisions that unfairly impact specific groups), privacy considerations, and security considerations.
        \item The conference expects that many papers will be foundational research and not tied to particular applications, let alone deployments. However, if there is a direct path to any negative applications, the authors should point it out. For example, it is legitimate to point out that an improvement in the quality of generative models could be used to generate Deepfakes for disinformation. On the other hand, it is not needed to point out that a generic algorithm for optimizing neural networks could enable people to train models that generate Deepfakes faster.
        \item The authors should consider possible harms that could arise when the technology is being used as intended and functioning correctly, harms that could arise when the technology is being used as intended but gives incorrect results, and harms following from (intentional or unintentional) misuse of the technology.
        \item If there are negative societal impacts, the authors could also discuss possible mitigation strategies (e.g., gated release of models, providing defenses in addition to attacks, mechanisms for monitoring misuse, mechanisms to monitor how a system learns from feedback over time, improving the efficiency and accessibility of ML).
    \end{itemize}
    
\item {\bf Safeguards}
    \item[] Question: Does the paper describe safeguards that have been put in place for responsible release of data or models that have a high risk for misuse (e.g., pre-trained language models, image generators, or scraped datasets)?
    \item[] Answer: \answerNA{} 
    \item[] Justification: 
    \item[] Guidelines:
    \begin{itemize}
        \item The answer \answerNA{} means that the paper poses no such risks.
        \item Released models that have a high risk for misuse or dual-use should be released with necessary safeguards to allow for controlled use of the model, for example by requiring that users adhere to usage guidelines or restrictions to access the model or implementing safety filters. 
        \item Datasets that have been scraped from the Internet could pose safety risks. The authors should describe how they avoided releasing unsafe images.
        \item We recognize that providing effective safeguards is challenging, and many papers do not require this, but we encourage authors to take this into account and make a best faith effort.
    \end{itemize}

\item {\bf Licenses for existing assets}
    \item[] Question: Are the creators or original owners of assets (e.g., code, data, models), used in the paper, properly credited and are the license and terms of use explicitly mentioned and properly respected?
    \item[] Answer: \answerYes{} 
    \item[] Justification: The datasets used in our work are open-source and are all properly cited. 
    \item[] Guidelines:
    \begin{itemize}
        \item The answer \answerNA{} means that the paper does not use existing assets.
        \item The authors should cite the original paper that produced the code package or dataset.
        \item The authors should state which version of the asset is used and, if possible, include a URL.
        \item The name of the license (e.g., CC-BY 4.0) should be included for each asset.
        \item For scraped data from a particular source (e.g., website), the copyright and terms of service of that source should be provided.
        \item If assets are released, the license, copyright information, and terms of use in the package should be provided. For popular datasets, \url{paperswithcode.com/datasets} has curated licenses for some datasets. Their licensing guide can help determine the license of a dataset.
        \item For existing datasets that are re-packaged, both the original license and the license of the derived asset (if it has changed) should be provided.
        \item If this information is not available online, the authors are encouraged to reach out to the asset's creators.
    \end{itemize}

\item {\bf New assets}
    \item[] Question: Are new assets introduced in the paper well documented and is the documentation provided alongside the assets?
    \item[] Answer: \answerNA{} 
    \item[] Justification: 
    \item[] Guidelines:
    \begin{itemize}
        \item The answer \answerNA{} means that the paper does not release new assets.
        \item Researchers should communicate the details of the dataset\slash code\slash model as part of their submissions via structured templates. This includes details about training, license, limitations, etc. 
        \item The paper should discuss whether and how consent was obtained from people whose asset is used.
        \item At submission time, remember to anonymize your assets (if applicable). You can either create an anonymized URL or include an anonymized zip file.
    \end{itemize}

\item {\bf Crowdsourcing and research with human subjects}
    \item[] Question: For crowdsourcing experiments and research with human subjects, does the paper include the full text of instructions given to participants and screenshots, if applicable, as well as details about compensation (if any)? 
    \item[] Answer: \answerNA{} 
    \item[] Justification: 
    \item[] Guidelines:
    \begin{itemize}
        \item The answer \answerNA{} means that the paper does not involve crowdsourcing nor research with human subjects.
        \item Including this information in the supplemental material is fine, but if the main contribution of the paper involves human subjects, then as much detail as possible should be included in the main paper. 
        \item According to the NeurIPS Code of Ethics, workers involved in data collection, curation, or other labor should be paid at least the minimum wage in the country of the data collector. 
    \end{itemize}

\item {\bf Institutional review board (IRB) approvals or equivalent for research with human subjects}
    \item[] Question: Does the paper describe potential risks incurred by study participants, whether such risks were disclosed to the subjects, and whether Institutional Review Board (IRB) approvals (or an equivalent approval/review based on the requirements of your country or institution) were obtained?
    \item[] Answer: \answerNA{} 
    \item[] Justification: 
    \item[] Guidelines:
    \begin{itemize}
        \item The answer \answerNA{} means that the paper does not involve crowdsourcing nor research with human subjects.
        \item Depending on the country in which research is conducted, IRB approval (or equivalent) may be required for any human subjects research. If you obtained IRB approval, you should clearly state this in the paper. 
        \item We recognize that the procedures for this may vary significantly between institutions and locations, and we expect authors to adhere to the NeurIPS Code of Ethics and the guidelines for their institution. 
        \item For initial submissions, do not include any information that would break anonymity (if applicable), such as the institution conducting the review.
    \end{itemize}

\item {\bf Declaration of LLM usage}
    \item[] Question: Does the paper describe the usage of LLMs if it is an important, original, or non-standard component of the core methods in this research? Note that if the LLM is used only for writing, editing, or formatting purposes and does \emph{not} impact the core methodology, scientific rigor, or originality of the research, declaration is not required.
    \item[] Answer: \answerNA{} 
    \item[] Justification: 
    \item[] Guidelines:
    \begin{itemize}
        \item The answer \answerNA{} means that the core method development in this research does not involve LLMs as any important, original, or non-standard components.
        \item Please refer to our LLM policy in the NeurIPS handbook for what should or should not be described.
    \end{itemize}

\end{enumerate}

\end{document}